\documentclass{article}

    \usepackage[final]{neurips_2023}

\usepackage[utf8]{inputenc} %
\usepackage[T1]{fontenc}    %
\usepackage{hyperref}       %
\usepackage{url}            %
\usepackage{booktabs}       %
\usepackage{amsfonts}       %
\usepackage{nicefrac}       %
\usepackage{microtype}      %
\usepackage{xcolor}         %
\usepackage{wrapfig}
\usepackage{algorithm}
\usepackage{algorithmic}
\usepackage{subfigure}
\usepackage{amsthm}
\usepackage{amssymb}
\usepackage{mathtools}
\usepackage{thm-restate}
\usepackage{dsfont}
\usepackage{enumitem}
\newcounter{relctr} %
\everydisplay\expandafter{\the\everydisplay\setcounter{relctr}{0}} %

\AtBeginDocument{} %
\DeclareMathOperator{\E}{\mathds{E}} %
\DeclareMathOperator{\PR}{\mathds{P}} %
\DeclareMathOperator{\sset}{\mathcal{S}}
\DeclareMathOperator{\aset}{\mathcal{A}}
\DeclareMathOperator{\pieval}{{\pi_e}}
\DeclareMathOperator{\saset}{{\mathcal{X}}}
\DeclareMathOperator{\abssaset}{{\widetilde{\mathcal{X}}}}
\DeclareMathOperator{\absrf}{\tilde{r}}

\newtheorem{definition}{Definition}

\usepackage{color-edits}
\addauthor[Brahma]{bp}{blue}

\title{State-Action Similarity-Based Representations for Off-Policy Evaluation}

\author{%
  Brahma S. Pavse \textnormal{and} Josiah P. Hanna\\
  University of Wisconsin -- Madison \\
  \texttt{pavse@wisc.edu, jphanna@cs.wisc.edu} \\
}

\begin{document}

\maketitle

\begin{abstract}
 In reinforcement learning, off-policy evaluation (\textsc{ope}) is the problem of estimating the expected return of an evaluation policy given a fixed dataset that was collected by running one or more different policies. One of the more empirically successful algorithms for \textsc{ope} has been the fitted q-evaluation (\textsc{fqe}) algorithm that uses temporal difference updates to learn an action-value function, which is then used to estimate the expected return of the evaluation policy. Typically, the original fixed dataset is fed directly into \textsc{fqe} to learn the action-value function of the evaluation policy. Instead, in this paper, we seek to enhance the data-efficiency of \textsc{fqe} by first transforming the fixed dataset using a learned encoder, and then feeding the transformed dataset into \textsc{fqe}. To learn such an encoder, we introduce an \textsc{ope}-tailored state-action behavioral similarity metric, and use this metric and the fixed dataset to  learn an encoder that models this metric. Theoretically, we show that this metric allows us to bound the error in the resulting \textsc{ope} estimate. Empirically, we show that other state-action similarity metrics lead to representations that cannot represent the action-value function of the evaluation policy, and that our state-action representation method boosts the data-efficiency of \textsc{fqe} and lowers \textsc{ope} error relative to other \textsc{ope}-based representation learning methods on challenging \textsc{ope} tasks. We also empirically show that the learned representations significantly mitigate divergence of \textsc{fqe} under varying distribution shifts. Our code is available here: \url{https://github.com/Badger-RL/ROPE}.
\end{abstract}

\section{Introduction}
In real life applications of reinforcement learning, practitioners often wish to assess the performance of a learned policy before allowing it to make decisions with real life consequences \citep{theocharous_personalized_2015}. That is, they want to be able to evaluate the performance of a policy without actually deploying it.
One approach of accomplishing this goal is to apply methods for off-policy evaluation (\textsc{ope}).
\textsc{ope} methods evaluate the performance of a given evaluation policy using a fixed offline dataset previously collected by one or more policies that may be different from the evaluation policy.

One of the core challenges in \textsc{ope} is that the offline datasets may have limited size. In this situation, it is often critical that \textsc{ope} algorithms are data-efficient. That is, they are able produce accurate estimates of the evaluation policy value even when only small amounts of data are available.
In this paper, we seek to enhance the data-efficiency of \textsc{ope} methods through representation learning. 
While prior works have studied representation learning for \textsc{ope}, they have mostly considered representations that induce guaranteed convergent learning without considering whether data-efficiency increases \citep{chang_learning_2022,wang_instabilities_2021}.
For example, \cite{chang_learning_2022} introduce a method for learning Bellman complete representations for \textsc{fqe} but empirically find that having such a learned representation provides little benefit compared to \textsc{fqe} without the learned representation.
Thus, in this work we ask the question, "can explicit representation learning lead to more data-efficient \textsc{ope}?"

To answer this question, we take inspiration from recent advances in learning state similarity metrics for control \citep{castro_mico_2022, zhang_learning_2021}.
These works define behavioral similarity metrics that measure the distance between two states.
They then show that state representations can be learned such that states that are close under the metric will also have similar representations.
In our work, we introduce a new \textsc{ope}-tailored behavioral similarity metric called \textbf{R}epresentations for \textbf{O}ff-\textbf{P}olicy \textbf{E}valuation (\textsc{rope}) and show that learning \textsc{rope} representations can lead to more accurate \textsc{ope}.

Specifically, 
\textsc{rope} first uses the fixed offline dataset to learn a state-action encoder based on this \textsc{ope}-specific state-action similarity metric, and then applies this encoder to the same dataset to produce a new representation for all state-action pairs. The transformed data is then fed into the fitted q-evaluation (\textsc{fqe}) algorithm \citep{le_batch_2019} to produce an \textsc{ope} estimate. We theoretically show that the error between the policy value estimate with \textsc{fqe} + \textsc{rope} and the true evaluation policy value is upper-bounded in terms of how \textsc{rope} aggregates state-action pairs. We empirically show that \textsc{rope} improves the data-efficiency of \textsc{fqe} and leads to lower \textsc{ope} error compared to other \textsc{ope}-based representation learning baselines. Additionally, we empirically show that \textsc{rope} representations mitigate divergence of \textsc{fqe} under extreme distribution. To the best of our knowledge, our work is the first to propose an \textsc{ope}-specific state-action similarity metric that increases the data-efficiency of \textsc{ope}.

\section{Background}

In this section, we formalize our problem setting and discuss prior work.

\subsection{Notation and Problem Setup}
    We consider an infinite-horizon
    Markov decision process (\textsc{mdp}) \citep{puterman_mdp_2014}, $\mathcal{M} = \langle \sset, \aset, \mathcal{R}, P, \gamma , d_0\rangle$,
    where $\sset$ is the state-space, $\aset$ is the action-space,
    $\mathcal{R}:\sset \times\aset \to \Delta([0,\infty))$ is the reward function,
    $P:\sset\times\aset\to\Delta(\sset)$ is the transition dynamics function, $\gamma\in[0,1)$ is the discount factor, and $d_0\in \Delta(\sset)$ is the initial state distribution, where $\Delta(X)$ is the set of all probability distributions over a set $X$. We refer to the joint state-action space as $\saset := \sset\times\aset$. The
    agent acting, according to policy $\pi$, in the \textsc{mdp} generates a trajectory: $S_0, A_0, R_0, S_1, A_1, R_1, ...$, where $S_0\sim d_0$, $A_t\sim\pi(\cdot|S_t)$, 
    $R_t\sim \mathcal{R}(S_t, A_t)$, and $S_{t+1}\sim P(\cdot|S_t, A_t)$ for
    $t\geq0$. We define $r(s,a) := \E[\mathcal{R}(s, a)]$.

    We define the performance of policy $\pi$ to be its expected discounted return, $\rho(\pi) \coloneqq \E_{}[\sum_{t=0}^\infty \gamma^t R_t]$.
    We then have the action-value function of a policy for a given state-action pair, $q^\pi(s,a) = r(s,a) + \gamma \E_{S' \sim P(s,a), A' \sim \pi}[q^\pi(S', A')]$, which gives the expected discounted return when starting in state $s$ and then taking action $a$. Then $\rho(\pi)$ can also be expressed as $\rho(\pi) = \E_{S_0 \sim d_0, A_0 \sim \pi}[q^\pi(S_0, A_0)]$.
    
    It is often more convenient to work with vectors instead of atomic states and actions.
    We use $\phi: \mathcal{S} \times \mathcal{A} \rightarrow \mathbb{R}^d$ to denote a representation function that maps state-action pairs to vectors with some dimensionality $d$.

\subsection{Off-Policy Evaluation (OPE)}
    In off-policy evaluation, we are given a fixed dataset of $m$ transition tuples $\mathcal{D} := \{(s_i, a_i, s_i', r_i)\}_{i=1}^{m}$ and an evaluation policy, $\pieval$.
    Our goal is to use $\mathcal{D}$ to estimate $\rho(\pieval)$.
    Crucially, $\mathcal{D}$ may have been generated by a set of \textit{behavior} policies that are different from $\pieval$, which means that simply averaging the discounted returns in $\mathcal{D}$ will produce an inconsistent estimate of $\rho(\pieval)$.
    We do \textit{not} assume that these behavior policies are known to us, however, we do make the standard assumption that $\forall s \in \sset, \forall a \in \aset$ if $\pieval(a|s) > 0$ then the state-action pair $(s,a)$ has non-zero probability of appearing in $\mathcal{D}$.

    As done by \cite{fu_opebench_2021}, we measure the accuracy of an \textsc{ope} estimator with the \textit{mean absolute error} (\textsc{mae}) to be robust to outliers.
    Let $\hat{\rho}(\pieval, \mathcal{D})$ be the estimate returned by an \textsc{ope} method using $\mathcal{D}$.
    The \textsc{mae} of this estimate is given as:
    \[
        \operatorname{\textsc{mae}}[\hat{\rho}] \coloneqq \E_{\mathcal{D}}[|\hat{\rho}(\pieval, \mathcal{D}) - \rho(\pieval)|].
    \]
    While in practice $\rho(\pieval)$ is unknown, it is standard for the sake of empirical analysis~\citep{voloshin_opebench_2021, fu_opebench_2021} to estimate it by executing rollouts of $\pieval$.
\subsection{Fitted Q-Evaluation}

One of the more successful \textsc{ope} methods has been fitted q-evaluation (\textsc{fqe}) which uses batch temporal difference learning \citep{sutton_learning_1988} to estimate $\rho(\pieval)$ \citep{le_batch_2019}.
\textsc{fqe} involves two conceptual steps: 1) repeat temporal difference policy evaluation updates to estimate $q^{\pieval}(s,a)$ and then 2) estimate $\rho(\pieval)$ as the mean action-value at the initial state distribution.
Formally, let the action-value function be parameterized by $\xi$ i.e. $q_{\xi}$, then the following loss function is minimized to estimate $q^{\pieval}$:
\begin{equation*}
    \mathcal{L}_{\text{FQE}} (\xi) := \E_{(s, a, s', r)\sim\mathcal{D}}\left[\left(r(s,a) + \gamma \E_{a' \sim \pieval(\cdot|s')} [q_{\bar{\xi}}(s',a')] - q_{\xi}(s,a)\right)^2\right]
\end{equation*}
where $\bar{\xi}$ is a separate copy of the parameters $\xi$ and acts as the target function approximator \citep{mnih_dqn_2015} that is updated to $\xi$ at a certain frequency. The learned $q_{\xi^*}$ is then used to estimate the policy value: $\hat{\rho}(\pieval) \coloneqq \E_{s_0\sim d_0, a_0 \sim \pieval}[q_{\xi^*}(s_0, a_0)]$. While conceptually \textsc{fqe} can be implemented with many classes of function approximator to represent the $q_{\xi}$, in practice, deep neural networks are often the function approximator of choice. When using deep neural networks, \textsc{fqe} can be considered a policy evaluation variant of neural fitted q-iteration \citep{riedmiller_neural_2005}.

\subsection{Related Work}

In this section, we discuss the most relevant prior literature on off-policy evaluation and representation learning.
Methods for \textsc{ope} are generally categorized as importance-sampling based \citep{precup_off-policy_nodate,thomas_high-confidence_nodate,hanna_importance_2021,liu_breaking_2018,yang_off-policy_2020}, model-based \citep{yang_representation_2021,zhang_autoregressive_2021,hanna_bootstrapping_2017}, value-function-based \citep{le_batch_2019,uehara_minimax_2020}, or hybrid \citep{jiang_doubly_2016,thomas_data-efficient_2016,farajtabar_more_2018}.
Our work focuses on \textsc{fqe}, which is a representative value-function-based method, since it has been shown to have strong empirical performance \citep{fu_opebench_2021, chang_learning_2022}.
We refer the reader to \cite{levine_offline_2020} for an in-depth survey of \textsc{ope} methods.
\paragraph{Representation Learning for Off-policy Evaluation and Offline RL}

A handful of works have considered the interplay of representation learning with \textsc{ope} methods and offline RL.
\cite{yang_representation_2021} benchmark a number of existing representation learning methods for offline RL and show that pre-training representation can be beneficial for offline RL.
They also consider representation learning based on behavioral similarity and find that such representations do not enable successful offline RL. However, their study is focused on evaluating existing algorithms and on \text{control}.
\cite{pavse_scaling_2023} introduced state abstraction \citep{li_towards_2006} as an approach to lower the variance of \textsc{ope} estimates in importance-sampling based methods. However, their work made the strict assumption of granting access to a bisimulation abstraction in theory and relied on a hand-specified abstraction in practice.
Only recently have works started to consider learning representations specifically for \textsc{ope}. \cite{chang_learning_2022} introduced a method for learning Bellman complete representations that enabled convergent approximation of $q^{\pieval}$ with linear function approximation.
\cite{wang_instabilities_2021} show that using the output of the penultimate layer of $\pieval$'s action-value function provides realizability of $q_{\pieval}$, but is insufficient for accurate policy evaluation under extreme distribution shift.
Our work explicitly focuses on boosting the data-efficiency of \textsc{ope} methods and lowers the error of \textsc{ope} estimates compared to  \cite{chang_learning_2022} and \cite{wang_instabilities_2021}.

\paragraph{Representation Learning via Behavioral Similarity}

The representation learning method we introduce builds upon prior work in learning representations in which similar states share similar representations.
 Much of this prior work is based on the notion of a bisimulation abstraction in which two states with identical reward functions and that lead to identical groups of next states should be classified as similar \citep{ferns_bisim_2004, ferns_bisim_2011, ferns_bisim_2014, castro_scalable_2019}.
 The bisimulation metric itself is difficult to learn both computationally and statistically and so recent work has introduced various approximations \citep{castro_mico_2022, castro_scalable_2019, zhang_learning_2021, gelada_deepmdp_2019}.
 To the best of our knowledge, all of this work has considered the \textit{online, control} setting and has only focused on state representation learning.
 In contrast, we introduce a method for learning \emph{state-action} representations for \textsc{ope} with a fixed dataset.
 One exception is the work of \cite{dadashi_ploff_2021}, which proposes to learn state-action representations for offline policy \emph{improvement}. However, as we will show in Section \ref{sec:empirical_study}, the distance metric that they base their representations on is inappropriate in the \textsc{ope} context.

\section{ROPE: State-Action Behavioral Similarity Metric for Off-Policy Evaluation}

In this section, we introduce our primary algorithm: \textbf{R}epresentations for \textbf{\textsc{ope}} (\textsc{rope}), a representation learning method based on state-action behavioral similarity that is tailored to the off-policy evaluation problem. That is, using a fixed off-policy dataset $\mathcal{D}$, \textsc{rope} learns similar representations for state-action pairs that are similar in terms of the action-value function of $\pieval$.

Prior works on representation learning based on state behavioral similarity define a metric that relates the similarity of two states and then map similar states to similar representations \citep{castro_mico_2022, zhang_learning_2021}.
We follow the same high-level approach except we focus instead on learning state-action representations for \textsc{ope}.  One advantage of learning state-action representations over state representations is that we can learn a metric specifically for $\pieval$ by directly sampling actions from $\pieval$ instead of using importance sampling, which can be difficult when the multiple behavior policies are unknown. Moreover, estimating the importance sampling ratio from data is known to be challenging \citep{hanna_importance_2021, nachum_bestdice_2020}.

Our new notion of similarity between state-action pairs is given by the recursively-defined \textsc{rope} distance, $d_{\pieval}(s_1, a_1;s_2, a_2) := |r(s_1,a_1)-r(s_2,a_2)| + \gamma \E_{s_1',s_2'\sim P, a_1', a_2'\sim\pi_e}[d_{\pieval}(s_1', a_1'; s_2', a_2')]$. Intuitively, $d_{\pieval}$ measures how much two state-action pairs, $(s_1,a_1)$ and $(s_2,a_2)$, differ in terms of short-term reward and discounted expected distance between next state-action pairs encountered by $\pieval$. In order to compute $d_{\pieval}$, we define the \textsc{rope} operator:
\begin{definition}[\textsc{rope} operator] Given an evaluation policy $\pieval$, the \textsc{rope} operator $\mathcal{F}^{\pi_e}: \mathbb{R}^{\mathcal{X}\times\mathcal{X}} \to \mathbb{R}^{\mathcal{X}\times\mathcal{X}}$ is given by:
\begin{equation}
    \label{eq:operator}
    \mathcal{F}^{\pi_e}(d)(s_1, a_1; s_2, a_2) := \underbrace{|r(s_1, a_1) - r(s_2, a_2)|}_{\text{short-term distance}} + \gamma \underbrace{\E_{s_1',s_2'\sim P, a_1', a_2'\sim\pi_e}[d(s_1', a_1'; s_2', a_2')]}_{\text{long-term distance}}
\end{equation}
where $d:\mathcal{X}\times\mathcal{X}\to\mathbb{R}$, $s_1' \sim P(s_1'|s_1,a_1), s_2' \sim P(s_2'|s_2,a_2), a_1' \sim \pieval(\cdot|s_1'), a_2' \sim \pieval(\cdot | s_2')$
\label{def:operator}
\end{definition}
Given the operator, $\mathcal{F}^{\pi_e}$, we show that the operator is a contraction mapping, computes the \textsc{rope} distance, $d_{\pieval}$, and that $d_{\pieval}$ is a \emph{diffuse metric}. For the background on metrics and full proofs, refer to the Appendix \ref{app:th_bg} and \ref{app:th_results}.
\begin{restatable}{proposition}{propcontraction}
\label{prop:contraction}
The operator $\mathcal{F}^{\pi_e}$ is a contraction mapping on $\mathbb{R}^{\mathcal{X}\times\mathcal{X}}$ with respect to the $L^\infty$ norm.
\end{restatable}

\begin{restatable}{proposition}{propfixedpoint}
\label{prop:fixedpoint}
The operator $\mathcal{F}^{\pi_e}$ has a unique fixed point $d_{\pi_e}\in\mathbb{R}^{\mathcal{X}\times\mathcal{X}}$. Let $d_0\in\mathbb{R}^{\mathcal{X}\times\mathcal{X}}$, then $\lim_{t\to\infty}\mathcal{F}_t^{\pi_e}(d_0) = d_{\pi_e}$.
\end{restatable}

Propositions \ref{prop:contraction} and \ref{prop:fixedpoint} ensure that repeatedly applying the operator on some function $d:\mathcal{X}\times\mathcal{X}\to\mathbb{R}$ will make $d$ converge to our desired distance metric, $d_{\pieval}$. An important aspect of $d_{\pieval}$ is that it is a diffuse metric:

\begin{restatable}{proposition}{propdiffuse}
$d_{\pi_e}$ is a diffuse metric.
\end{restatable}

where a diffuse metric is the same as a psuedo metric (see Definition~\ref{def:psuedo} in Appendix~\ref{app:th_bg}) except that self-distances can be non-zero i.e. it may be true that $d_{\pi_e}(s,a;s,a) > 0$. This fact arises due to the stochasticity in the transition dynamics and action sampling from $\pi_e$. If we assume a deterministic transition function and a deterministic $\pi_e$, $d_{\pi_e}$ will reduce to a pseudo metric, which gives zero self-distance. In practice, we use a sample approximation of the \textsc{rope} operator to estimate $d_{\pieval}$.

Given that $d_{\pieval}$ is well-defined, we have the following theorem that shows why it is useful in the \textsc{ope} context:
\begin{restatable}{theorem}{thmDbound}
\label{thm:dbound}
For any evaluation policy $\pi_e$ and $(s_1,a_1), (s_2, a_2)\in\mathcal{X}$, we have that $|q^{\pi_e}(s_1,a_1) - q^{\pi_e}(s_2,a_2)|\leq d_{\pi_e}(s_1,a_1,;s_2,a_2)$.
\end{restatable}
Given that our goal is learn representations based on $d_{\pieval}$, Theorem \ref{thm:dbound} implies that whenever $d_{\pieval}$ considers two state-action pairs to be close or have similar representations, they will also have close action-values. In the context of \textsc{ope}, if the distance metric considers two state-action pairs that have \emph{different} action-values to be zero distance apart/have the same representation, then \textsc{fqe} will have to output two different action-values for the same input representation, which inevitably means \textsc{fqe} must be inaccurate for at least one state-action pair. 

\subsection{Learning State-Action Representations with ROPE}

In practice, our goal is to use $d_{\pieval}$ to learn a state-action representation $\phi(s,a) \in \mathbb{R}^d$ such that the distances between these representations matches the distance defined by $d_{\pieval}$. To do so, we follow the approach by \cite{castro_mico_2022} and directly parameterize the value $d_{\pieval}(s_1,a_1; s_2,a_2)$ as follows:
\begin{multline}
  d_{\pieval}(s_1, a_1; s_2, a_2) \approx \tilde{d}_{\omega}(s_1, a_1; s_2, a_2) \coloneqq \frac{||\phi_\omega(s_1,a_1)||_2^2 + ||\phi_\omega(s_2,a_2)||_2^2}{2} \\+ \beta \theta(\phi_\omega(s_1, a_1), \phi_\omega(s_2, a_2)) 
\end{multline}
in which $\phi$ is parameterized by some function approximator whose parameter weights are denoted by $\omega$, $\theta(\cdot, \cdot)$ gives the angular distance between the vector arguments, and $\beta$ is a parameter controlling the weight of the angular distance.
We can then learn the desired $\phi_{\omega}$ through a sampling-based bootstrapping procedure \citep{castro_mico_2022}.
More specifically, the following loss function is minimized to learn the optimal $\omega^*$:
\begin{equation}
    \mathcal{L}_{\text{ROPE}} (\omega) := \E_\mathcal{D}\left[\left(\left|r(s_1,a_1) - r(s_2,a_2)\right| + \gamma \E_{\pieval}[\tilde{d}_{\bar{\omega}}(s_1',a_1'; s_2',a_2')] - \tilde{d}_{\omega}(s_1,a_1; s_2,a_2)\right)^2\right]
\end{equation}
where $\bar{\omega}$ is separate copy of $\omega$ and acts as a target function approximator \citep{mnih_dqn_2015}, which is updated to $\omega$ at a certain frequency. Once $\phi_{\omega^*}$ is obtained using $\mathcal{D}$, we use $\phi_{\omega^*}$  with \textsc{fqe} to perform \textsc{ope} with the same data.
Conceptually, the \textsc{fqe} procedure is unchanged except the learned action-value function now takes $\phi_{\omega^*}(s,a)$ as its argument instead of the state and action directly. 

With \textsc{rope}, state-action pairs are grouped together when they have small pairwise \textsc{rope} distance. Thus, a given group of state-action pairs have similar state-action representations and are behaviorally similar (i.e, have similar rewards and lead to similar future states when following $\pi_e$). Consequently, these state-action pairs will have a similar action-value, which allows data samples from any member of the group to learn the group’s shared action-value as opposed to learning the action-value for each state-action pair individually. This generalized usage of data leads to more data-efficient learning. We refer the reader to Appendix \ref{sec:pseudocode} for \textsc{rope}'s pseudo-code.

\subsection{Action-Value and Policy Value Bounds}

We now theoretically analyze how \textsc{rope} state-action representations help \textsc{fqe} estimate $\rho(\pieval)$.
For this analysis, we focus on hard groupings where groups of similar state-action pairs are aggregated into one cluster and no generalization is performed across clusters; in practice, we learn state-action representations in which the difference between representations approximates the \textsc{rope} distance between state-action pairs.
Furthermore, for theoretical analysis, we consider exact computation of the \textsc{rope} diffuse metric and of action-values using dynamic programming. First, we present the following lemma. For proofs, refer to Appendix \ref{app:th_results}.
\begin{restatable}{lemma}{lemmaQbound}
\label{lemma:qbound}
Assume the rewards $\mathcal{R}:\sset \times\aset \to \Delta([0,1])$ then given an aggregated \textsc{mdp} $\widetilde{\mathcal{M}} = \langle \widetilde{\sset}, \widetilde{\aset},\widetilde{\mathcal{R}}, \widetilde{P}, \gamma , \tilde{d}_0\rangle$ constructed by aggregating state-actions in an $\epsilon$-neighborhood based on $d_{\pieval}$, and an encoder $\phi: \saset\to\abssaset$ that maps state-actions in $\mathcal{X}$ to these clusters, the action-value for the evaluation policy $\pieval$ in the two \textsc{mdp}s are bounded as:
    \[
        | q^{\pieval}(x) - \tilde{q}^{\pieval}(\phi(x))| \leq \frac{2\epsilon}{(1 - \gamma)}
    \]
\end{restatable}
Lemma \ref{lemma:qbound} states that the error in our estimate of the true action-value function of $\pieval$ is upper-bounded by the clustering radius of $d_{\pieval}$, $\epsilon$. Lemma \ref{lemma:qbound} then leads us to our main result:
\begin{restatable}{theorem}{thmJbound}
    \label{thm:jbound}
    Under the same conditions as Lemma \ref{lemma:qbound}, the difference between the expected fitted q-evaluation (\textsc{fqe}) estimate and the expected estimate of \textsc{fqe}+\textsc{rope} is bounded:
    \[
        \big| \E_{s_0,a_0\sim\pieval}[q^{\pieval}(s_0,a_0)] - \E_{s_0,a_0\sim\pieval}[q^{\pieval}(\phi(s_0,a_0))]\big| \leq \frac{2\epsilon}{(1 - \gamma)}
    \]
\end{restatable}

Theorem \ref{thm:jbound} tells us that the error in our estimate of $\rho(\pieval)$ is upper-bounded by the size of the clustering radius $\epsilon$. 
The implication is that grouping state-action pairs according to the \textsc{rope} diffuse metric enables us to upper bound error in the \textsc{ope} estimate.
At an extreme, if we only group state-action pairs with \textit{zero} \textsc{rope} distance together then we obtain zero absolute error meaning that the action-value function for the aggregated \textsc{mdp} is able to realize the action-value function of the original \textsc{mdp}.

\section{Empirical Study}
\label{sec:empirical_study}

In this section, we present an empirical study of \textsc{rope} designed to answer the following questions:
\begin{enumerate}[topsep=-1pt,itemsep=0pt,leftmargin=6mm]
    \item Does \textsc{rope} group state-actions that are behaviorally similar according to $q^{\pieval}$?
    \item Does \textsc{rope} improve the data-efficiency of \textsc{fqe} and achieve lower \textsc{ope} error than other \textsc{ope}-based representation methods?
    \item How sensitive is \textsc{rope} to hyperparameter tuning and  extreme distribution shifts?
\end{enumerate}

\subsection{Empirical Set-up}

We now describe the environments and datasets used in our experiments.

\noindent
\textbf{Didactic Domain.} We provide intuition about \textsc{rope} on our gridworld domain. In this tabular and deterministic environment, an agent starts from the bottom left of a $3\times3$ grid and moves to the terminal state at the top right. The reward function is the negative of the Manhattan distance from the top right. $\pieval$ stochastically moves up or right from the start state and then deterministically moves towards the top right, and moves deterministically right when it is in the center. The behavior policy $\pi_b$  acts uniformly at random in each state. We set $\gamma =0.99$.

\noindent
\textbf{High-Dimensional Domains.} We conduct our experiments on five domains: HumanoidStandup, Swimmer, HalfCheetah, Hopper, and Walker2D, each of which has  $393$, $59$, $23$, $14$, and $23$ as the native state-action dimension respectively. We set $\gamma =0.99$. 

\noindent
\textbf{Datasets.} We consider $12$ different datasets: $3$ custom datasets for HumanoidStandup, Swimmer, and HalfCheetah; and $9$ \textsc{d4rl} datasets \citep{fu_d4rl_2020} for HalfCheetah, Hopper, and Walker2D. Each of the three custom datasets is of size $100$K transition tuples with an equal split between samples generated by $\pieval$ and a lower performing behavior policy. For the \textsc{d4rl} datasets, we consider three types for each domain: random, medium, medium-expert, which consists of samples from a random policy, a lower performing policy, and an equal split between a lower performing and expert evaluation policy ($\pieval$). Each dataset has $1$M transition tuples. Note that due to known discrepancies between environment versions and state-action normalization procedures \footnote{\url{https://github.com/Farama-Foundation/D4RL/tree/master}}, we generate our own datasets using the publicly available policies\footnote{\url{https://github.com/google-research/deep_ope}} instead of using the publicly available datasets. See Appendix~\ref{app:empirical_res} for the details on the data generation procedure.

\noindent
\textbf{Evaluation Protocol.} Following \cite{fu_opebench_2021, voloshin_opebench_2021} and to make error magnitudes more comparable across domains, we use relative mean absolute error (\textsc{rmae}). \textsc{rmae} is computed using a single dataset $\mathcal{D}$ and by generating $n$ seeds: $\text{\textsc{rmae}}_i (\hat{\rho}(\pi_e)) := \frac{\left|\rho(\pi_e) - \hat{\rho_i}(\pi_e)\right|}{\left|\rho(\pi_e) - \rho(\pi_{\text{rand}}) \right)|}$, where  $\hat{\rho_i}(\pi_e)$ is computed using the $i^\text{th}$ seed  and
$\rho(\pi_{\text{rand}})$ is the value of a random policy. We then report the Interquartile Mean (\textsc{iqm}) \citep{agarwal_precipice_2021} of these $n$ \textsc{rmae}s.

\noindent
\textbf{Representation learning + OPE.} Each algorithm is given access to the same fixed dataset to learn $q^{\pieval}$. The representation learning algorithms (\textsc{rope} and baselines) use this dataset to first pre-train a representation encoder, which is then used to transform the fixed dataset. This transformed dataset is then used to estimate $q^{\pieval}$.
Vanilla \textsc{fqe} directly operates on the original state-action pairs.

\subsection{Empirical Results}

We now present our main empirical results.

\subsubsection{Designing ROPE: A State-Action Behavioral Similarity Metric for OPE}
\label{sec:gw_sa_metrics}

The primary consideration when designing a behavioral similarity distance function for \textsc{ope}, and specifically, for \textsc{fqe} is
that the distance function should not consider two state-action pairs with different $q^{\pieval}$ values to be the same. Suppose we have a distance function $d$, two state-actions pairs, $(s_1, a_1)$ and $(s_2, a_2)$, and their corresponding $q^{\pieval}$. Then if $d(s_1, a_1; s_2, a_2)= 0$, it should be the case that $q^{\pieval}(s_1, a_1)= q^{\pieval}(s_2, a_2)$. On the other hand, if $d(s_1, a_1; s_2, a_2)= 0$ but $q^{\pieval}(s_1, a_1)$ and $q^{\pieval}(s_2, a_2)$ are very different, then \textsc{fqe} will have to output \emph{different} action-values for the \emph{same} input, thus inevitably making \textsc{fqe} inaccurate on these state-action pairs.

While there have been a variety of proposed behavioral similarity metrics for control, they do not always satisfy the above criterion for \textsc{ope}. We consider various state-action behavioral similarity metrics. Due to space constraints, we show results only for: on-policy \textsc{mico} \citep{castro_mico_2022} $d_{\pi_b}(s_1, a_1; s_2, a_2) := |r(s_1,a_1)-r(s_2,a_2)| + \gamma\E_{a_1',a_2'\sim \pi_b}[d_{\pi_b}((s_1', a_1'), (s_2', a_2'))]$, which groups state-actions that have equal $q^{\pi_b}$, and defer results for the random-policy metric \citep{dadashi_ploff_2021} and policy similarity metric \citep{agarwal_psm_2021} to the Appendix \ref{app:empirical_res}.

\begin{figure*}[hbtp]
    \centering
        \subfigure[Action-values of $\pi_e$]{\label{fig:gw_act_vals}\includegraphics[scale=0.35]{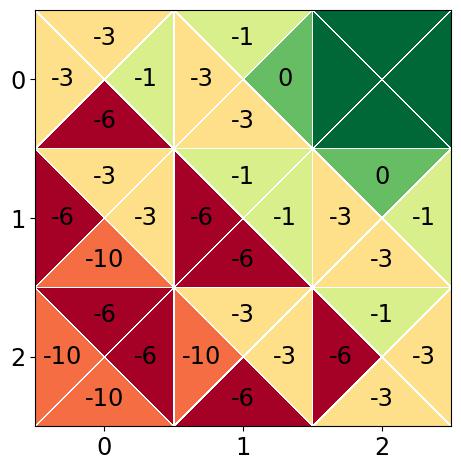}}
        \subfigure[\textsc{rope} groupings]{\label{fig:gw_rope_grouping}\includegraphics[scale=0.35]{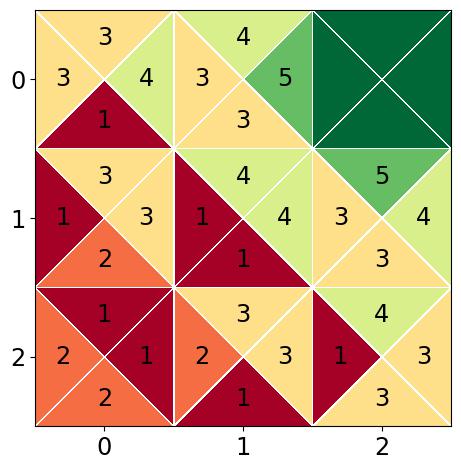}}
        \subfigure[On-policy \textsc{mico} groupings]{\label{fig:gw_mico_grouping}\includegraphics[scale=0.35]{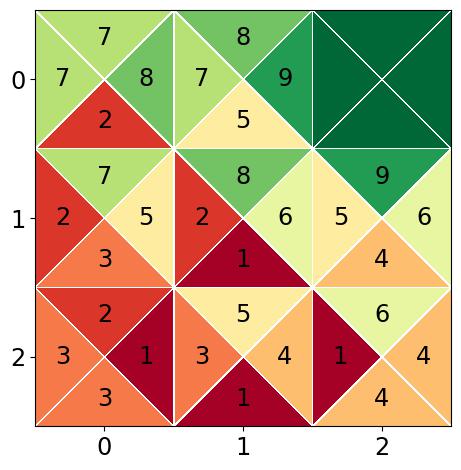}}
    \caption{\footnotesize Figure (a): $q^{\pi_e}$; center number in each triangle is the $q^{\pieval}$ for that state-action pair. Center and right: group clustering according to \textsc{rope} (ours; Figure (b)) and on-policy \textsc{mico} (Figure (c)) (center number in each triangle is group ID). Two state-action pairs are grouped together if their distance according to the specific metric is $0$. The top right cell is blank since it is the terminal state and is not grouped.}
    \label{fig:gw_grouping_visual}
\end{figure*}

We visualize how these different metrics group state-action pairs in our gridworld example where a state-action is represented by a triangle in the grid (Figure~\ref{fig:gw_grouping_visual}). The gridworld is $3\times3$ grid represented by $9$ squares (states), each having $4$ triangles (actions). A numeric entry in a given triangle represents either: 1) the action-value of that state-action pair for $\pieval$ (Figure \ref{fig:gw_act_vals}) or 2) the group ID of the given state-action pair (Figures \ref{fig:gw_rope_grouping} and \ref{fig:gw_mico_grouping}). Along with the group ID, each state-action pair is color-coded indicating its group. In this tabular domain, we compute the distances using dynamic programming with expected updates.

The main question we answer is: does a metric group two state-action pairs together when they have the same action-values under $\pi_e$? In Figure \ref{fig:gw_act_vals} we see the $q^{\pi_e}$ values for each state-action where all state-action pairs that have the same action-value are grouped together under the same color (e.g. all state-action pairs with $q^{\pi_e}(\cdot,\cdot) = -6$ belong to the same group (red)). In Figure \ref{fig:gw_rope_grouping}, we see that \textsc{rope}'s grouping is exactly aligned with the grouping in Figure \ref{fig:gw_act_vals} i.e. state-action pairs that have the same action-values have the same group ID and color. On the other hand, from Figure \ref{fig:gw_mico_grouping}, we see that on-policy \textsc{mico} misaligns with Figure \ref{fig:gw_act_vals}. In Appendix~\ref{app:empirical_res}, we also see similar misaligned groupings using the random-policy metric \cite{dadashi_ploff_2021} and policy similarity metric \cite{agarwal_psm_2021}. The misalignment of these metrics is due to the fact that they do not group state-action pairs togethers that share $q^{\pieval}$ values.

\subsubsection{Deep OPE Experiments}

We now consider \textsc{ope} in challenging, high dimensional continuous state and action space domains. We compare the \textsc{rmae} achieved by an \textsc{ope} algorithm using \emph{different} state-action representations as input. If algorithm A achieves lower error than algorithm B, then A is more data-efficient than B.

\begin{figure*}[hbtp]
    \centering
        \subfigure[Swimmer]{\includegraphics[scale=0.1325]{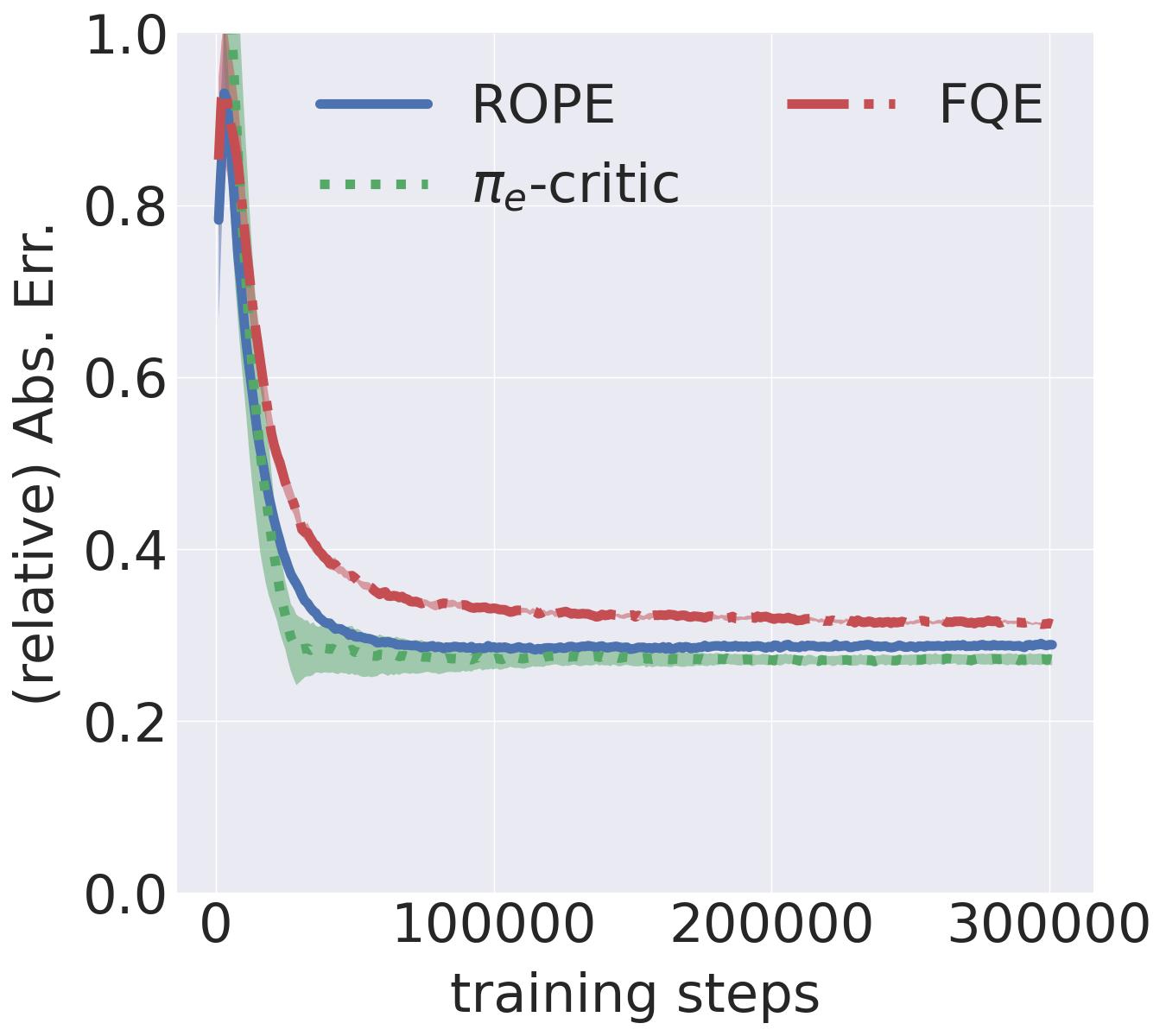}}
        \subfigure[HalfCheetah]{\includegraphics[scale=0.1325]{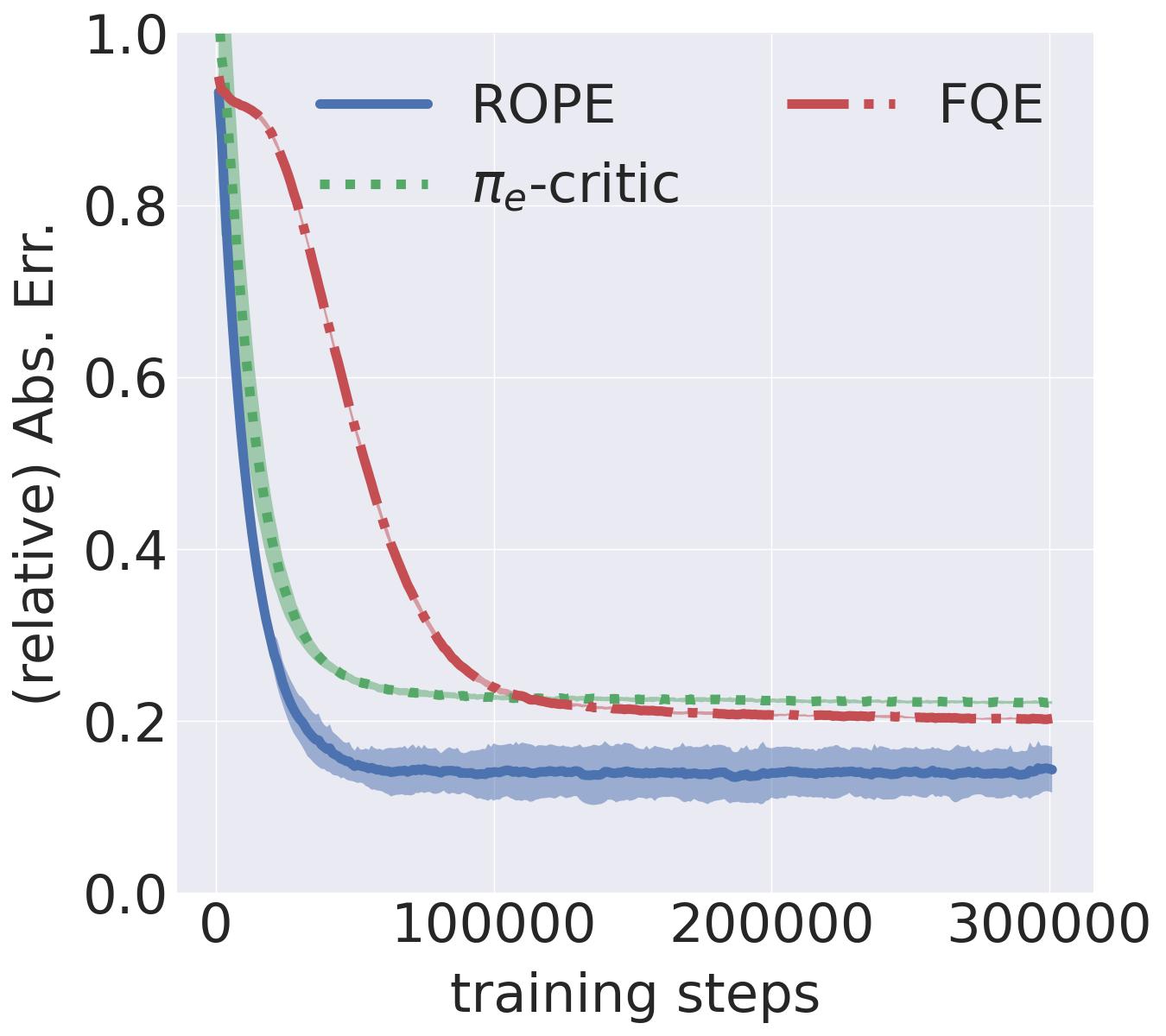}}
        \subfigure[HumanoidStandup]{\includegraphics[scale=0.1325]{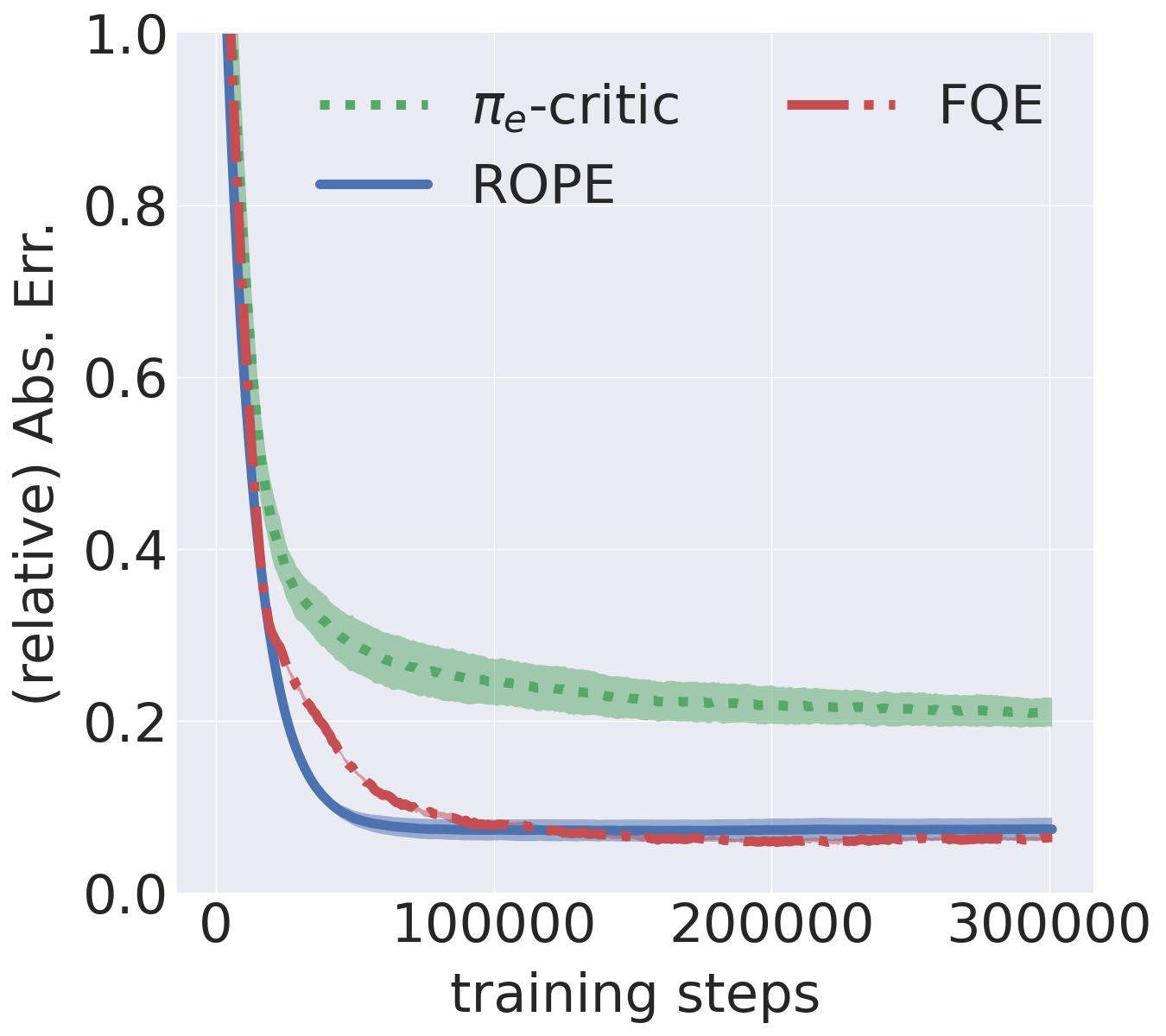}}
    \caption{\footnotesize \textsc{rmae} vs. training iterations of \textsc{fqe} on the custom datasets. \textsc{iqm} of errors for each domain were computed over $20$ trials with $95\%$ confidence intervals. Lower is better.}
    \label{fig:main_custom}
\end{figure*}

\paragraph{Custom Dataset Results}
For the custom datasets, we consider mild distribution shift scenarios, which are typically easy for \textsc{ope} algorithms. In Figure~\ref{fig:main_custom}, we report the \textsc{rmae} vs. training iterations of \textsc{fqe} with different state-action features fed into \textsc{fqe}. We consider three different state-action features: 1) \textsc{rope} (ours), 2) $\pi_e$-critic, which is a representation outputted by the penultimate layer of the action-value function of $\pieval$ \citep{wang_instabilities_2021}, and 3) the original state-action features. Note that there is no representation \emph{learning} involved for 2) and 3). We set the learning rate for all neural network training (encoder and \textsc{fqe}) to be the same, hyperparameter sweep \textsc{rope} across $\beta$ and the dimension of \textsc{rope}'s encoder output, and report the lowest \textsc{rmae} achieved at the end of \textsc{fqe} training. For hyperparameter sensitivity results, see Section \ref{sec:ablations}. For training details, see Appendix~\ref{app:empirical_res}.

We find that \textsc{fqe} converges to an estimate of $\rho(\pieval)$ when it is fed these different state-action features. We also see that when \textsc{fqe} is fed features from \textsc{rope} it produces more data-efficient \textsc{ope} estimates than vanilla \textsc{fqe}. Under these mild distribution shift settings, $\pieval$-critic also performs well since the output of the penultimate layer of $\pieval$'s action-value function should have sufficient information to accurately estimate the action-value function of $\pieval$.

\paragraph{D4RL Dataset Results}
On the \textsc{d4rl} datasets, we analyze the final performance achieved by representation learning + \textsc{ope} algorithms on datasets with varying distribution shift. In addition to the earlier baselines, we evaluate Bellman Complete Learning Representations (\textsc{bcrl}) \citep{chang_learning_2022}, which learns linearly Bellman complete representations and produces an \textsc{ope} estimate with Least-Squares Policy Evaluation (\textsc{lspe}) instead of \textsc{fqe}. We could not evaluate $\pieval$-critic since the \textsc{d4rl} $\pieval$ critics were unavailable\footnote{\url{https://github.com/google-research/deep_ope}}. For \textsc{bcrl}, we use the publicly available code \footnote{\url{https://github.com/CausalML/bcrl}}. For a fair comparison, we hyperparameter tune the representation output dimension and encoder architecture size of \textsc{bcrl}. We hyperparameter tune \textsc{rope} the same way as done for the custom datasets. We set the learning rate for all neural network training (encoder and \textsc{fqe}) to be the same. In Table \ref{table:main_d4rl}, we report the lowest \textsc{rmae} achieved at the end of the \textsc{ope} algorithm's training. For the corresponding training graphs, see Appendix~\ref{app:empirical_res}.

\begin{table*}[t]\centering
\begin{tabular}{@{}llll@{}} \toprule
  & \multicolumn{3}{c}{Algorithm} \\ 
\cmidrule(r){2-4}
Dataset & \textsc{bcrl} & \textsc{fqe}  & \textsc{rope} (ours) \\ \midrule
HalfCheetah-random  & $0.979 \pm 0.000$ & \pmb{$0.807 \pm 0.010$}   &$0.990 \pm 0.001$ \\
HalfCheetah-medium  & $0.830 \pm 0.007$ & $0.770 \pm 0.007$ & \pmb{$0.247 \pm 0.001$}\\
HalfCheetah-medium-expert  & $0.685 \pm 0.013$ & $0.374 \pm 0.001$   &\pmb{$0.078 \pm 0.043$} \\
\midrule
Walker2D-random  & $1.022 \pm 0.001$ & Diverged  &\pmb{$0.879 \pm 0.009$}\\
Walker2D-medium & $0.953 \pm 0.019$ & Diverged  &\pmb{$0.462 \pm 0.093$}\\
Walker2D-medium-expert  & $0.962 \pm 0.037$ & Diverged  & \pmb{$0.252 \pm 0.126$}\\
\midrule
Hopper-random  & Diverged & Diverged  & \pmb{$0.680 \pm 0.05 $}\\
Hopper-medium & $61.223 \pm 92.282$ & Diverged & \pmb{$0.208 \pm 0.048$}\\
Hopper-medium-expert  & $ 9.08 \pm 4.795$ & Diverged  & \pmb{$0.192 \pm 0.055$}\\
\bottomrule
\end{tabular}
\caption{\footnotesize Lowest \textsc{rmae} achieved by algorithm on \textsc{d4rl} datasets. \textsc{iqm} of errors for each domain were computed over $20$ trials with $95\%$ confidence intervals. Algorithms that diverged had a significantly high final error and/or upward error trend (see Appendix~\ref{app:empirical_res} for training curves). Lower is better.} 
\label{table:main_d4rl}
\end{table*}

We find that \textsc{rope} improves the data-efficiency of \textsc{fqe} substantially across varying distribution shifts. \textsc{bcrl} performs competitively, but its poorer \textsc{ope} estimates compared to \textsc{rope} is unsurprising since it is not designed for data-efficiency. It is also known that \textsc{bcrl} may produce less accurate \textsc{ope} estimates compared to \textsc{fqe} \citep{chang_learning_2022}. \textsc{fqe} performs substantially worse on some datasets; however, it is known that \textsc{fqe} can diverge under extreme distribution shift \citep{wang_fqedivergence_2020,wang_instabilities_2021}. It is interesting, however, that \textsc{rope} is robust in these settings. We observe this robustness across a wide range of hyperparameters as well (see Section \ref{sec:ablations}). We also find that when there is low diversity of rewards in the batch (for example, in the random datasets), it is more likely that the short-term distance component of \textsc{rope} is close to $0$, which can result in a representation collapse.

\subsubsection{Ablations}
\label{sec:ablations}
Towards a deeper understanding of \textsc{rope}, we now present an ablation study of \textsc{rope}.

\begin{figure*}[hbtp]
    \centering
        \includegraphics[scale=0.13]{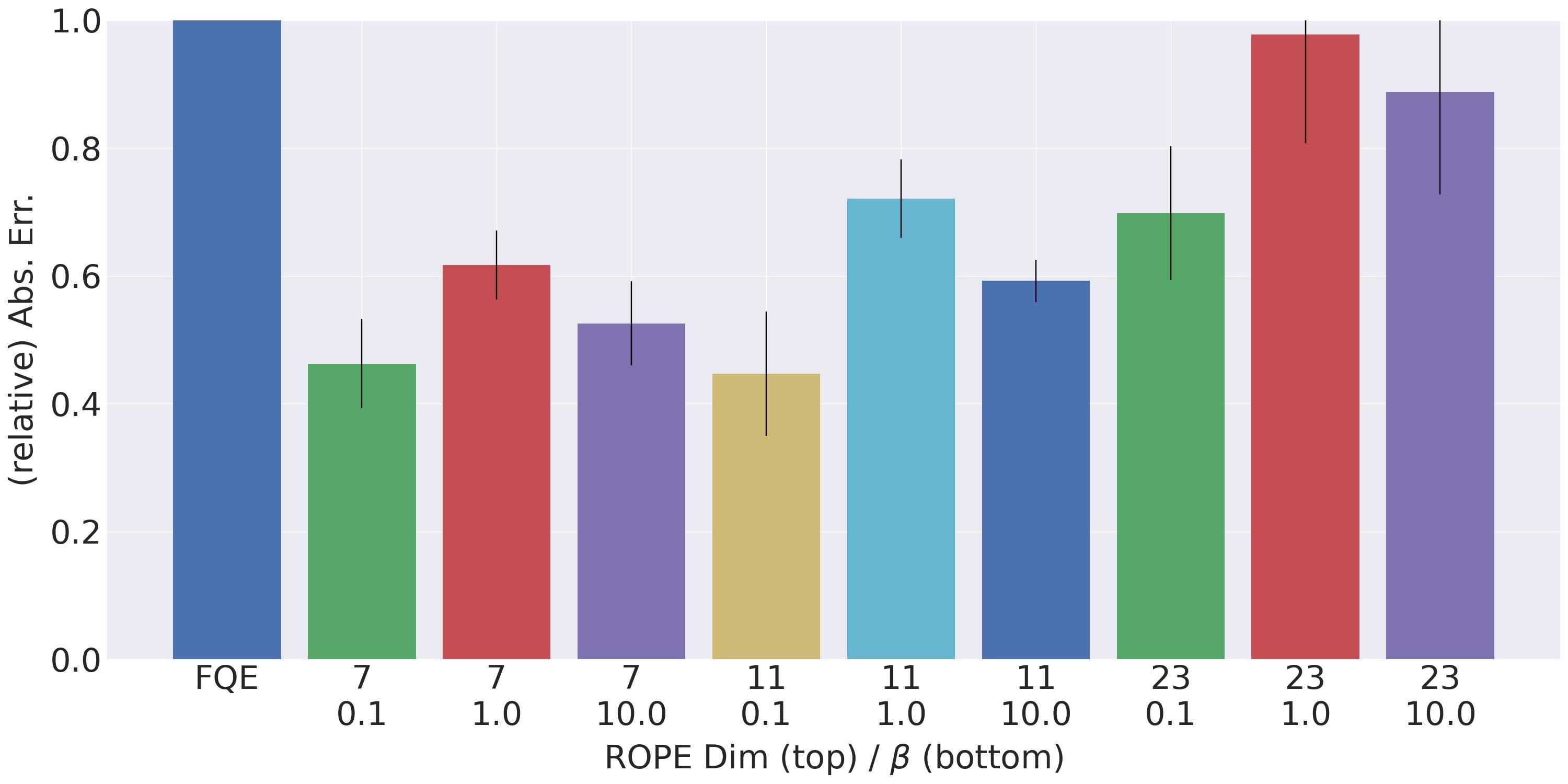}
    \caption{\footnotesize Hyperparameter sensitivity. \textsc{fqe} vs. \textsc{rope} when varying \textsc{rope}'s encoder output dimension (top) and $\beta$ (bottom) on the Walker2D-medium \textsc{d4rl} dataset. \textsc{iqm} of errors are computed over $20$ trials with $95\%$ confidence intervals. Lower is better.}
    \label{fig:hparam}
\end{figure*}

\paragraph{Hyperparameter Sensitivity} In \textsc{ope}, hyperparameter tuning with respect to \textsc{rmae} is difficult since $\rho(\pieval)$ is unknown in practice \citep{paine_hparam_2020}. Therefore, we need \textsc{ope} algorithms to not only produce accurate \textsc{ope} estimates, but also to be robust to hyperparameter tuning. Specifically, we investigate whether \textsc{rope}'s representations produce more data-efficient \textsc{ope} estimates over \textsc{fqe} across \textsc{rope}'s hyperparameters. In this experiment, we set the action-value function's learning rate to be the same for both algorithms. The hyperparameters for \textsc{rope} are: 1) the output dimension of the encoder and 2) $\beta$, the weight on the angular distance between encodings. We plot the results in Figure \ref{fig:hparam} and observe that \textsc{rope} is able to produce substantially more data-efficient estimates compared to \textsc{fqe} for a wide range of its hyperparameters on the Walker2D-medium dataset, where \textsc{fqe} diverged (see Table \ref{table:main_d4rl}). While it is unclear what the optimal hyperparameters should be, we find similar levels of robustness on other datasets as well (see Appendix~\ref{app:empirical_res}).
\begin{figure*}[hbtp]
    \centering
        \subfigure[Hopper-random]{\includegraphics[scale=0.225]{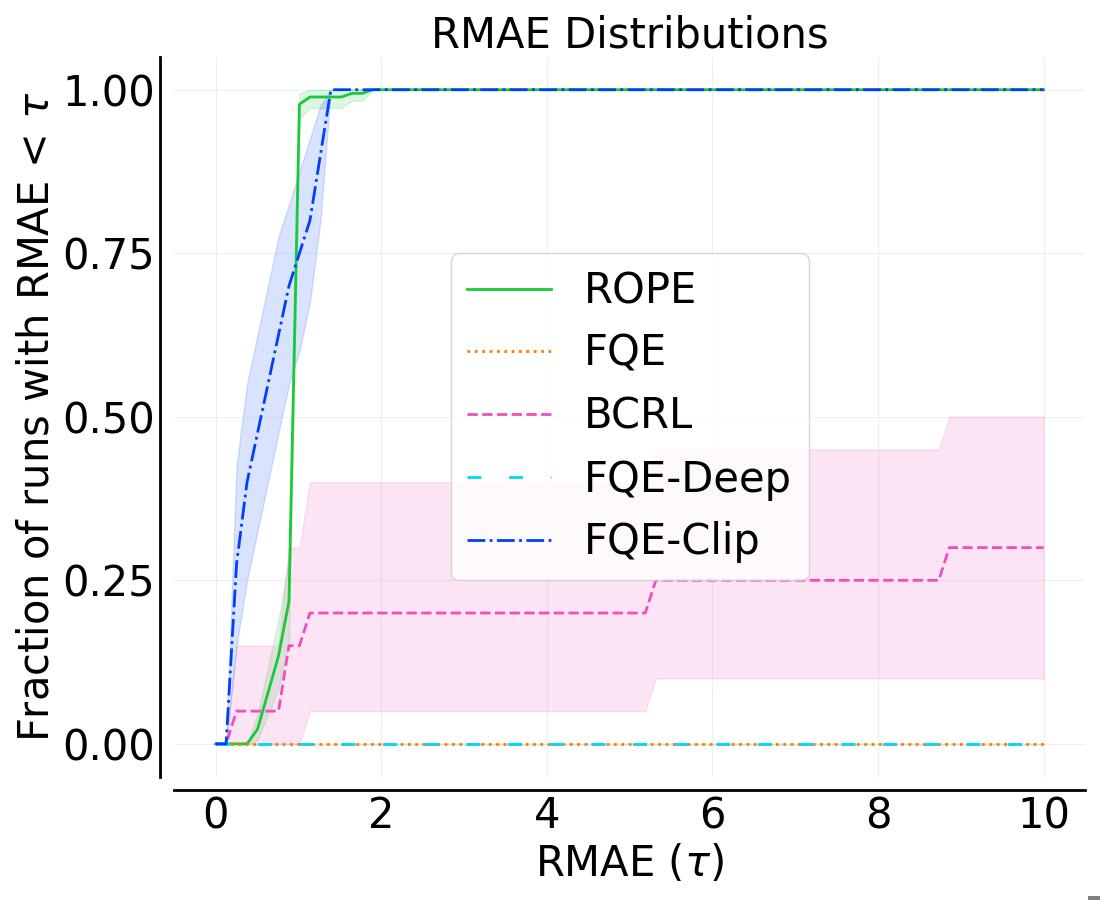}}
        \subfigure[Hopper-medium]{\includegraphics[scale=0.225]{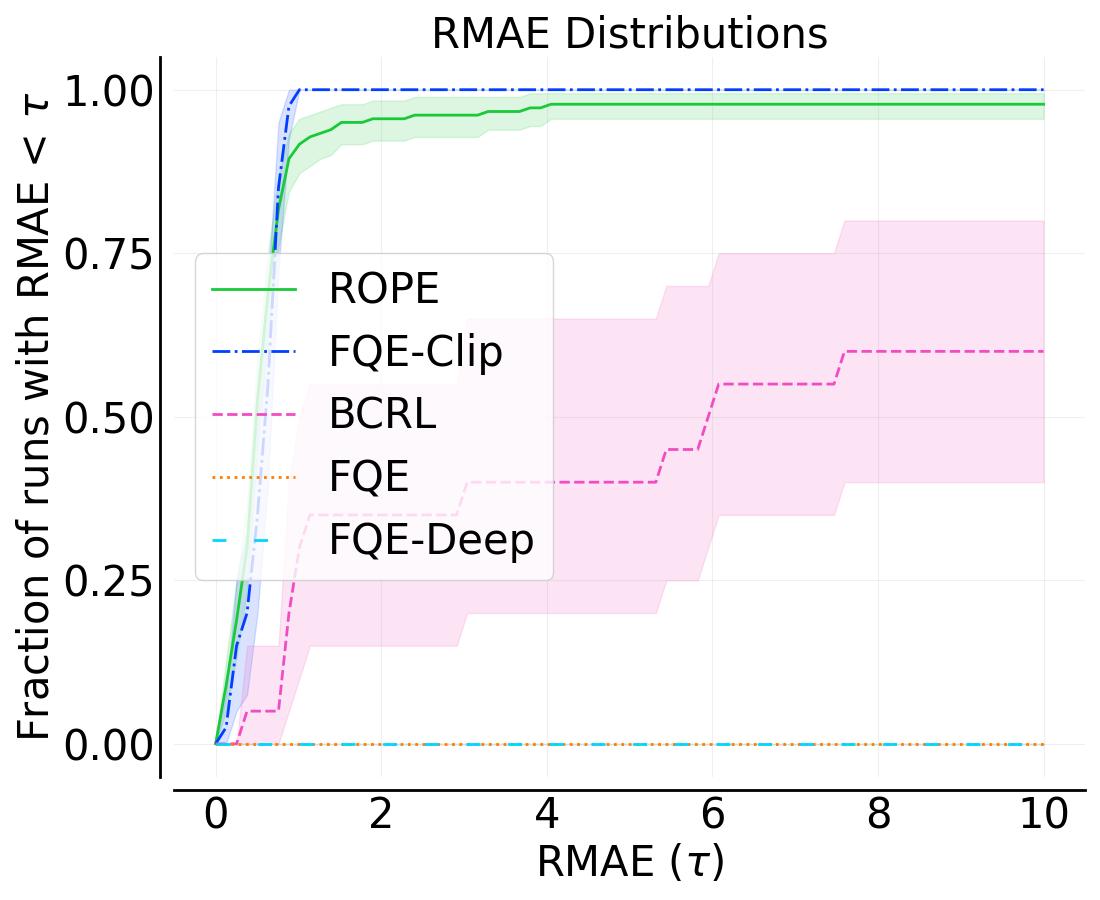}}
    \caption{\footnotesize \textsc{rmae} distributions across all runs and hyperparameters for each algorithm, resulting in $\geq 20$ runs for each algorithm. The shaded region is a $95\%$ confidence interval. Larger area under the curve is better. For visualization, we cut off the horizontal axis at $10$ \textsc{rmae}. \textsc{fqe} and \textsc{fqe-deep} are flat at $0$ i.e. neither had runs that produced an error less than $10$.}
    \label{fig:divergence}
\end{figure*}
\paragraph{ROPE Representations Mitigate FQE Divergence}
It has been shown theoretically \citep{wang_fqedivergence_2020} and empirically \citep{wang_instabilities_2021} that under extreme distribution shift, \textsc{fqe} diverges i.e. it produces \textsc{ope} estimates that have arbitrarily large error. In Table~\ref{table:main_d4rl}, we also see similar results where \textsc{fqe} produces very high error on some datasets. \textsc{fqe} tends to diverge due to the deadly triad \citep{sutton_rlbook_2018}: 1) off-policy data, 2) bootstrapping, and 3) function approximation.

A rather surprising but encouraging result that we find is that even though \textsc{rope} faces the deadly triad, it produces representations that \emph{significantly} mitigate \textsc{fqe}'s divergence across a large number of trials and hyperparameter variations. To investigate how much \textsc{rope} aids convergence, we provide the performance profile\footnote{\url{https://github.com/google-research/rliable/tree/master}} \citep{agarwal_precipice_2021} based on the \textsc{rmae} distribution plot in Figure~\ref{fig:divergence}. Across all trials and hyperparameters, we plot the fraction of times an algorithm achieved an error less than some threshold. In addition to the earlier baselines, we also plot the performance of 1) \textsc{fqe-clip} which is \textsc{fqe} but whose bootstrapping targets are clipped between $[\frac{r_{\text{min}}}{1 - \gamma}, \frac{r_{\text{max}}}{1 - \gamma}]$, where $r_\text{min}$ and $r_{\text{max}}$ are the minimum and maximum rewards in the fixed dataset; and 2) \textsc{fqe-deep}, which is regular \textsc{fqe} but whose action-value function network is double the capacity of \textsc{fqe} (see Appendix~\ref{app:empirical_res} for specifics). 

From Figure~\ref{fig:divergence}, we see that nearly $\approx 100\%$ of the runs of \textsc{rope} achieve an \textsc{rmae} of $\leq 2$, while none of the \textsc{fqe} and \textsc{fqe-deep} runs produce even $\leq 10$ \textsc{rmae}. The failure of \textsc{fqe-deep} suggests that the extra capacity \textsc{rope} has over \textsc{fqe} (since \textsc{rope} has its own neural network encoder) is insufficient to explain why \textsc{rope} produces accurate \textsc{ope} estimates. We also find that in order to use \textsc{fqe} with the native state-action representations, it is necessary to use domain knowledge and clip the bootstrapped target. While \textsc{fqe-clip} avoids divergence, it is very unstable during training (see Appendix~\ref{app:empirical_res}). \textsc{rope}'s ability to produce stable learning in \textsc{fqe} without any clipping is promising since it suggests that it is possible to improve the robustness of \textsc{fqe} if an appropriate representation is learned.

\section{Limitations and Future Work}
In this work, we showed that \textsc{rope} was able to improve the data-efficiency of \textsc{fqe} and produce lower-error \textsc{ope} estimates than other \textsc{ope}-based representations. Here, we highlight limitations and opportunities for future work. A limitation of \textsc{rope} and other bisimulation-based metrics is that if the diversity of rewards in the dataset is low, they are susceptible to representation collapse since the short-term distance is close to $0$. Further investigation is needed to determine how to overcome this limitation. Another very interesting future direction is to understand why \textsc{rope}'s representations significantly mitigated \textsc{fqe}'s divergence. A starting point would be to explore potential connections between \textsc{rope} and Bellman complete representations \citep{csaba_completeness_2005} and other forms of representation regularizers for \textsc{fqe}\footnote{\url{https://offline-rl-neurips.github.io/2021/pdf/17.pdf}}.

\section{Conclusion}
In this paper we studied the challenge of pre-training representations to increase the data efficiency of the \textsc{fqe} \textsc{ope} estimator. Inspired by work that learns state similarity metrics for control, we introduced \textsc{rope}, a new diffuse metric for measuring behavioral similarity between state-action pairs for \textsc{ope} and used \textsc{rope} to learn state-action representations using available offline data. We theoretically showed that \textsc{rope}: 1) bounds the difference between the action-values between different state-action pairs and 2) results in bounded error between the value of $\pieval$ according to the ground action-value and the action-value function that is fed with \textsc{rope} representations as input. We empirically showed that \textsc{rope} boosts the data-efficiency of \textsc{fqe} and achieves lower \textsc{ope} error than other \textsc{ope}-based representation learning algorithms. Finally, we conducted a thorough ablation study and showed that \textsc{rope} is robust to hyperparameter tuning and \emph{significantly} mitigates \textsc{fqe}'s divergence, which is a well-known challenge in \textsc{ope}. To the best of our knowledge, our work is the first  that successfully uses representation learning to improve the data-efficiency of \textsc{ope}.

\section*{Remarks on Negative Societal Impact}
Our work is largely focused on studying fundamental \textsc{rl} research questions, and thus we do not see any immediate negative societal impacts.  The aim of our work is to enable effective \textsc{ope} in many real world domains. Effective \textsc{ope} means that a user can estimate policy performance prior to deployment which can help avoid deployment of poor policies and thus positively impact society.

\section*{Acknowledgments}

Thanks to Adam Labiosa and the anonymous reviewers for feedback that greatly improved our work. Support for this research was provided by American Family Insurance through a research partnership with the University of Wisconsin—Madison’s Data Science Institute.

\bibliographystyle{plainnat}

\bibliography{sources}

\newpage
\appendix

\section{Theoretical Background}
\label{app:th_bg}

In this section, we include relevant background material.

\begin{definition}[Metric] A metric, $d:X\times X\to\mathbb{R}_{\geq 0}$ has the following properties for some $x,y,z\in X$:
\begin{enumerate}
    \item $d(x,x) = 0$
    \item $d(x,y) = 0 \Longleftrightarrow x = y$  
    \item Symmetry: $d(x,y) = d(y,x)$
    \item Triangle inequality: $d(x,z)\leq d(x,y) + d(y,z)$
\end{enumerate}
\end{definition}

\begin{definition}[Pseudo Metric] A pseudo metric, $d:X\times X\to\mathbb{R}_{\geq 0}$ has the following properties for some $x,y,z\in X$:
\label{def:psuedo}
\begin{enumerate}
    \item $d(x,x) = 0$
    \item Symmetry: $d(x,y) = d(y,x)$
    \item Triangle inequality: $d(x,z)\leq d(x,y) + d(y,z)$
\end{enumerate}
Crucially, a pseudo metric differs from a metric in that if $d(x,y) = 0$ it may be the case that $x \neq y$.
\end{definition}

\begin{definition}[Diffuse Metric] A diffuse metric, $d:X\times X\to\mathbb{R}_{\geq 0}$ has the following properties for some $x,y,z\in X$:
\label{def:diffuse}
\begin{enumerate}
    \item $d(x,x) \geq 0$
    \item Symmetry: $d(x,y) = d(y,x)$
    \item Triangle inequality: $d(x,z)\leq d(x,y) + d(y,z)$
\end{enumerate}
Crucially, a diffuse metric differs from a pseudo metric in that self-distances may be non-zero.
\end{definition}

For readers interested in distances that admit non-zero self-distances, we refer them to material on \emph{partial metrics} \citep{matthews_partial_1992}. We make the following note as \cite{castro_mico_2022}: the original definition of partial metrics (see \cite{matthews_partial_1992}) uses a different triangle inequality criterion than the one in Definition \ref{def:diffuse} and is too strict (i.e. diffuse metrics violate this triangle inequality criterion), so we consider the diffuse metric definition presented in this paper.

We now present background material on the Wasserstein and related distances.
\begin{definition}[Wasserstein Distance \citep{villani_wasser_2008}] Let $d:X\times X\to\mathbb{R}_{\geq 0}$ be a distance function and $\Omega$ the set of all joint distributions with marginals $\mu$ and $\lambda$ over the space $X$, then we have:
\begin{equation}
    W(d)(\mu, \lambda) = \left(\inf_{\omega\in\Omega}\mathbb{E}_{x_1,x_2\sim\omega}[d(x_1,x_2)]\right)
\end{equation}
\end{definition}

\begin{definition}[Dual formulation of the Wasserstein Distance \citep{villani_wasser_2008}] Let $d:X\times X\to\mathbb{R}_{\geq 0}$ be a distance function and marginals $\mu$ and $\lambda$ over the space $X$, then we have:
\label{def:dual_wass}
\begin{equation}
    W(d)(\mu, \lambda) = \sup_{f\in\text{Lip}_{1,d}(X)} \mathbb{E}_{x_1\sim\mu}[f(x_1)] - \mathbb{E}_{x_2\sim\lambda}[f(x_2)]
\end{equation}
where $\text{Lip}_{1,d}(X)$ denotes the $1-$Lipschitz functions $f:X\to\mathbb{R}$ such that $|f(x_1)-f(x_2)|\leq d(x_1,x_2)$.
\end{definition}

\begin{definition}[Łukaszyk–Karmowski distance \citep{ukaszyk_ld_2004}] Let $d:X\times X\to\mathbb{R}_{\geq 0}$ be a distance function and marginals $\mu$ and $\lambda$ over the space $X$, then we have:
\label{def:lk_dist}
\begin{equation}
    D_{\text{LK}}(d)(\mu, \lambda) = \left(\mathbb{E}_{x_1\sim\mu,x_2\sim\lambda}[d(x_1,x_2)]\right)
\end{equation}
\end{definition}

We then have the following fact: $W(d)(\mu, \lambda)\leq D_{\text{LK}}(d)(\mu, \lambda)$ i.e. the Wasserstein distance is upper-bounded by the Łukaszyk–Karmowski distance \citep{castro_mico_2022}.

\section{Theoretical Results}
\label{app:th_results}

\propcontraction*
\begin{proof}
Consider $d, d'\in \mathbb{R}^{\mathcal{X}\times\mathcal{X}}$, then we have:
\begin{align*}
    ||(\mathcal{F}^{\pi_e}d)(s_1, a_1; s_2, a_2) &- (\mathcal{F}^{\pi_e}d')(s_1, a_1; s_2, a_2)||_{\infty} \\ 
    &= ||\gamma \E_{s_1',s_2'\sim P, a_1', a_2'\sim\pi_e}[d(s_1', a_1'; s_2', a_2') - d'(s_1', a_1'; s_2', a_2')] ||_{\infty}\\
    &= |\gamma| \cdot ||\E_{s_1',s_2'\sim P, a_1', a_2'\sim\pi_e}[d(s_1', a_1'; s_2', a_2') - d'(s_1', a_1'; s_2', a_2')] ||_{\infty}\\
    &\leq \gamma \max_{s_1',a_1', s_2',a_2'}|d(s_1', a_1'; s_2', a_2') - d'(s_1', a_1'; s_2', a_2')]| = \gamma ||d - d'||_{\infty}
\end{align*}
\end{proof}

\propfixedpoint*
\begin{proof}
Since $\mathcal{F}^{\pi_e}$ is a contraction mapping and that $\mathbb{R}^{\mathcal{X}\times\mathcal{X}}$ is complete under the $L^\infty$ norm, by Banach's fixed-point theorem, $\lim_{t\to\infty}\mathcal{F}_t^{\pi_e}(d) = d_{\pi_e}$.
\end{proof}

\propdiffuse*
\begin{proof}
To prove that $d_{\pi_e}$ is a diffuse metric, we need to show it has the following properties for $(s_1,a_1), (s_2, a_2), (s_3, a_3) \in\mathcal{X}$. We follow \cite{castro_mico_2022}'s strategy (see Proposition 4.10) to prove that a distance function is a diffuse metric. Recall that $d_{\pieval}(s_1, a_1;s_2, a_2) := |r(s_1,a_1)-r(s_2,a_2)| + \gamma \E_{s_1',s_2'\sim P, a_1', a_2'\sim\pi_e}[d_{\pieval}(s_1', a_1'; s_2', a_2')]$.
\begin{enumerate}
    \item Non-negativity i.e. $d_{\pi_e}(s_1,a_1;s_2,a_2)\geq 0$. Since $|r(s_1,a_1)-r(s_2,a_2)|\geq 0$, recursively rolling out the definition of $d_{\pieval}$ means that $d_{\pi_e}(s_1,a_1;s_2,a_2)$ is a sum of discounted non-negative terms.
    \item Symmetry i.e. $d_{\pi_e}(s_1,a_1;s_2,a_2) = d_{\pi_e}(s_2,a_2;s_1,a_1)$. Since $|r(s_1,a_1)-r(s_2,a_2)| = |r(s_2,a_2)-r(s_1,a_1)|$, unrolling $d_{\pi_e}(s_1,a_1;s_2,a_2)$ and $d_{\pi_e}(s_2,a_2;s_1,a_1)$ recursively results in the discounted sum of the same terms.
    \item Triangle inequality i.e. $d_{\pi_e}(s_1,a_1;s_2,a_2)\leq d_{\pi_e}(s_1,a_1;s_3,a_3) + d_{\pi_e}(s_2,a_2;s_3,a_3)$. To show this fact, we will first consider an initialization to the distance function $d_0(s_1,a_1; s_2,a_2) = 0, \forall (s_1,a_1), (s_2,a_2)\in\mathcal{X}$ and consider repeated
    applications of the operator $\mathcal{F}^{\pi_e}$ to $d_0$, which we know will make $d_0$ converge to $d_{\pi_e}$ (Proposition \ref{prop:fixedpoint}). We will show by induction that each successive update $d_{t+1} = \mathcal{F}^{\pi_e}(d_t)$ satisfies the triangle inequality, which implies that $d_{\pi_e}$ satisfies the triangle inequality.

    We have the base the case at $t=0$ trivially holding true due to the initialization of $d_0$. Now let the inductive hypothesis be true for all $t > 1$ i.e.
    $d_{t}(s_1, a_1; s_2, a_2) \leq d_{t}(s_1, a_1; s_3, a_3) + d_{t}(s_3, a_3; s_2, a_2)$ for any $(s_1,a_1), (s_2, a_2), (s_3, a_3) \in\mathcal{X}$. However, we know that:
    \begin{align*}
        d_{t+1}(s_1, a_1; s_2, a_2) &= |r(s_1, a_1) - r(s_2, a_2)| + \gamma \E_{s_1',s_2'\sim P, a_1', a_2'\sim\pi_e}[d_t(s_1', a_1'; s_2', a_2')]\\
        &\overset{(a)}{=} |r(s_1, a_1) - r(s_2, a_2)| + r(s_3, a_3) - r(s_3, a_3) + \gamma \E_{s_1',s_2'\sim P, a_1', a_2'\sim\pi_e}[d_t(s_1', a_1'; s_2', a_2')]\\
        &\overset{(b)}{\leq} |r(s_1, a_1) - r(s_3, a_3)| + |r(s_2, a_2) - r(s_3, a_3)| \\&+ \gamma \E_{s_1',s_2'\sim P, a_1', a_2'\sim\pi_e}[d_t(s_1', a_1'; s_2', a_2')]\\
        & \overset{(c)}{\leq} |r(s_1, a_1) - r(s_3, a_3)| + |r(s_2, a_2) - r(s_3, a_3)| \\&+ \gamma \E_{s_1',s_2',s_3'\sim P, a_1', a_2',a_3'\sim\pi_e}[d_{t}(s_1', a_1'; s_3', a_3') + d_{t}(s_3', a_3'; s_2', a_2')]\\
        &= |r(s_1, a_1) - r(s_3, a_3)| + \gamma \E_{s_1',s_3'\sim P, a_1',a_3'\sim\pi_e}[d_{t}(s_1', a_1'; s_3', a_3')] \\&+ |r(s_2, a_2) - r(s_3, a_3)| + \gamma \E_{s_2',s_3'\sim P, a_2',a_3'\sim\pi_e}[d_{t}(s_3', a_3'; s_2', a_2')]\\
        &= d_{t+1}(s_1, a_1; s_3, a_3) + d_{t+1}(s_2, a_2; s_3, a_3)\\
        d_{t+1}(s_1, a_1; s_2, a_2) &\leq  d_{t+1}(s_1, a_1; s_3, a_3) + d_{t+1}(s_2, a_2; s_3, a_3)
    \end{align*}
    where (a) is due to adding and subtracting $r(s_3, a_3)$, (b) is due to Jensen's inequality, (c) is due to application of the inductive hypothesis. Thus, the triangle inequality is satisfied for all $t \geq 0$, and given that $d_{t+1}\to d_{\pi_e}$, we have that $d_{\pi_e}$ also satisfies the triangle inequality.
\end{enumerate}
\end{proof}

\thmDbound*
\begin{proof}
To prove this fact, we follow \cite{castro_mico_2022} (see Proposition 4.8) and use a co-inductive argument \citep{kozen_coinductive_2006}. We will show that if $|q^{\pi_e}(s_1,a_1) - q^{\pi_e}(s_2,a_2)|\leq d(s_1,a_1,;s_2,a_2)$ holds true for some specific symmetric $d\in\mathbb{R}^{\mathcal{X}\times\mathcal{X}}$, then the statement also holds true for $\mathcal{F}^{\pi_e}(d)$, which means it will hold for $d_{\pi_e}$.

We have that for any $(s,a)\in\mathcal{X}$, $\max_{s,a} \frac{-|r(s,a)|}{1 - \gamma} \leq q^{\pi_e}(s,a) \leq \max_{s,a} \frac{|r(s,a)|}{1 - \gamma}$. Thus, for any $(s_1,a_1), (s_2,a_2)\in\mathcal{X}$, we have that $|q^{\pi_e}(s_1,a_1) - q^{\pi_e}(s_2,a_2)|\leq 2\max_{s,a} \frac{|r(s,a)|}{1 - \gamma}$. We can then assume that our specific symmetric $d$ is the constant function $d(s_1, a_1; s_2, a_2) = 2\max_{s,a} \frac{|r(s,a)|}{1 - \gamma}$, which satisfies our requirement that $|q^{\pi_e}(s_1,a_1) - q^{\pi_e}(s_2,a_2)|\leq d(s_1,a_1,;s_2,a_2)$.

Therefore, we have $q^{\pi_e}(s_1, a_1) - q^{\pi_e}(s_2, a_2)$
    \begin{align*}
         &= r(s_1,a_1)-r(s_2,a_2) + \gamma \sum_{s_1'\in\sset}\sum_{a_1'\in\aset}P(s_1'|s_1,a_1)\pi_e(a_1'|s_1')q^{\pi_e}(s_1',a_1') - \gamma \sum_{s_2'\in\sset}\sum_{a_2'\in\aset}P(s_2'|s_2,a_2)\pi_e(a_2'|s_2')q^{\pi_e}(s_2',a_2')\\
         &\leq |r(s_1,a_1)-r(s_2,a_2)| + \gamma \sum_{s_1',s_2'\in\sset}\sum_{a_1',a_2'\in\aset}P(s_1'|s_1,a_1)\pi_e(a_1'|s_1')P(s_2'|s_2,a_2)\pi_e(a_2'|s_2')(q^{\pi_e}(s_1',a_1') - q^{\pi_e}(s_2',a_2'))\\
         &\overset{(a)}{\leq} |r(s_1,a_1)-r(s_2,a_2)| + \gamma \sum_{s_1',s_2'\in\sset}\sum_{a_1',a_2'\in\aset}P(s_1'|s_1,a_1)\pi_e(a_1'|s_1')P(s_2'|s_2,a_2)\pi_e(a_2'|s_2')d(s_1',a_1';s_2',a_2')\\
         &= \mathcal{F}^{\pi_e}(d)(s_1, a_1;s_2,a_2)
    \end{align*}
    where (a) follows from the induction hypothesis. Similarly, by symmetry, we can show that $q^{\pi_e}(s_2, a_2) - q^{\pi_e}(s_1, a_1) \leq \mathcal{F}^{\pi_e}(d)(s_1, a_1;s_2,a_2)$. Thus, we have it that $|q^{\pi_e}(s_1,a_1) - q^{\pi_e}(s_2,a_2)|\leq d_{\pi_e}(s_1,a_1,;s_2,a_2)$.
\end{proof}

\lemmaQbound*
\begin{proof}
The proof closely follows that of Lemma 8 of \cite{kemertas_robustmetric_2021}, which is in turn based on Theorem 5.1 of \cite{ferns_bisim_2004}. The main difference between their theorems and ours is that the former is based on state representations and the latter is based on optimal state-value functions, while ours is focused on state-action representations for $\pieval$. 

We first remark that this new aggregated MDP, $\widetilde{\mathcal{M}}$, can be viewed as a Markov reward process (MRP) where the "states" are aggregated state-action pairs of the original MDP, $\mathcal{M}$. We now define the reward function and transition dynamics of the clustered MRP $\widetilde{\mathcal{M}}$, where $|\phi(x)|$ is the size of the cluster $\phi(x)$. Note that $\PR$ denotes the probability of the event.
    \begin{equation*}
        \tilde{r}(\phi(x)) = \frac{1}{|\phi(x)|}\sum_{y\in\phi(x)}r(y)
    \end{equation*}
    \begin{equation*}
        \widetilde{P}(\phi(x') | \phi(x)) = \frac{1}{|\phi(x)|}\sum_{y\in\phi(x)}\PR(\phi(x')| y)
    \end{equation*}
    Then we have: $| q^{\pieval}(x) - \tilde{q}^{\pieval}(\phi(x))|$
    \begin{align*}
        &= \left\lvert r(x) - \tilde{r}(\phi(x)) + \gamma \sum_{x'\in\saset}P(x'|x)q^{\pieval}(x') - \gamma \sum_{\phi(x')\in\abssaset}\widetilde{P}(\phi(x')|\phi(x))\tilde{q}^{\pieval}(\phi(x'))\right\rvert\\
        &\overset{(a)}{=} \left\lvert r(x) - \frac{1}{|\phi(x)|}\sum_{y\in\phi(x)}r(y)+ \gamma \sum_{x'\in\saset}P(x'|x)q^{\pieval}(x') - \gamma \frac{1}{|\phi(x)|}\sum_{\phi(x')\in\abssaset}\sum_{y\in\phi(x)}\PR(\phi(x')| y)\tilde{q}^{\pieval}(\phi(x'))\right\rvert\\
        &\overset{(b)}{=} \frac{1}{|\phi(x)|}\left\lvert |\phi(x)|r(x) - \sum_{y\in\phi(x)}r(y)+ \gamma |\phi(x)|\sum_{x'\in\saset}P(x'|x)q^{\pieval}(x') - \gamma \sum_{\phi(x')\in\abssaset}\sum_{y\in\phi(x)}\PR(\phi(x')| y)\tilde{q}^{\pieval}(\phi(x'))\right\rvert\\
        &\overset{(c)}{=} \frac{1}{|\phi(x)|}\left\lvert \sum_{y\in\phi(x)}(r(x) - r(y))+ \sum_{y\in\phi(x)}\left(\gamma \sum_{x'\in\saset}P(x'|x)q^{\pieval}(x') - \gamma \sum_{\phi(x')\in\abssaset}\PR(\phi(x')| y)\tilde{q}^{\pieval}(\phi(x'))\right)\right\rvert\\
        &\overset{(d.1)}{\leq} \frac{1}{|\phi(x)|}\sum_{y\in\phi(x)}\left(\lvert r(x) - r(y)\rvert + \gamma \left\lvert\sum_{x'\in\saset}P(x'|x)q^{\pieval}(x') - \sum_{\phi(x')\in\abssaset}\PR(\phi(x')|y)\tilde{q}^{\pieval}(\phi(x'))\right\rvert\right)\\
        &\overset{(d.2)}{=} \frac{1}{|\phi(x)|}\sum_{y\in\phi(x)}\left(\lvert r(x) - r(y)\rvert + \gamma \left\lvert\sum_{x'\in\saset}P(x'|x)q^{\pieval}(x') - \sum_{\phi(x')\in\abssaset}\sum_{z\in\phi(x')}P(z|y)\tilde{q}^{\pieval}(\phi(x'))\right\rvert\right)\\
        &\overset{(d.3)}{=} \frac{1}{|\phi(x)|}\sum_{y\in\phi(x)}\left(\lvert r(x) - r(y)\rvert + \gamma \left\lvert\sum_{x'\in\saset}P(x'|x)q^{\pieval}(x') - \sum_{x'\in\saset}P(x'|y)\tilde{q}^{\pieval}(\phi(x'))\right\rvert\right)\\
        &\overset{(e)}{\leq} \frac{1}{|\phi(x)|}\sum_{y\in\phi(x)}\left(\lvert r(x) - r(y)\rvert + \gamma \left\lvert\sum_{x'\in\saset}\left(P(x'|x)q^{\pieval}(x') - P(x'|y)\tilde{q}^{\pieval}(\phi(x'))\right)\right\rvert\right)\\
        &\overset{(f)}{\leq} \frac{1}{|\phi(x)|}\sum_{y\in\phi(x)}\left(\lvert r(x) - r(y)\rvert + \gamma \left\lvert\sum_{x'\in\saset}\left(P(x'|x)q^{\pieval}(x') - P(x'|y)q^{\pieval}(x')\right)\right\rvert \right)\\
        & + \frac{\gamma}{|\phi(x)|}\sum_{y\in\phi(x)}\left(\left\lvert\sum_{x'\in\saset}P(x'|y)(q^{\pieval}(x') - \tilde{q}^{\pieval}(\phi(x')))\right\rvert \right)\\
        &\overset{(g)}{\leq} \frac{1}{|\phi(x)|}\sum_{y\in\phi(x)}\left(\lvert r(x) - r(y)\rvert + \gamma \left\lvert\sum_{x'\in\saset}\left(P(x'|x)- P(x'|y)\right)q^{\pieval}(x')\right\rvert + \gamma \left\Vert q - \widetilde{q}\right\Vert_{\infty} \right)\\
        &\overset{(h)}{=} \frac{1}{|\phi(x)|}\sum_{y\in\phi(x)}\left(\lvert r(x) - r(y)\rvert + \gamma \left\lvert\mathbb{E}_{x'\sim P(\cdot|x)}[q^{\pieval}(x')] - \mathbb{E}_{x'\sim P(\cdot|y)}[q^{\pieval}(x')]\right\rvert + \gamma \left\Vert q - \widetilde{q}\right\Vert_{\infty} \right)
    \end{align*}

    where (a) is due to the definition of $\absrf$ and $\widetilde{P}$, (b) is due to multiplying and dividing by $|\phi(x)|$, (c) is due to re-arranging terms, (d.1) is due to Jensen's inequality, (d.2 and d.3) are disaggregating the sums over clustered state-actions into sums over original state-actions by expanding $\PR(\phi(x')|y) = \sum_{x\in\phi(x')}P(x|y)$ for each clustered state-action, $\phi(x')$, (e) is grouping the terms, (f) is by adding and subtracting $\frac{1}{|\phi(x)|} \sum_{y\in\phi(x)}P(x'|y)q^{\pieval}(x')$, (g) is since the infinity norm of the difference of the action-values is greater than the expected difference, (h) is re-writing the expression in terms of expectations. 
    
    From Theorem \ref{thm:dbound} we know $q^{\pi_e}$ is $1$-Lipschitz with respect to the distance function $d_{\pi_e}$. Notice that (h) contains the dual formulation of the Wasserstein distance where $f = q^{\pieval}$ (see Definition \ref{def:dual_wass}). We can then re-write (h) in terms of original definition of the Wasserstein distance:
    \begin{align*}
        | q^{\pieval}(x) - \tilde{q}^{\pieval}(\phi(x))| &\overset{}{\leq} \frac{1}{|\phi(x)|}\sum_{y\in\phi(x)}\left(\lvert r(x) - r(y)\rvert + \gamma W(d_{\pi_e})(P(\cdot|x), P(\cdot|y)) + \gamma \left\Vert q - \widetilde{q}\right\Vert_{\infty} \right)\\
        &\overset{(i)}{\leq} \frac{1}{|\phi(x)|}\sum_{y\in\phi(x)}\left(\lvert r(x) - r(y)\rvert + \gamma D_{\text{LK}}(d_{\pi_e})(x', y') + \gamma \left\Vert q - \widetilde{q}\right\Vert_{\infty} \right)\\
        &\overset{(j)}{=} \frac{1}{|\phi(x)|}\sum_{y\in\phi(x)}\left(\lvert r(x) - r(y)\rvert + \gamma \E_{x'\sim\PR^{\pi_e}, y'\sim\PR^{\pi_e}}[d_{\pi_e}(x', y')] + \gamma \left\Vert q - \widetilde{q}\right\Vert_{\infty} \right)\\
        &\overset{(k)}{=} \frac{1}{|\phi(x)|}\sum_{y\in\phi(x)}\left(d_{\pieval}(x, y) + \gamma \left\Vert q - \widetilde{q}\right\Vert_{\infty} \right)\\
        &\overset{(l)}{\leq} 2\epsilon + \gamma \left\Vert q - \widetilde{q}\right\Vert_{\infty} \\
        | q^{\pieval}(x) - \tilde{q}^{\pieval}(\phi(x))| &\overset{(m)}{\leq} \frac{2\epsilon}{1 - \gamma}, \forall x\in\saset
    \end{align*}
    where (i) is due the fact that the Łukaszyk–Karmowski, $D_{\text{LK}}$, upper bounds the Wasserstein distance, (j) is using Definition \ref{def:lk_dist}, (k) is due to the definition of $d_{\pieval}$, and (l) is due the fact that the maximum distance between any two $x, y \in\phi(x)$ is at most $2\epsilon$, which is greater than the average distance between any one point to every other point in the cluster, and (m) is due to $\left\Vert q - \widetilde{q}\right\Vert_{\infty} \leq \frac{2\epsilon}{1 - \gamma}$.
\end{proof}

\thmJbound*
\begin{proof}
    From Lemma \ref{lemma:qbound} we have that $| q^{\pieval}(s_0,a_0) - q^{\pieval}(\phi(s_0,a_0))| \leq \frac{2\epsilon}{(1 - \gamma)}$.
    \begin{align*}
        \big| \E_{s_0,a_0\sim\pieval}[q^{\pieval}(s_0,a_0)] - \E_{s_0,a_0\sim\pieval}[q^{\pieval}(\phi(s_0,a_0))]\big| &= |\E_{s_0,a_0\sim\pieval}[q^{\pieval}(s_0,a_0) - q^{\pieval}(\phi(s_0,a_0))]| \\
        &\overset{(a)}{\leq} \E_{s_0,a_0\sim\pieval}[\left|q^{\pieval}(s_0,a_0) - q^{\pieval}(\phi(s_0,a_0))\right|] \\
        &\overset{(b)}{\leq} \E_{s_0,a_0\sim\pieval}\left(\frac{2\epsilon}{1 - \gamma}\right) \\
        &= \frac{2\epsilon}{(1 - \gamma)},
    \end{align*}
    where (a) follows from Jensen's inequality and (b) follows from Lemma \ref{lemma:qbound}.
\end{proof}

\section{ROPE Pseudo-code}
\label{sec:pseudocode}

 \begin{algorithm}[H]
  \caption{\textsc{rope+fqe}}
  \label{algo:batch_linear_td}
  \begin{algorithmic}[1]
    \STATE Input: policy to evaluate $\pi_e$, batch $\mathcal{D}$, encoder parameters class $\Omega$, action-value parameter class $\Xi$, encoder function $\phi:\sset\times\aset\to\mathbb{R}^d$, action-value function $q:\sset\times\aset\to\mathbb{R}$.
    \STATE $\hat{\omega} := \arg\min_{\omega\in\Omega}$\\\hspace{1cm}$\mathbb{E}_{(s_1,a_1,s_1'), (s_2,a_2,s_2')\sim\mathcal{D}}\left[\rho\left(\left|r(s_1,a_1) - r(s_2,a_2)\right| + \gamma \E_{a_1', a_2'\sim\pieval}[\tilde{d}_{\bar{\omega}}(s_1',a_1'; s_2',a_2')] - \tilde{d}_{\omega}(s_1,a_1; s_2,a_2)\right)\right]$ 
    
    \COMMENT {\textsc{rope} training phase; where $\tilde{d}_{\omega}(s_1, a_1; s_2, a_2) \coloneqq \frac{||\phi_\omega(s_1,a_1)||_2^2 + ||\phi_\omega(s_2,a_2)||_2^2}{2} + \beta \theta(\phi_\omega(s_1, a_1), \phi_\omega(s_2, a_2))$, $\bar{\omega}$ are fixed parameters of target network, and $\rho$ is the Huber loss. See Section 3.1 for more details.}
    \STATE $\hat{\xi} := \arg\min_{\xi\in\Xi}\E_{(s,a,s')\sim\mathcal{D}}\left[\rho\left(r(s,a) + \gamma \E_{a'\sim\pieval} [q_{\bar{\xi}}(\phi_{\hat{\omega}}(s',a'))] - q_{\xi}(\phi_{\hat{\omega}}(s,a))\right)\right]$ \COMMENT {\textsc{fqe} using fixed encoder $\phi_{\hat{\omega}}$ from Step 2, where $\rho$ is the Huber loss.}
    \STATE Return $q_{\hat{\xi}}$ \COMMENT {Estimated action-value function of $\pi_e$, $q^{\pi_e}$.}
  \end{algorithmic}
\end{algorithm}

\section{Empirical Results}
\label{app:empirical_res}

We now include additional experiments that were deferred from the main text.

\subsection{Gridworld Visualizations}
In Section \ref{sec:gw_sa_metrics}, we visualize how \textsc{rope} and on-policy \textsc{mico} group state-actions pairs. We now consider two additional metrics that group state-action pairs:
\begin{enumerate}
    \item Policy similarity metric \citep{agarwal_psm_2021}: $d_{\text{PSM}}(s_1, a_1; s_2, a_2) := |\pieval(a_1|s_1) - \pieval(a_2|s_2)| + \gamma\E_{a_1',a_2'\sim \pieval}[d_{\text{PSM}}((s_1', a_1'), (s_2', a_2'))]$. This metric measures short- and long-term similarity based on how $\pieval$ acts in different states, not in terms of the rewards and returns it receives.
    \item Random policy similarity metric \citep{dadashi_ploff_2021}: $d_{\text{RAND}}(s_1, a_1; s_2, a_2) := |r(s_1,a_1)-r(s_2,a_2)| + \gamma\E_{a' \sim \mathcal{U}(\aset)}[d_{\text{RAND}}((s_1', a'), (s_2', a'))]$. Similar to $d_{\pieval}$, but considers behavior of a random policy that samples actions uniformly.
\end{enumerate}
\begin{figure*}[hbtp]
    \centering
        \subfigure[Action-values of $\pi_e$]{\includegraphics[scale=0.35]{results/gridworld/pie_q_values.jpg}}
        \subfigure[$d_{\text{PSM}}$]{\includegraphics[scale=0.35]{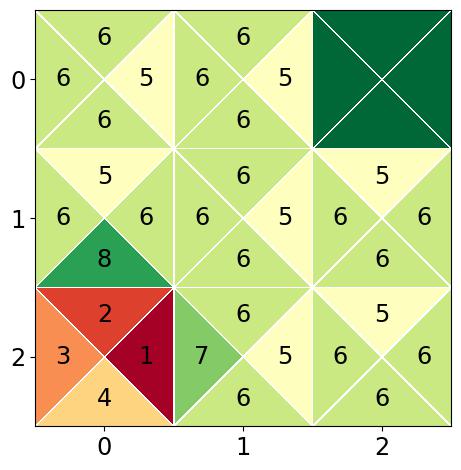}}
        \subfigure[$d_{\text{RAND}}$]{\includegraphics[scale=0.35]{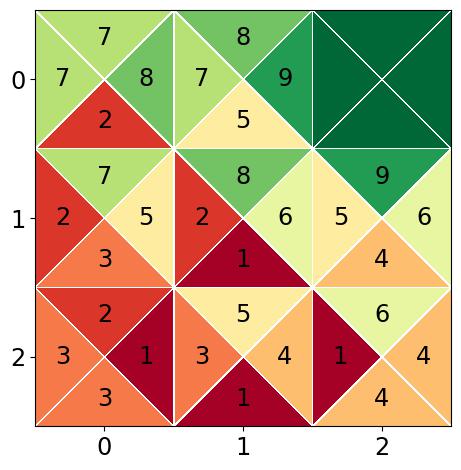}}
    \caption{\footnotesize Figure (a): $q^{\pi_e}$ for $\pi_e$. Center and right: group clustering according to \textsc{psm} (Figure (b)) and random-policy metric (Figure (c)) (center number in each triangle is group ID). Two state-action pairs are grouped together if their distance according to the specific metric is $0$. The top right cell is blank since it is the terminal state, which is not grouped.}
    \label{fig:supp_gw_grouping_visual}
\end{figure*}

From Figure \ref{fig:supp_gw_grouping_visual}, we reach the same conclusion as we did in Section \ref{sec:gw_sa_metrics}: that existing state-action similarity metrics are unsuitable for learning $q^{\pieval}$ due to how they group state-action pairs.

\subsection{Deep OPE Experiments}

We now present additional details on our empirical setup and additional experiments.

\subsubsection{Additional Empirical Setup Details}

Before applying any of the algorithms, we normalize the states of the dataset to make the each feature dimension have $0$ mean and $1$ standard deviation.

\paragraph{FQE Training Details} In all experiments and all datasets, we use a neural network as \textsc{fqe}'s action-value function with $2$ layers and $256$ neurons using \textsc{relu} activation function. We use mini-batch gradient descent to train the \textsc{fqe} network with mini-batch sizes of $512$ and for $300$K gradient steps. We use the Adam optimizer with learning rate $1e^{-5}$ and weight decay $1e^{-2}$. \textsc{fqe} minimizes the Huber loss. The only changes for \textsc{fqe-deep} are that it uses a neural network size of $4$ layers with $256$ neurons and trains for $500$K gradient steps. Preliminary results with lower learning rates such as $5e^{-6}$ and $1e^{-6}$ did not make a difference. \textsc{fqe} uses an exponentially-moving average target network with $\tau = 0.005$ updated every epoch.

\paragraph{ROPE and BCRL Details} In all experiments and datasets, we use a neural network as the state-action encoder for \textsc{rope} with $2$ layers and $256$ neurons with the \textsc{relu} activation. We use mini-batch gradient descent to train the the encoder network with mini-batch sizes of $512$ and for $300$K gradient steps.
For $\textsc{rope}$ and \textsc{bcrl}, we hyperparameter sweep the output dimension of the encoder. Additionally, for \textsc{rope}, we sweep over the angular distance scalar, $\beta$.  For the output dimension, we sweep over dimensions: $\{|X|/3,|X|/2, |X|\}$, where $|X|$ is the dimension of the original state-action space of the environment. For $\beta$, we sweep over $\{0.1, 1, 10\}$.  The best performing hyperparameter set is the one that results in lowest \textsc{rmae} (from $\rho(\pieval)$) at the end of \textsc{fqe} training. \textsc{rope} uses an exponentially-moving average target network with $\tau = 0.005$ updated every epoch. Finally, the output of \textsc{rope}'s encoder is fed through a LayerNorm \citep{ba_layernorm_2016} layer, followed by a \textsc{tanh} layer. \textsc{rope} minimizes the Huber loss.

When computing $d^{\pieval}\approx \tilde{d}_{\omega}$ \textsc{rope} uses the same procedure as \textsc{mico} (appendix C.2. of \cite{castro_mico_2022}):
\begin{equation*}
    \tilde{d}_{\omega}(s_1, a_1; s_2, a_2) \coloneqq \frac{||\phi_\omega(s_1,a_1)||_2^2 + ||\phi_{\bar{\omega}}(s_2,a_2)||_2^2}{2} + \beta \theta(\phi_\omega(s_1, a_1), \phi_{\bar{\omega}}(s_2, a_2)) 
\end{equation*}
where it applies the target network parameters, $\bar{\omega}$, on the $(s_2, a_2)$ pair for stability. For the angular distance $\theta(\phi_\omega(s_1, a_1), \phi_{\omega}(s_2, a_2))$, we have the cosine-similarity and the angle as below. Note in practice, for numerical stability, a small constant (e.g. $1e^{-6}$ or $5e^{-5}$) may have to be added when computing the square-root.
\begin{align*}
    \text{CS}(\phi_\omega(s_1, a_1), \phi_{\omega}(s_2, a_2)) &= \frac{\langle\phi_\omega(s_1, a_1),\phi_\omega(s_2, a_2)\rangle}{||\phi_\omega(s_1,a_1)|| ||\phi_\omega(s_2,a_2)||}\\
    \theta(\phi_\omega(s_1, a_1), \phi_{\omega}(s_2, a_2)) &= \text{arctan2}\left(\sqrt{1 - \text{CS}(\phi_\omega(s_1, a_1), \phi_{\omega}(s_2, a_2))^2}, \text{CS}(\phi_\omega(s_1, a_1), \phi_{\omega}(s_2, a_2))\right)
\end{align*}

\paragraph{Custom Datasets} 
We generate the datasets by training policies in the environment using \textsc{sac} \citep{haarnoja_sac_2018} and take the final policy at the end of training as $\pieval$ and we use an earlier policy with lower performance as the behavior policy. The expected discounted return of the policies and datasets for each domain is given in Table \ref{tab:pi_vals} ($\gamma = 0.99$). The values for the evaluation and behavior policies were computed by running each for $300$ rollout trajectories, which was more than a sufficient amount for the estimate to converge, and averaging the discounted return (note that \cite{chang_learning_2022} use $200$ rollout trajectories).

\begin{table}[H]
\centering
\begin{tabular}{l|l|l|l|l|l|l}
&
$\rho(\pieval)$ &
$\rho(\pi_b)$  \\
HumanoidStandup &
$14500$ &
$13000$ \\
Swimmer &
$43$ &
$31$ \\
HalfCheetah &
$544$ &
$308$  \\
\end{tabular}
\caption{Policy values of the evaluation policy and behavior policy.}
\label{tab:pi_vals}
\end{table}

\paragraph{D4RL Datasets} Due to known discrepancy issues between newer environments of gym\footnote{\url{https://github.com/Farama-Foundation/D4RL/tree/master}}, we generat our datasets instead of using the publicly available ones. To generate the datasets, we use the publicly available policies \footnote{\url{https://github.com/google-research/deep_ope}}. For each domain, the expert and evaluation policy was the $10$th (last policy) from training. The medium and behavior policy was the $5$th policy. We added a noise of $0.1$ to the policies.

\subsubsection{FQE Training Iteration Curves for D4RL Datasets}

In this section, we include the remaining \textsc{fqe} training iteration curves (\textsc{ope} error vs. gradient steps) for the \textsc{d4rl} dataset (Figure~\ref{fig:fqe_tr_itr_d4rl}). We can see that  \textsc{fqe} diverges in multiple settings while \textsc{rope} is very stable. While \textsc{fqe-clip} does not diverge, it is still highly unstable.

\begin{figure*}[hbtp]
    \centering
        \subfigure[HalfCheetah-random]{\includegraphics[scale=0.125]{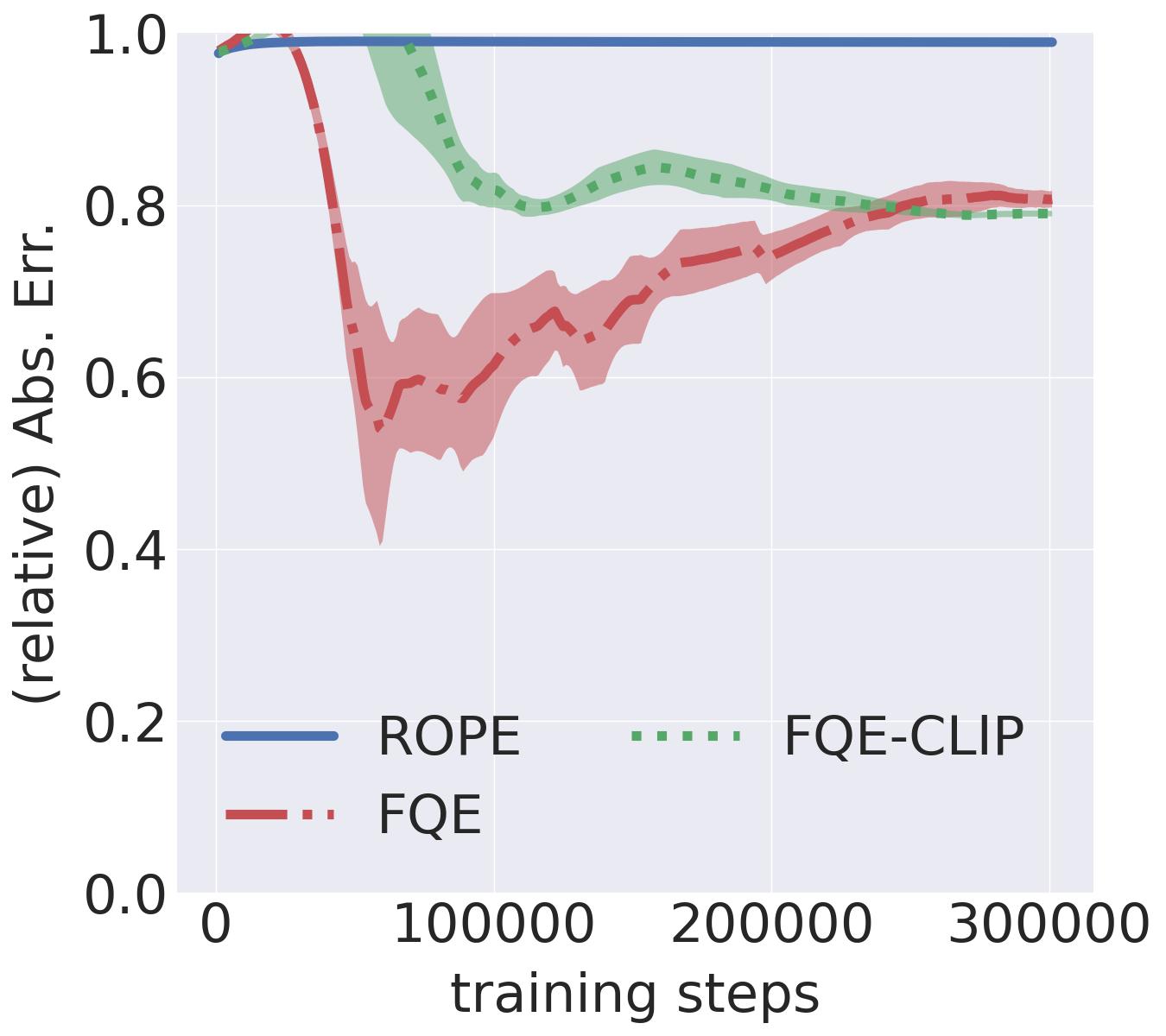}}
        \subfigure[HalfCheetah-medium]{\includegraphics[scale=0.125]{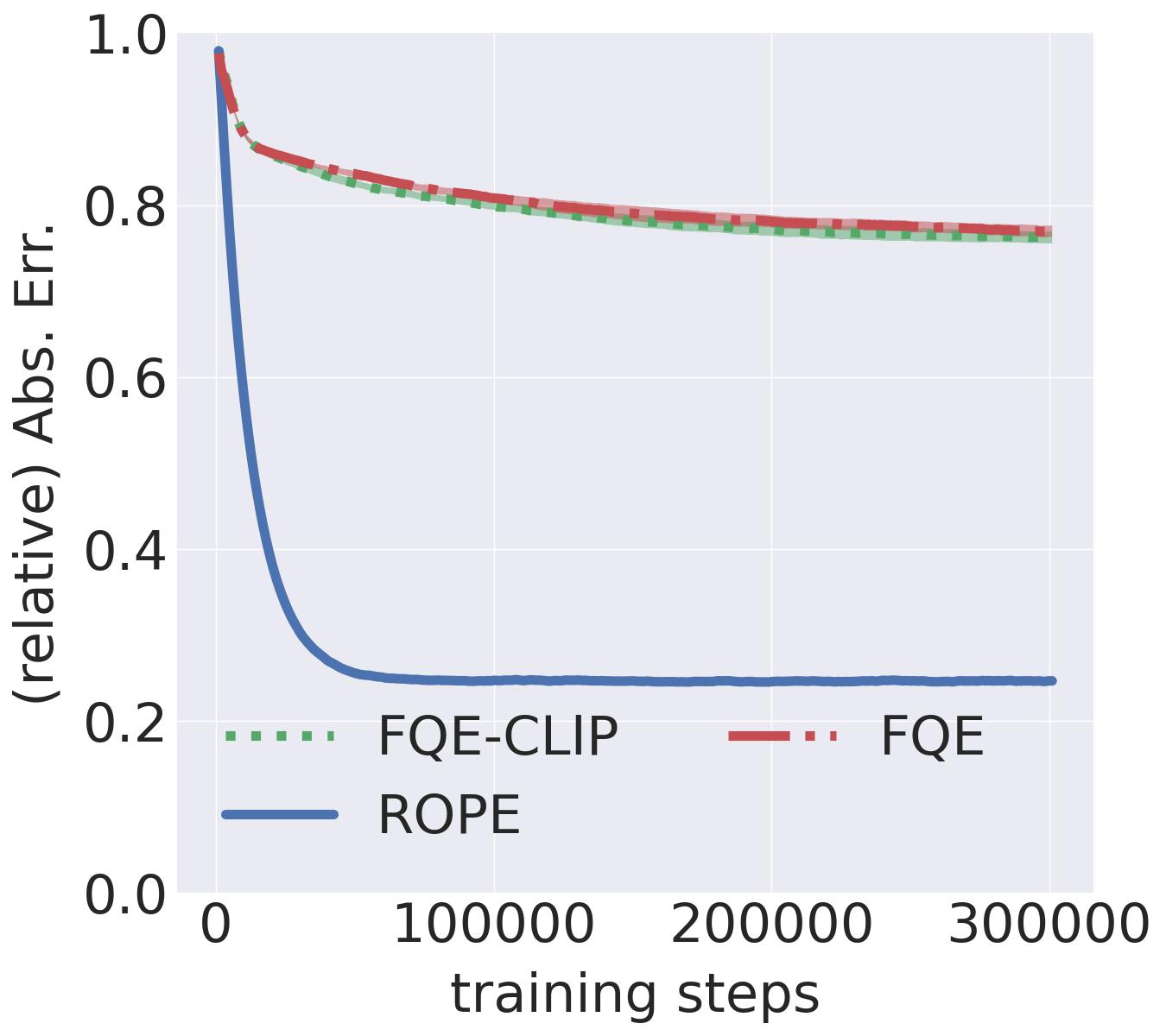}}
        \subfigure[HalfCheetah-medium-expert]{\includegraphics[scale=0.125]{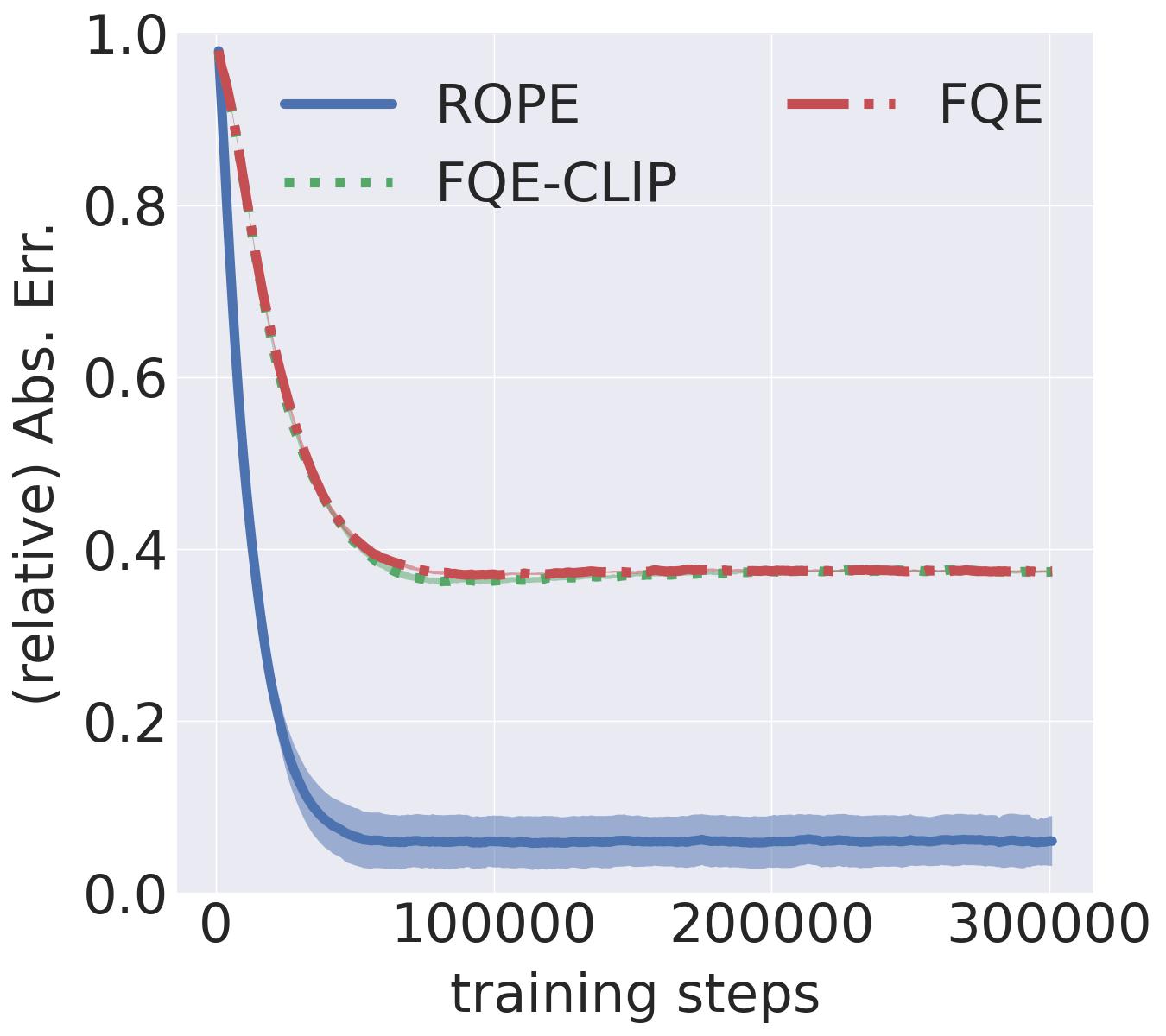}}\\
        \subfigure[Walker2D-random]{\includegraphics[scale=0.125]{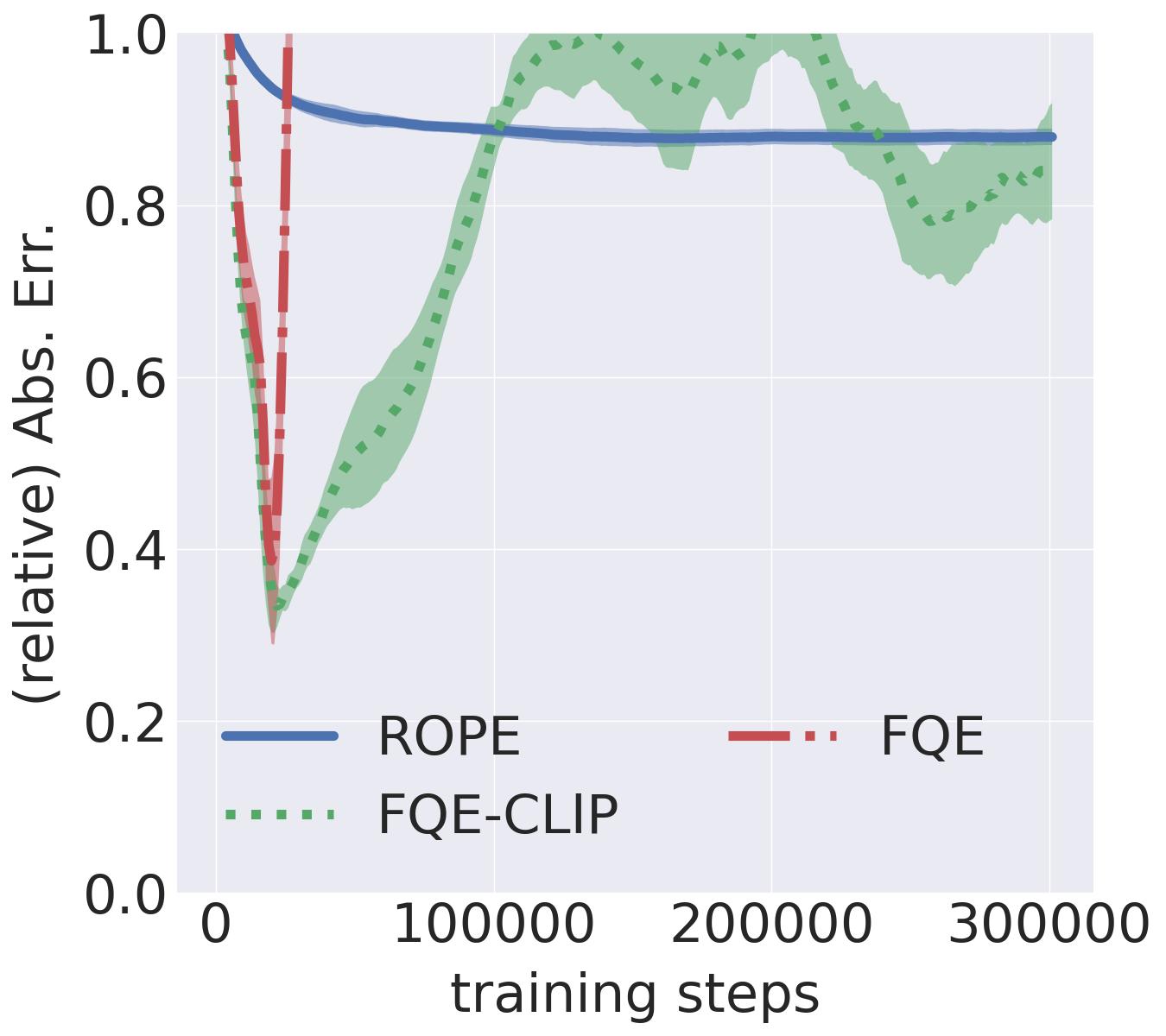}}
        \subfigure[Walker2D-medium]{\includegraphics[scale=0.125]{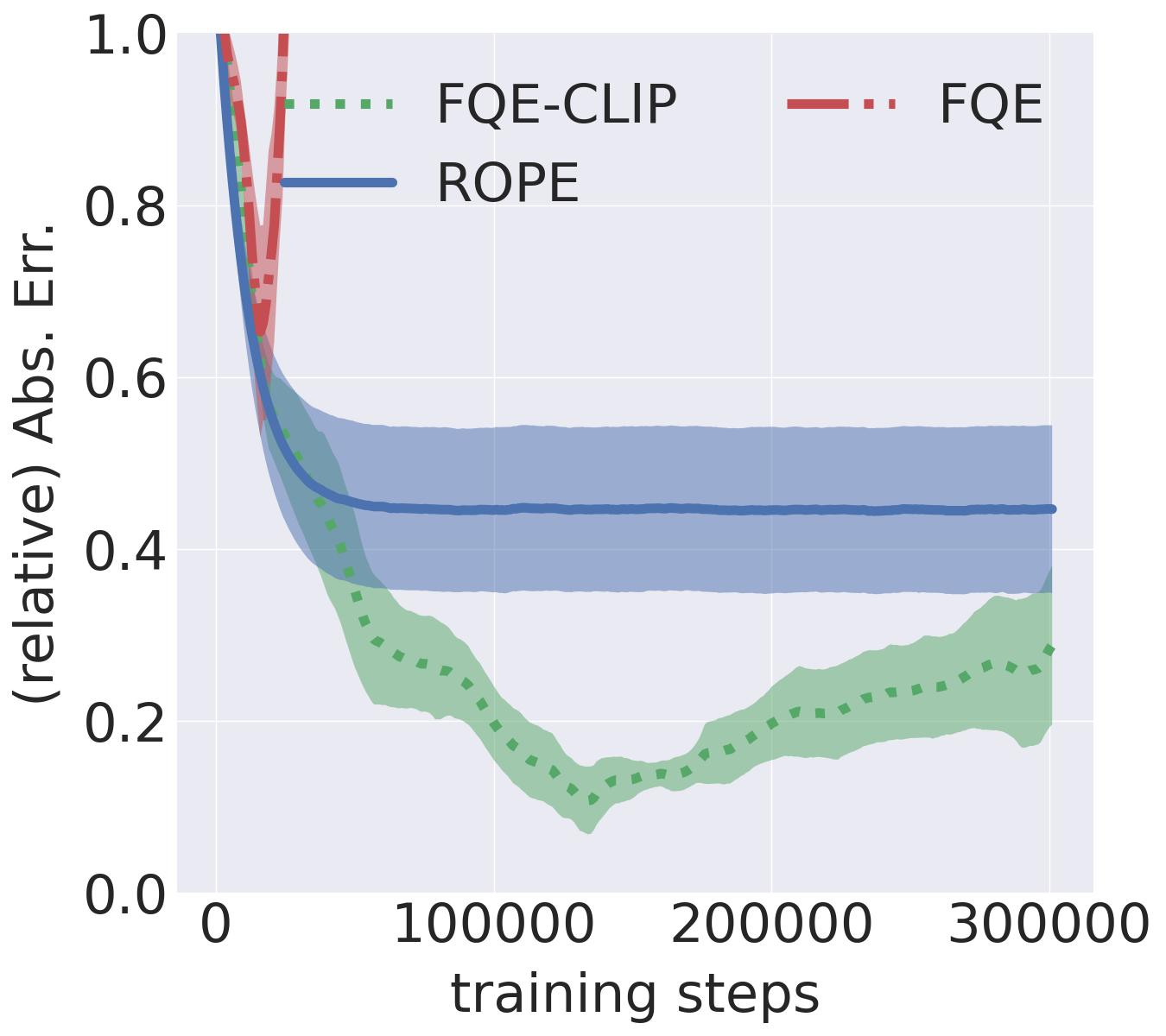}}
        \subfigure[Walker2D-medium-expert]{\includegraphics[scale=0.125]{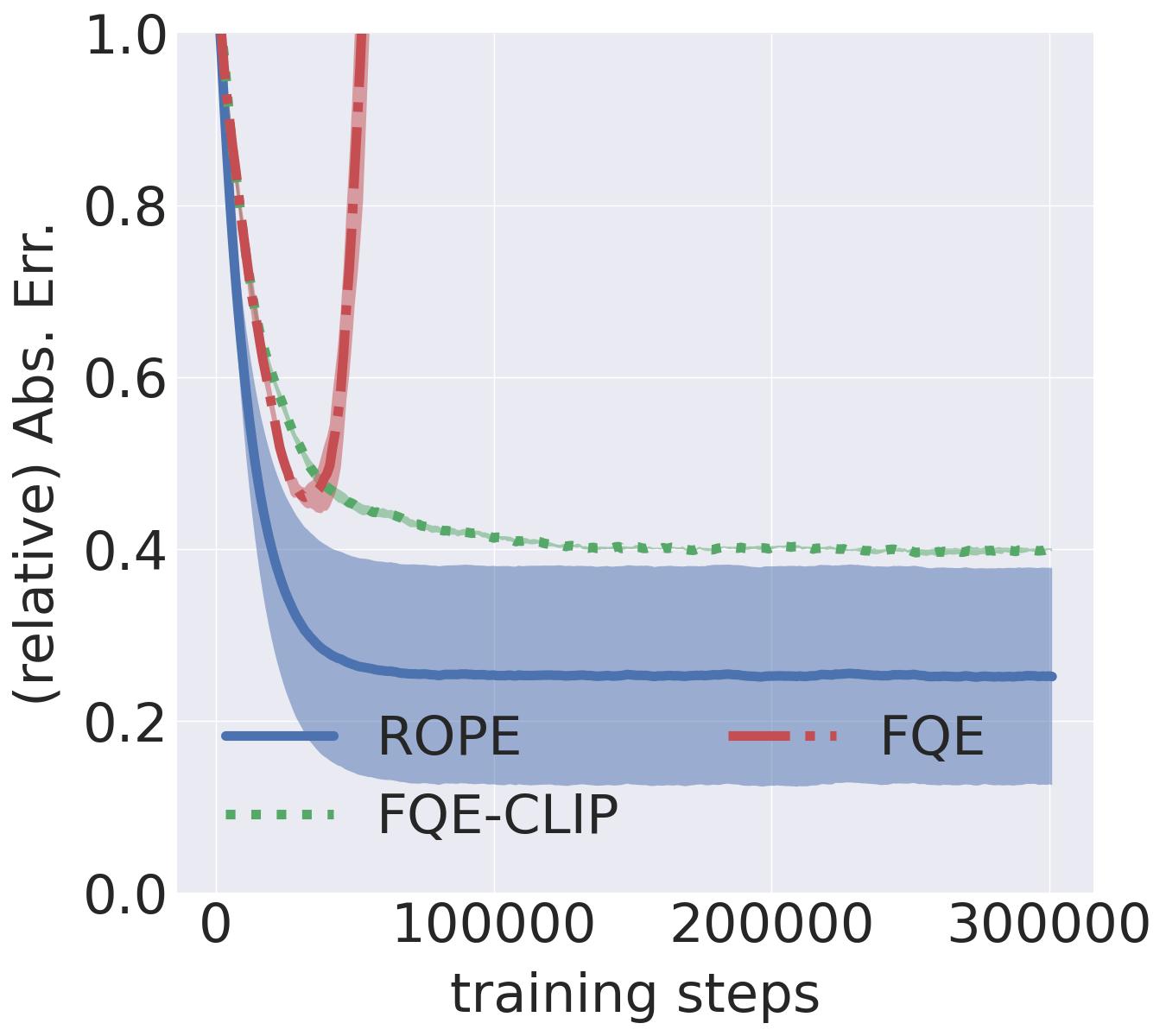}}\\
        \subfigure[Hopper-random]{\includegraphics[scale=0.125]{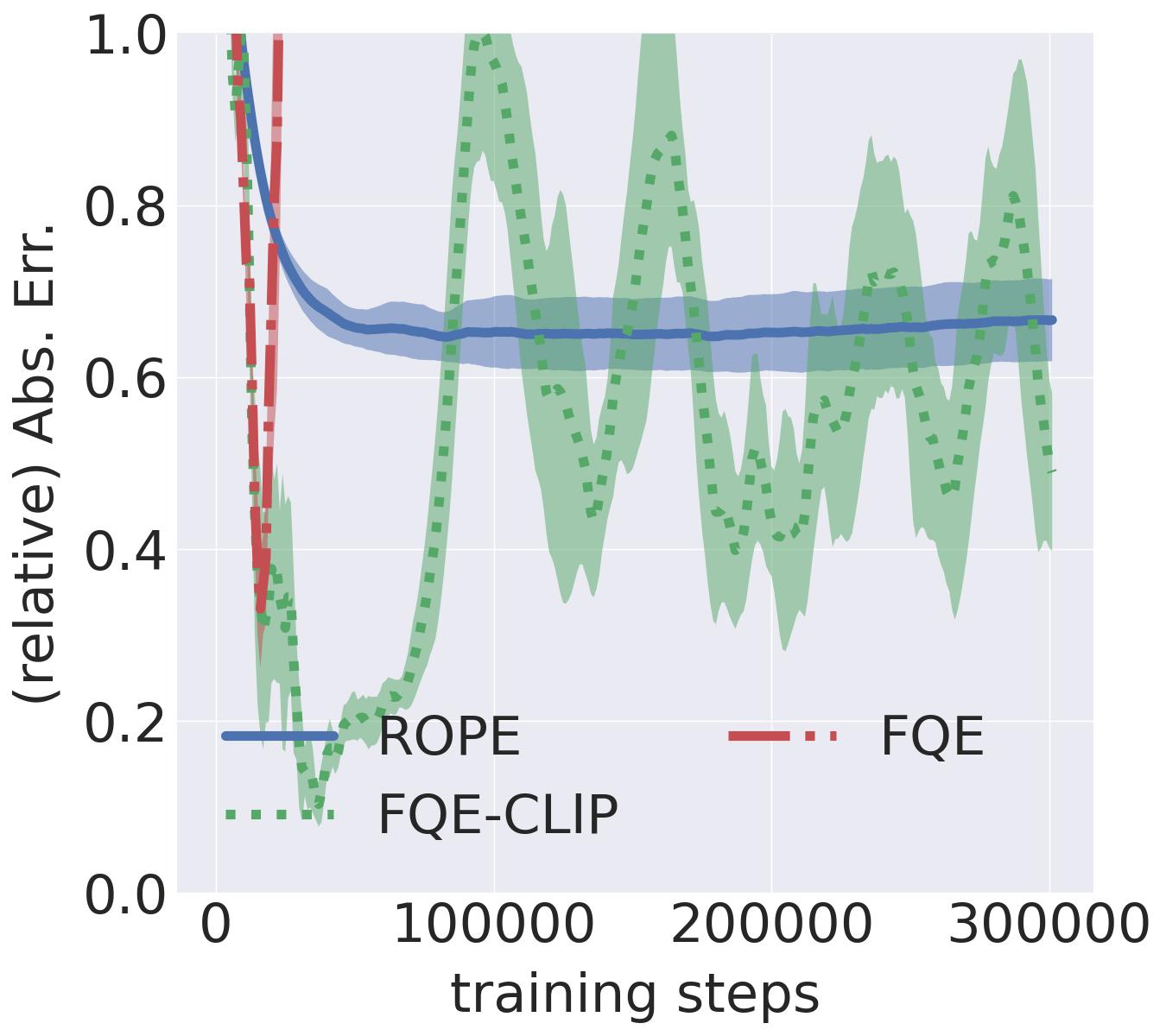}}
        \subfigure[Hopper-medium]{\includegraphics[scale=0.125]{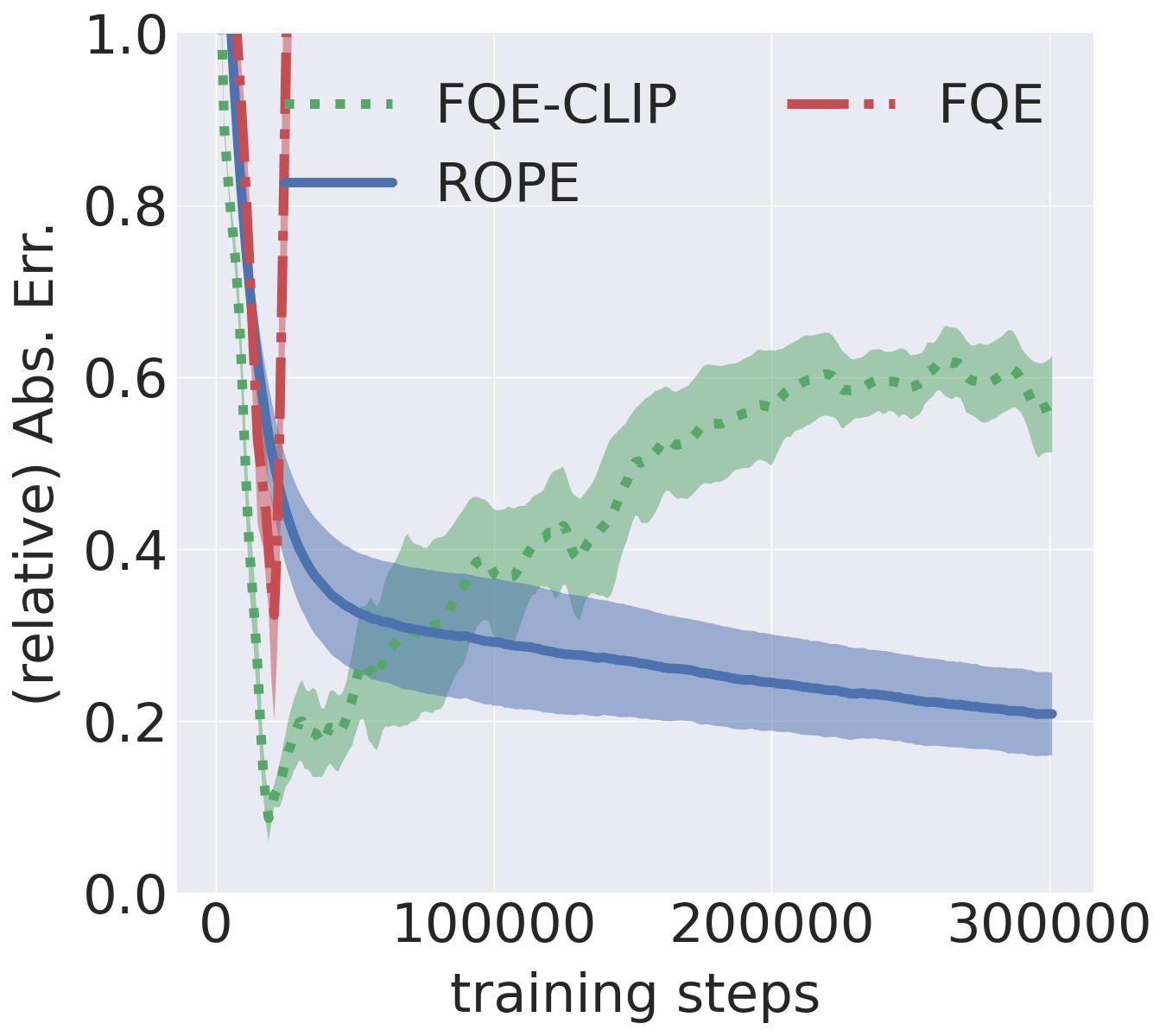}}
        \subfigure[Hopper-medium-expert]{\includegraphics[scale=0.125]{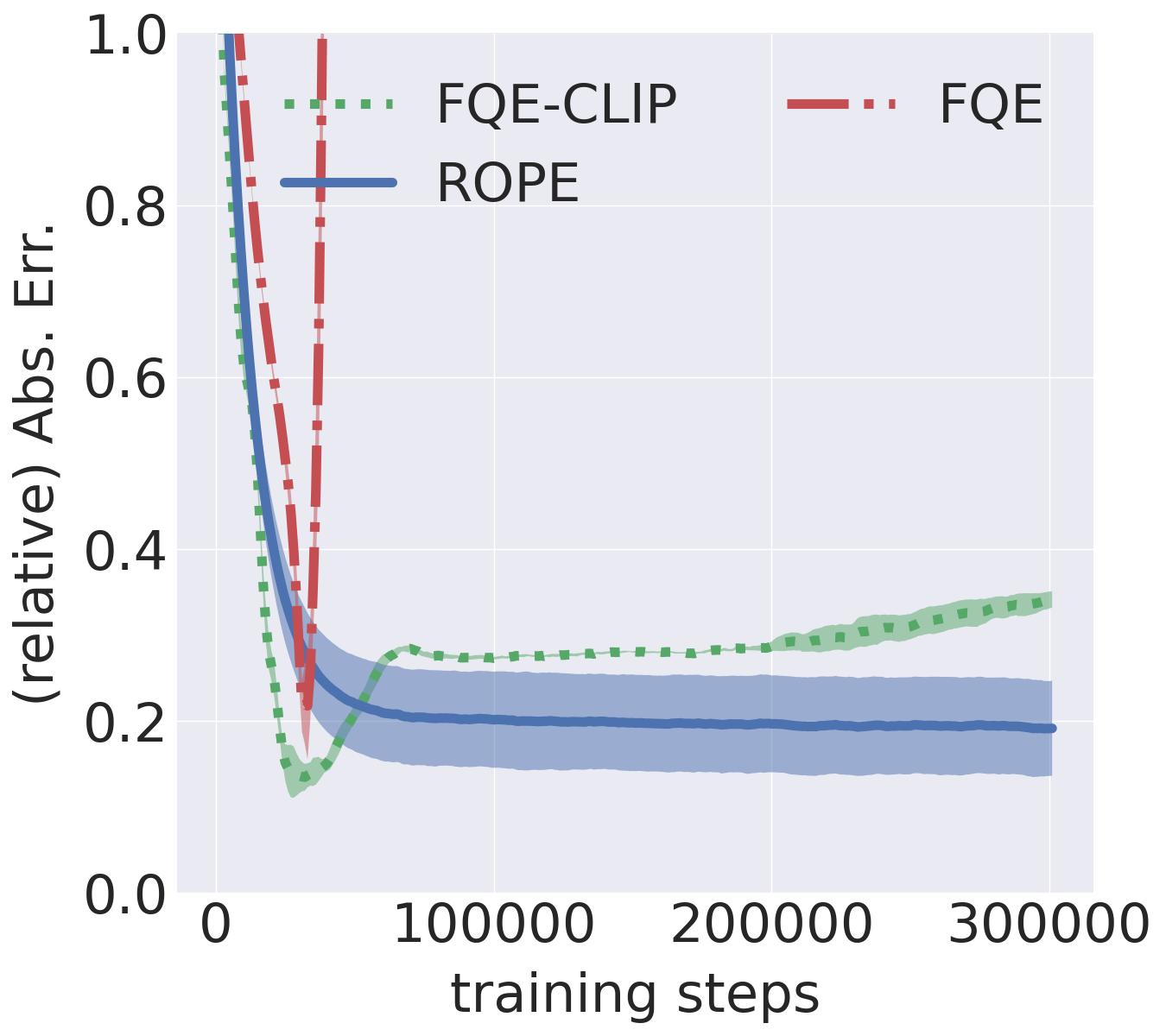}}\\
    \caption{\footnotesize \textsc{rmae} vs. training iterations of \textsc{fqe} on the \textsc{d4rl} datasets. \textsc{iqm} of errors for each domain were computed over $20$ trials with $95\%$ confidence intervals. Lower is better.}
    \label{fig:fqe_tr_itr_d4rl}
\end{figure*}

\subsubsection{Ablation: ROPE Hyperparameter Sensitivity}

Similar to the results in Section~\ref{sec:ablations}, we show \textsc{rope}'s hyperparameter sensitivity on all the custom and \textsc{d4rl} datasets. In general, we find that \textsc{rope} is robust to hyperparameter tuning, and it produces more data-efficient \textsc{ope} estimates than \textsc{fqe} for a wide variety of its hyperparameters. See Figures~\ref{fig:rope_abl_hp_start} to \ref{fig:rope_abl_hp_end}.

Note that in the bar graphs, we limit the vertical axis to $1$. In the Hopper and Walker \textsc{d4rl} experiments, \textsc{fqe} diverged and had an error significantly larger than $1$.

\begin{figure*}[hbtp]
    \centering
        \subfigure[Swimmer]{\includegraphics[scale=0.125]{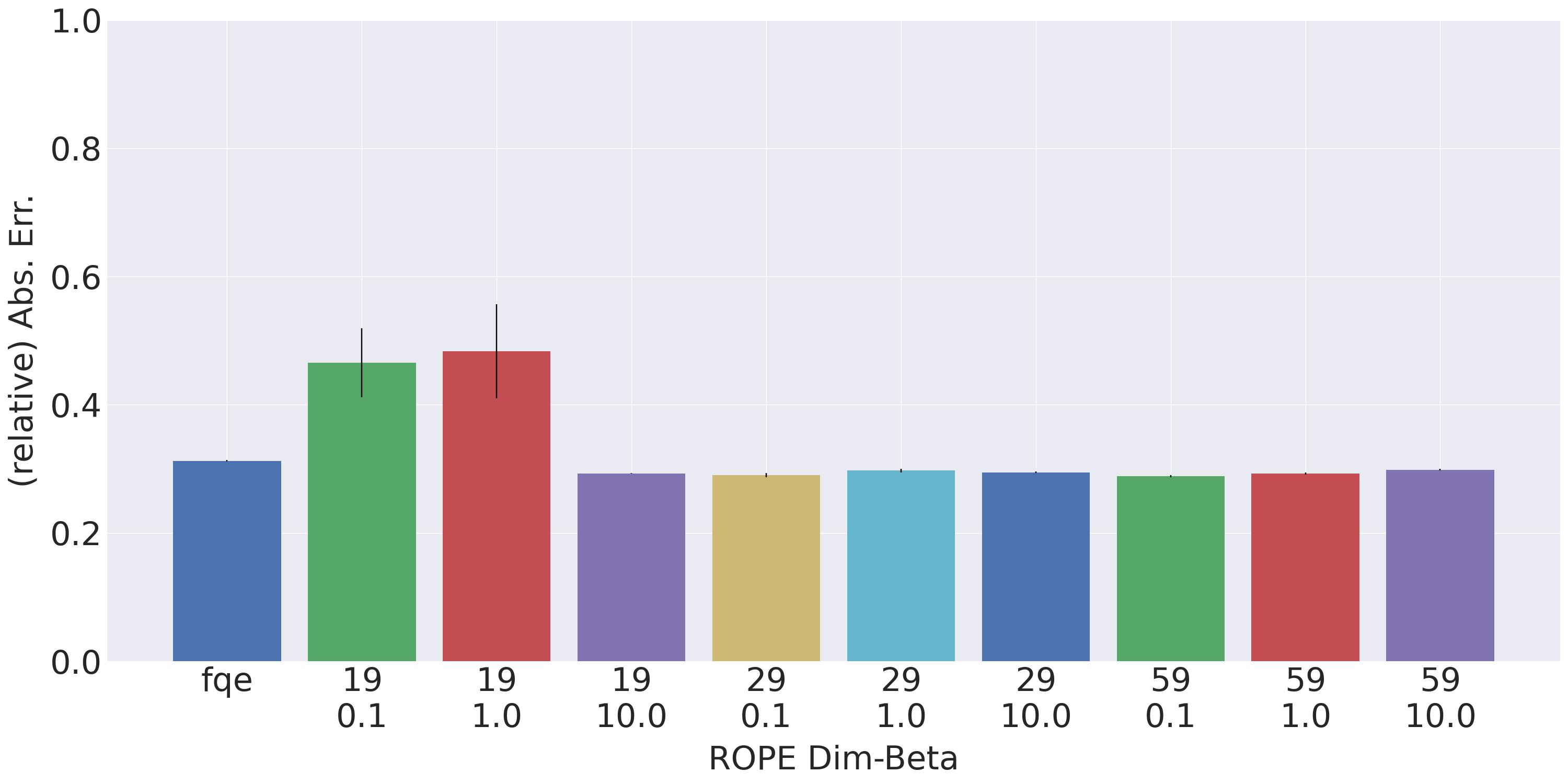}}
        \subfigure[HalfCheetah]{\includegraphics[scale=0.125]{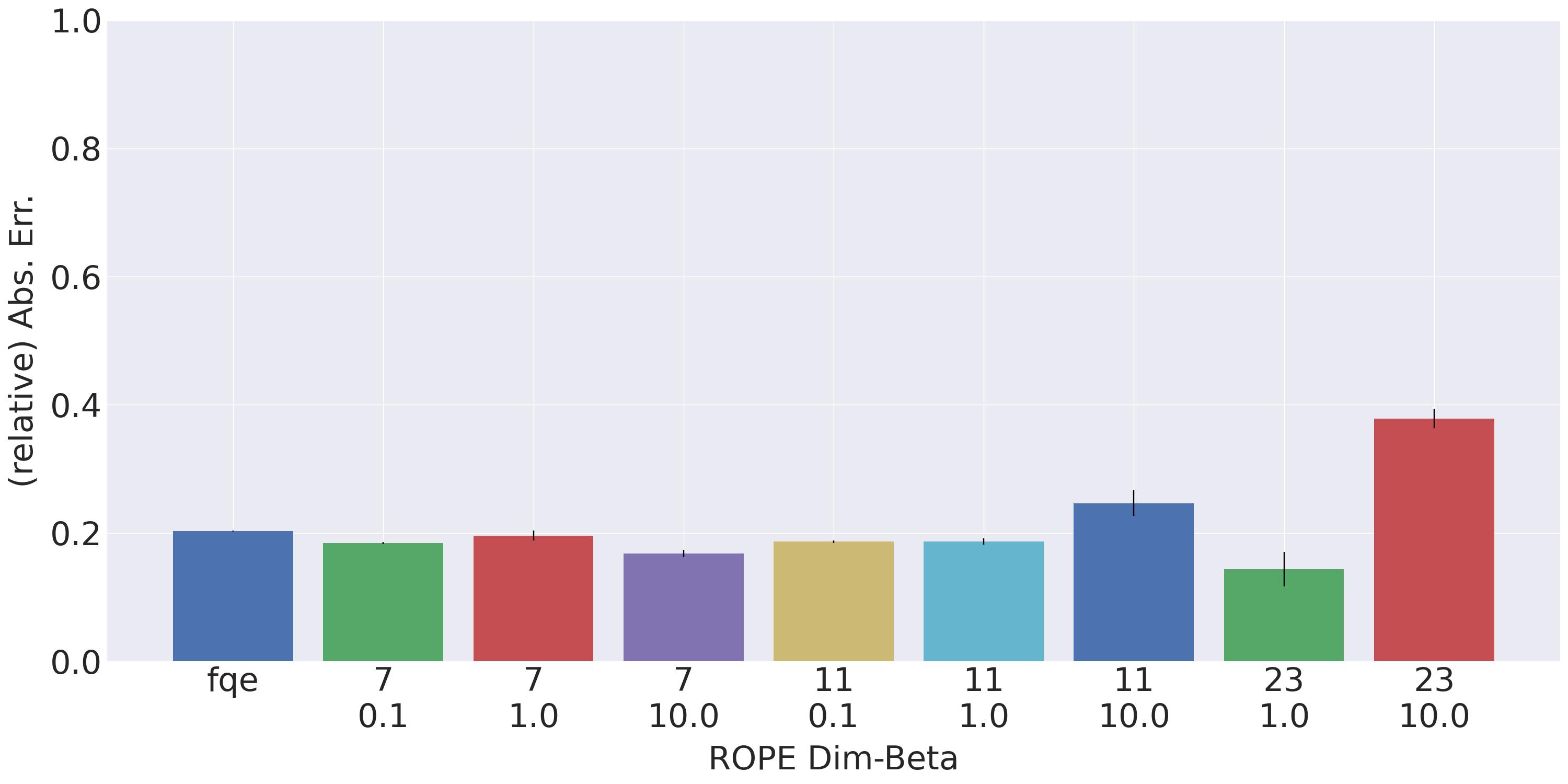}}
        \subfigure[HumanoidStandup]{\includegraphics[scale=0.125]{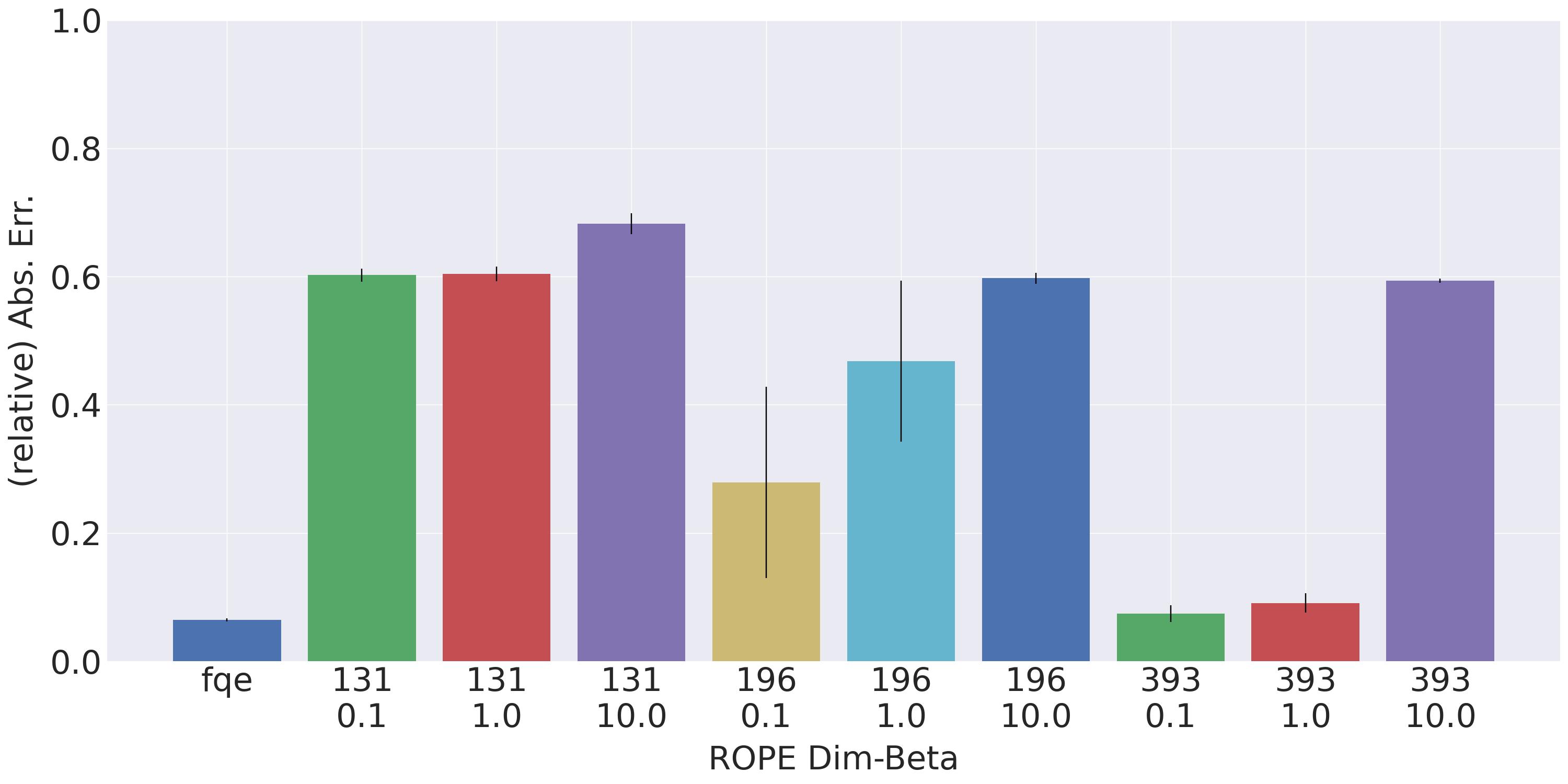}}
    \caption{\footnotesize \textsc{fqe} vs. \textsc{rope} when varying \textsc{rope}'s encoder output dimension (top) and $\beta$ (bottom) on the custom datasets.  \textsc{iqm} of errors are computed over $20$ trials with $95\%$ confidence intervals. Lower is better.}
    \label{fig:rope_abl_hp_start}
\end{figure*}

\begin{figure*}[hbtp]
    \centering
        \subfigure[HalfCheetah-random]{\includegraphics[scale=0.125]{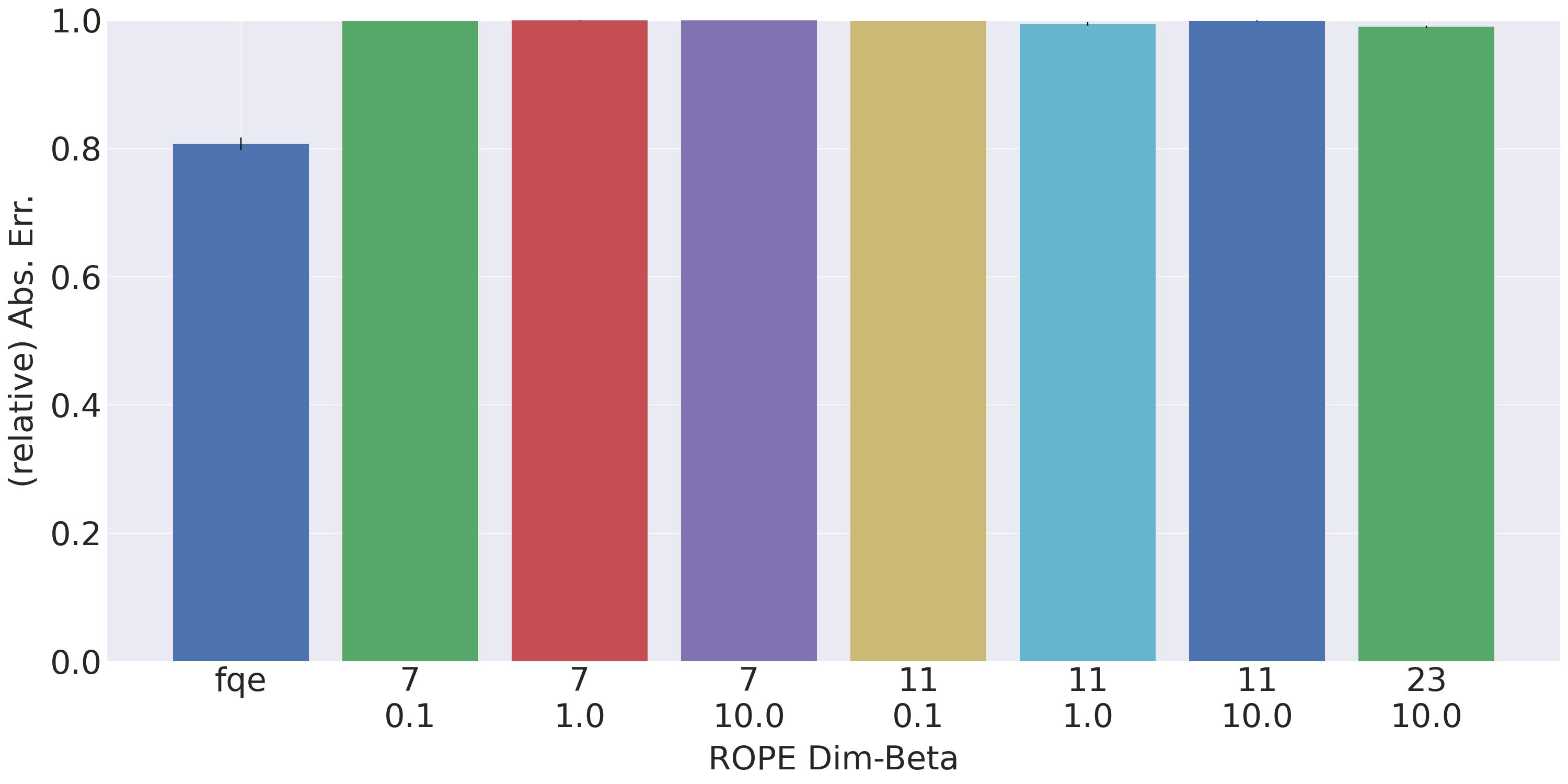}}
        \subfigure[HalfCheetah-medium]{\includegraphics[scale=0.125]{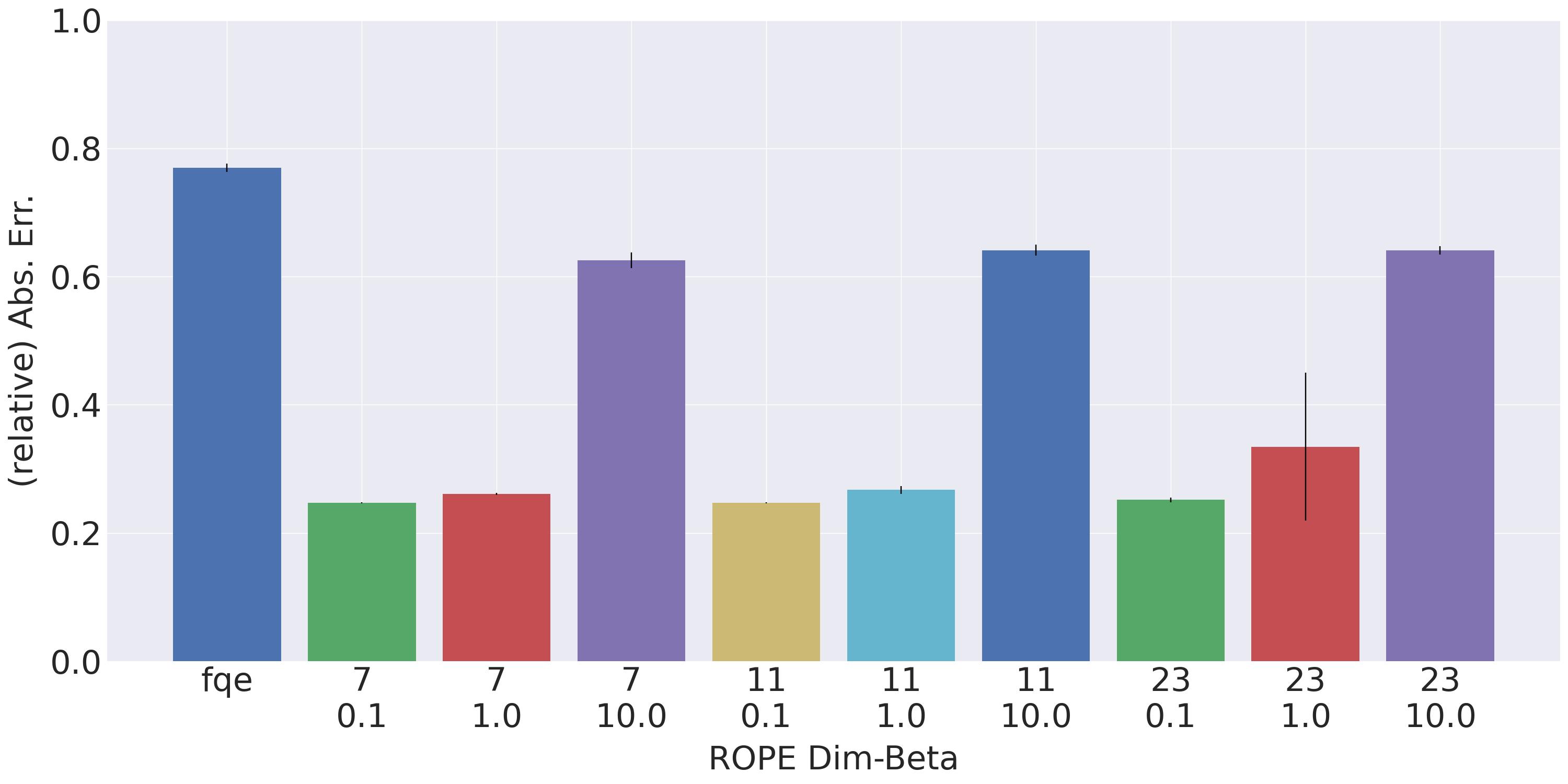}}
        \subfigure[HalfCheetah-medium-expert]{\includegraphics[scale=0.125]{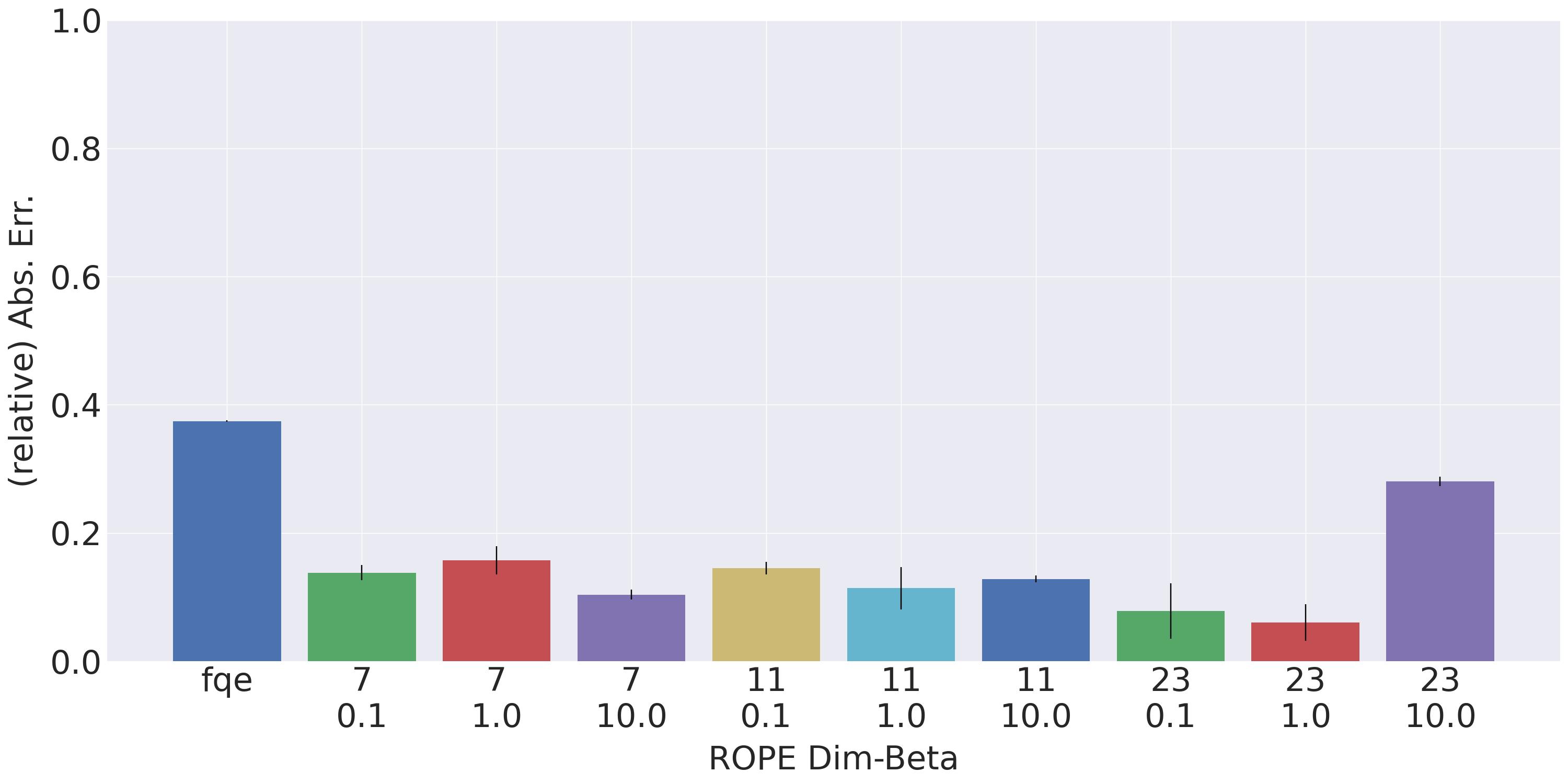}}
    \caption{\footnotesize \textsc{fqe} vs. \textsc{rope} when varying \textsc{rope}'s encoder output dimension (top) and $\beta$ (bottom) on the \textsc{d4rl} datasets.  \textsc{iqm} of errors are computed over $20$ trials with $95\%$ confidence intervals. Lower is better.}
\end{figure*}

\begin{figure*}[hbtp]
    \centering
        \subfigure[Walker2D-random]{\includegraphics[scale=0.125]{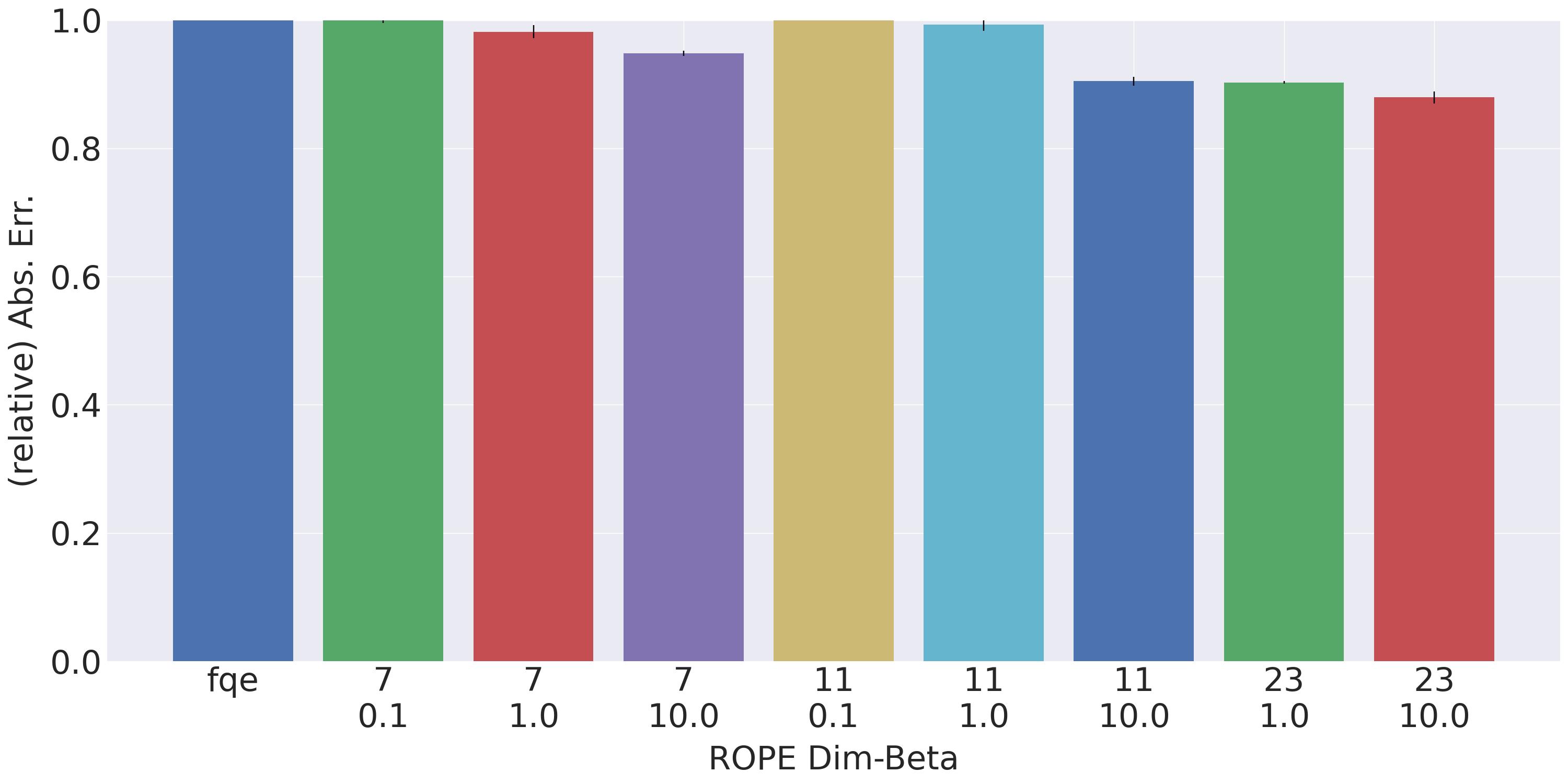}}
        \subfigure[Walker2D-medium]{\includegraphics[scale=0.125]{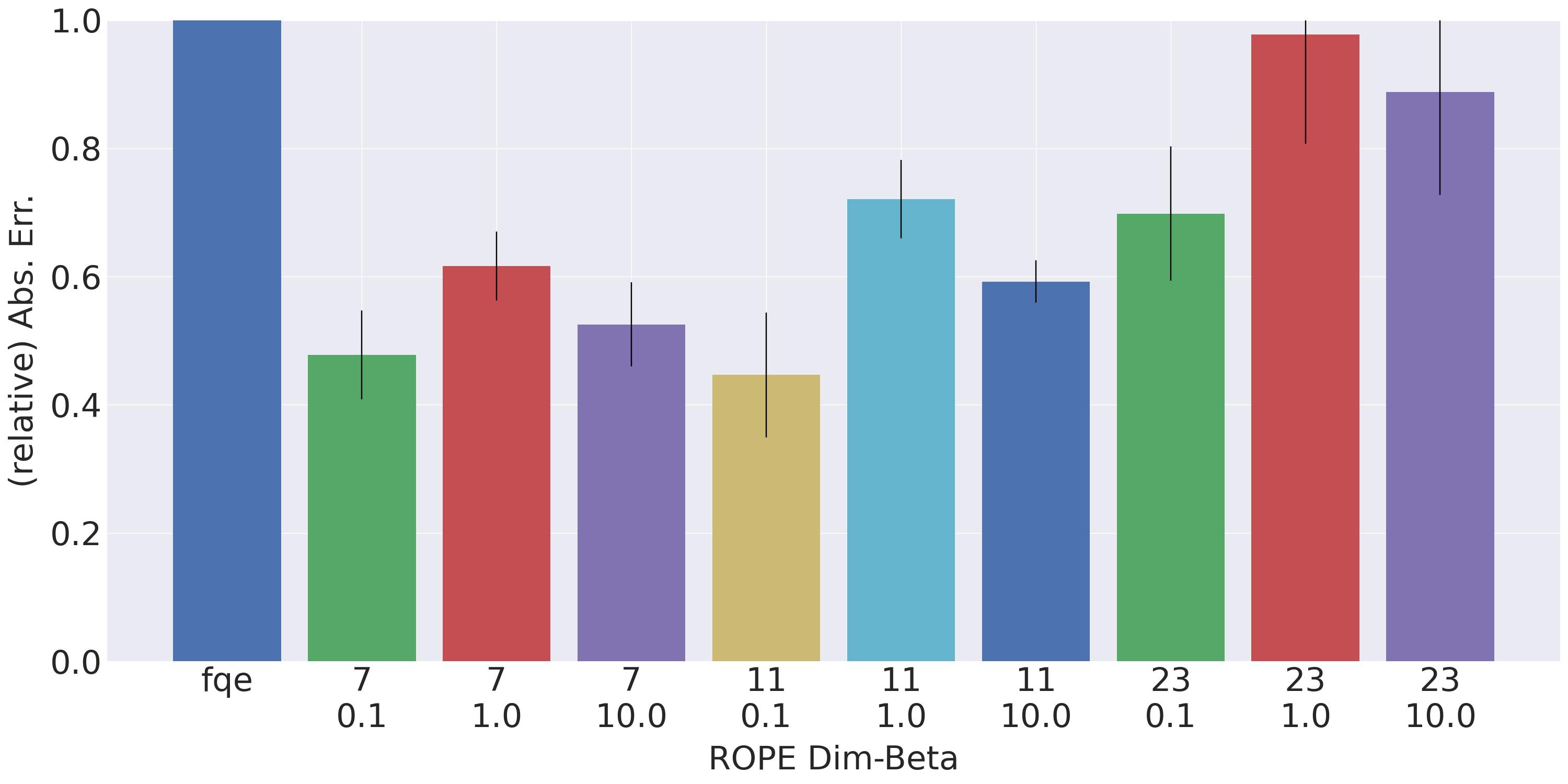}}
        \subfigure[Walker2D-medium-expert]{\includegraphics[scale=0.125]{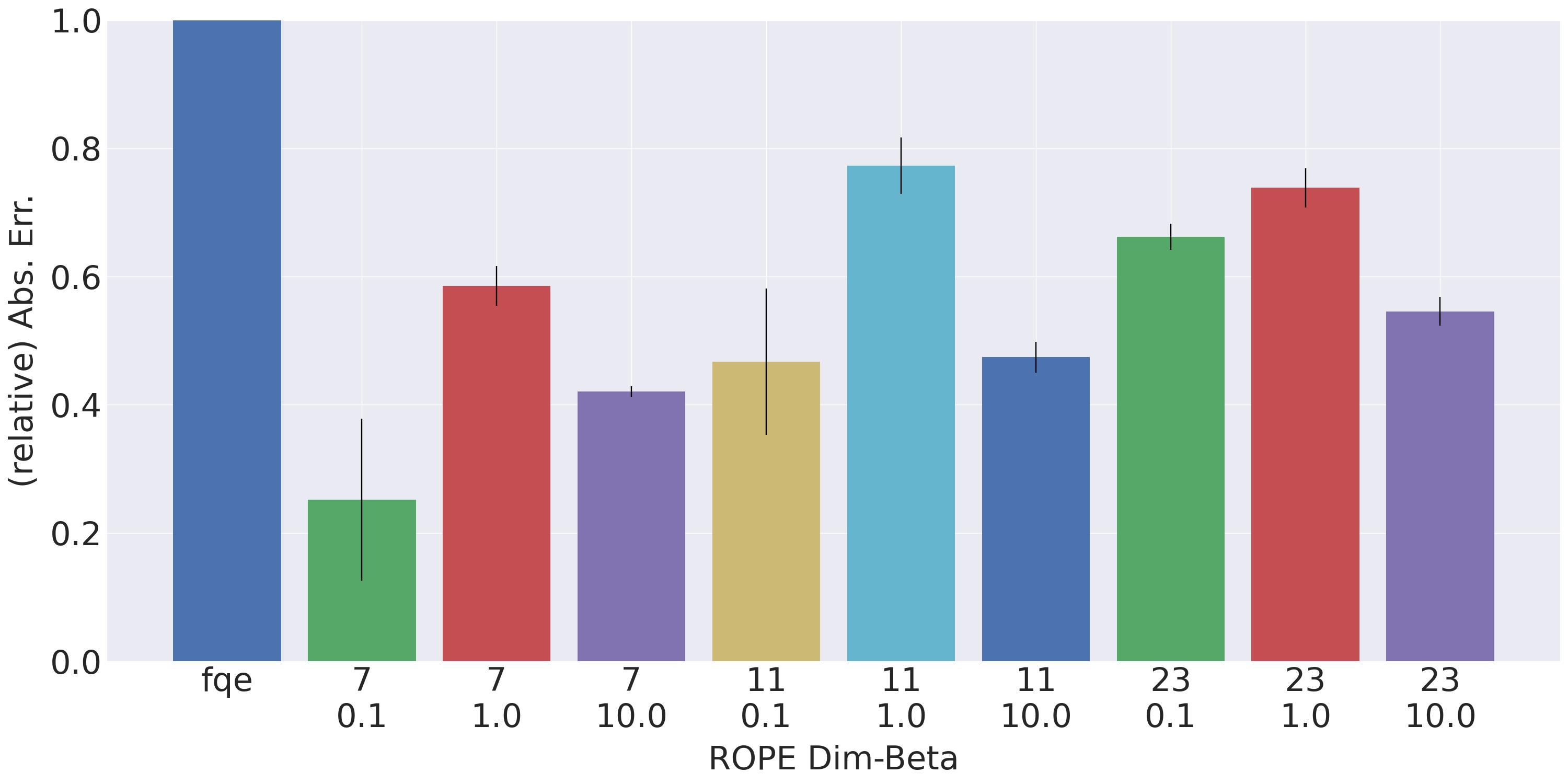}}
    \caption{\footnotesize \textsc{fqe} vs. \textsc{rope} when varying \textsc{rope}'s encoder output dimension (top) and $\beta$ (bottom) on the \textsc{d4rl} datasets.  \textsc{iqm} of errors are computed over $20$ trials with $95\%$ confidence intervals. Lower is better.}
\end{figure*}

\begin{figure*}[hbtp]
    \centering
        \subfigure[Hopper-random]{\includegraphics[scale=0.125]{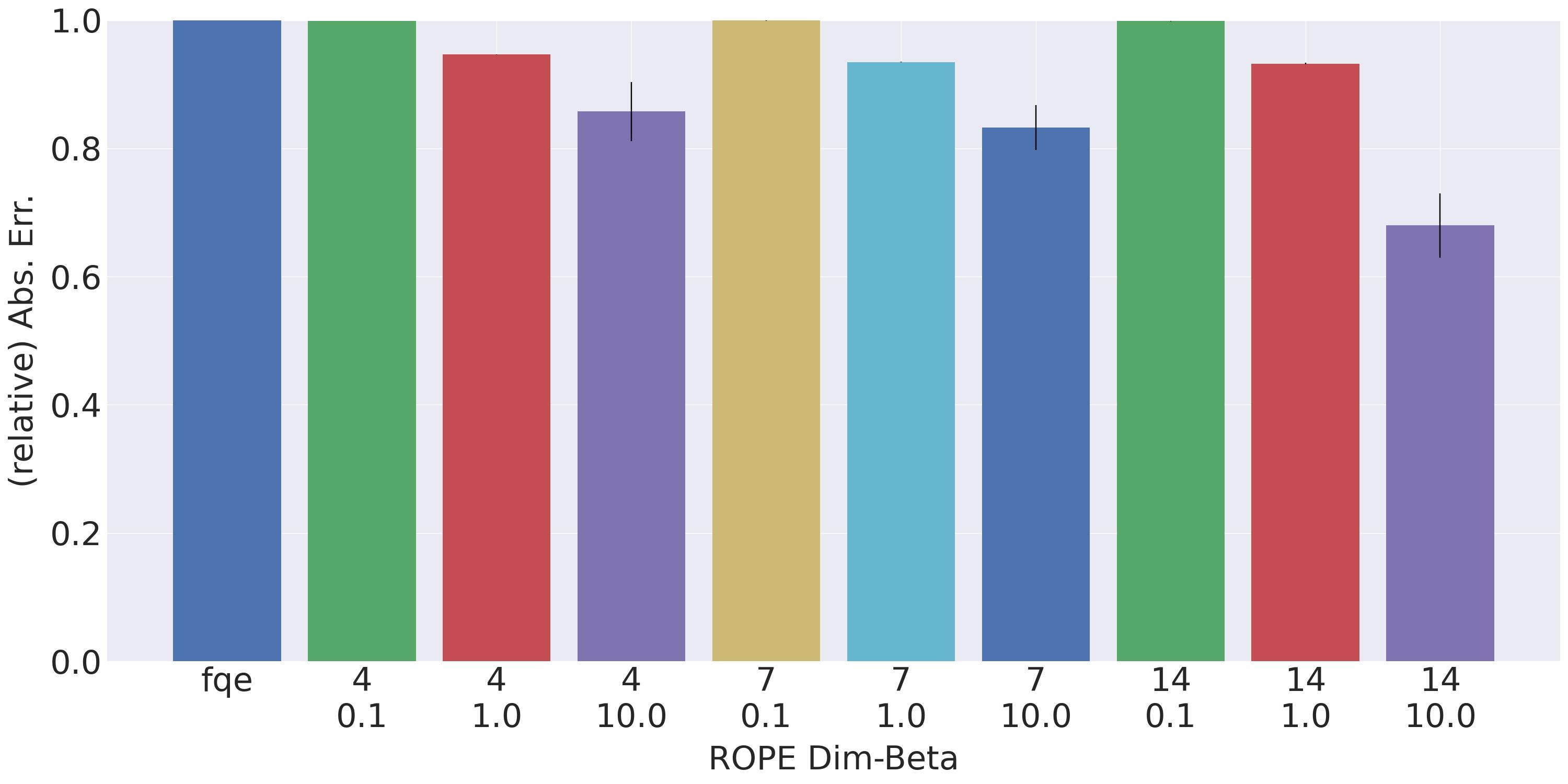}}
        \subfigure[Hopper-medium]{\includegraphics[scale=0.125]{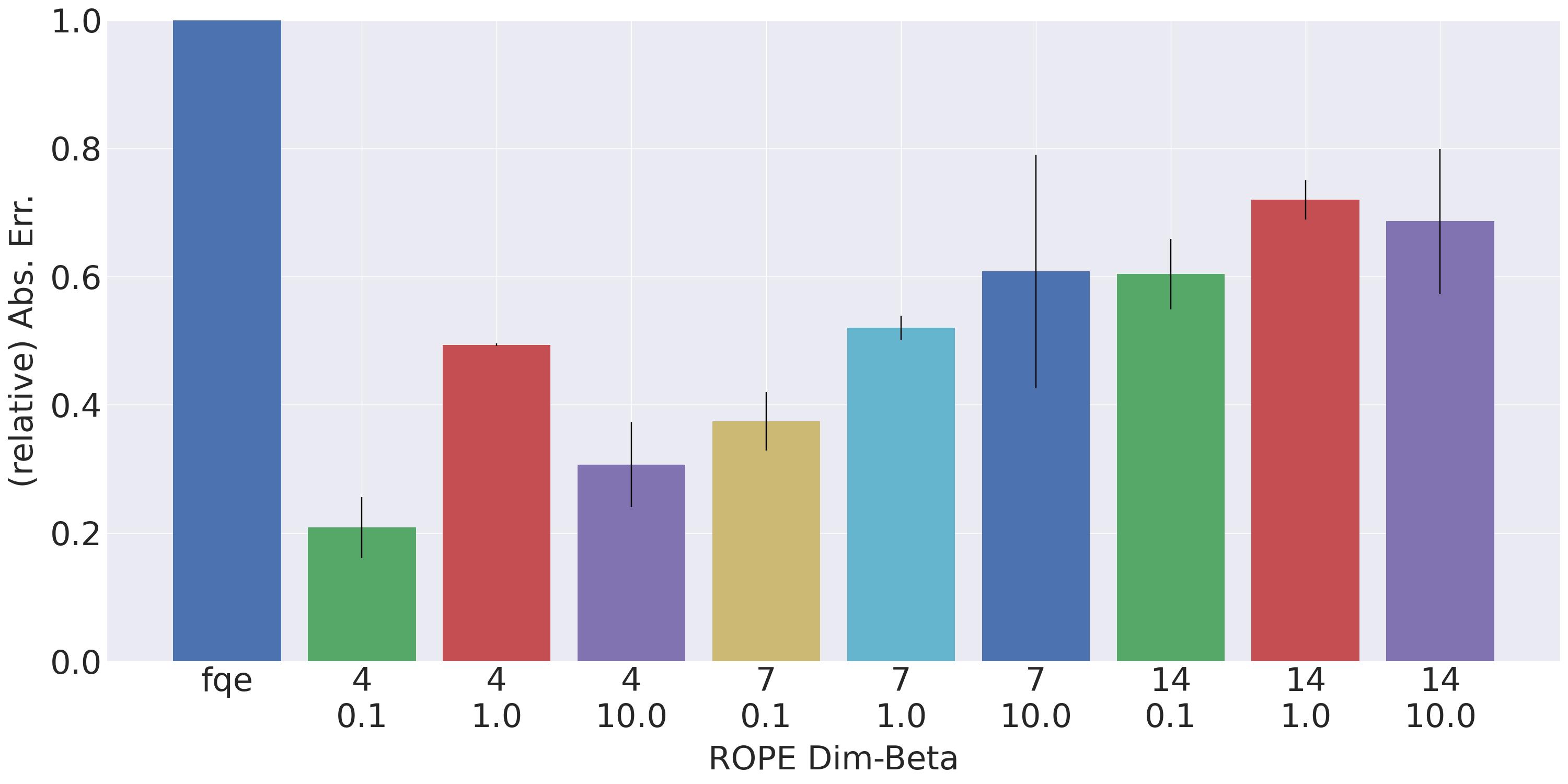}}
        \subfigure[Hopper-medium-expert]{\includegraphics[scale=0.125]{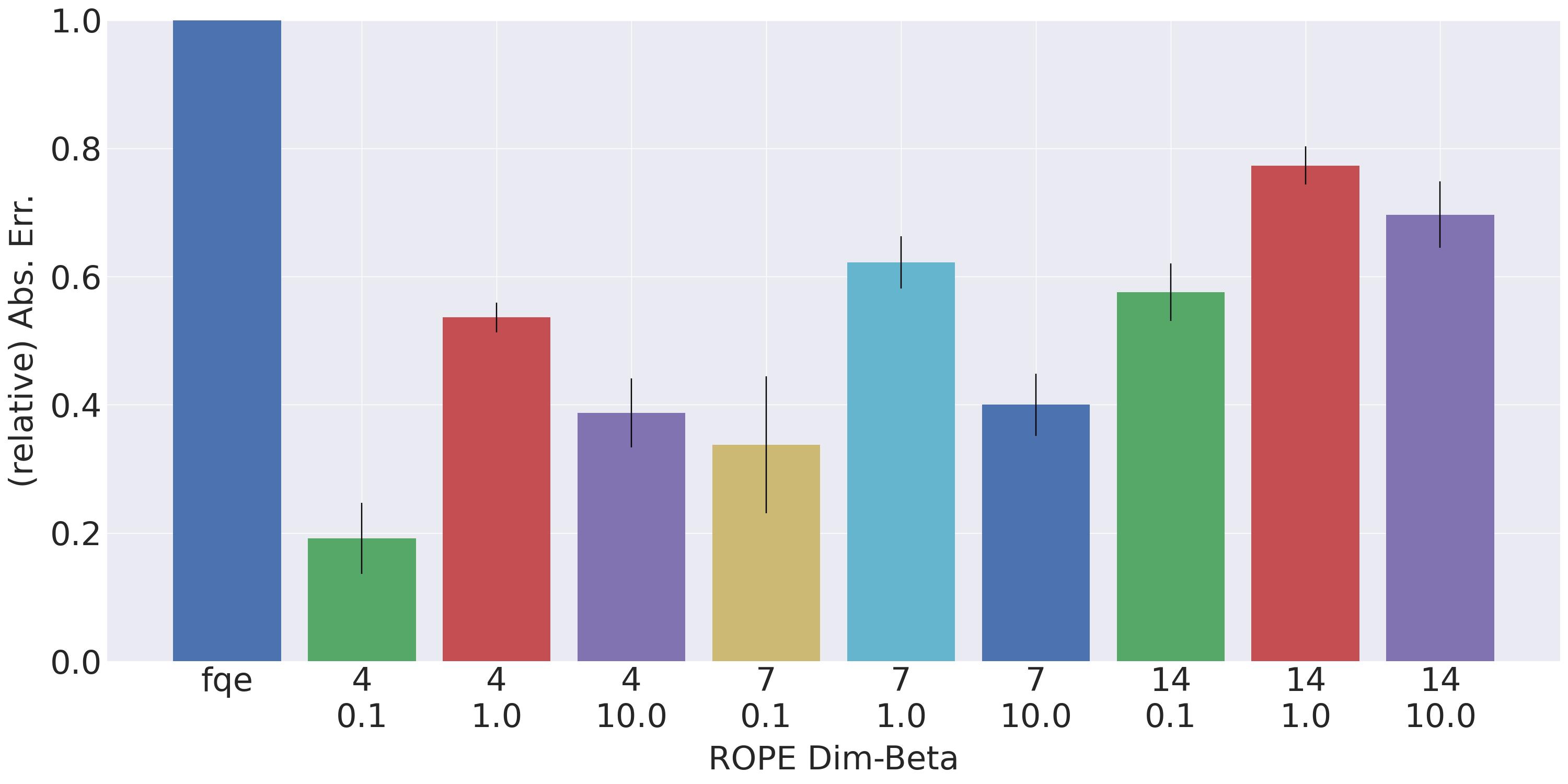}}
    \caption{\footnotesize \textsc{fqe} vs. \textsc{rope} when varying \textsc{rope}'s encoder output dimension (top) and $\beta$ (bottom) on the \textsc{d4rl} datasets.  \textsc{iqm} of errors are computed over $20$ trials with $95\%$ confidence intervals. Lower is better.}
    \label{fig:rope_abl_hp_end}
\end{figure*}

\subsubsection{Ablation: RMAE Distributions}

In this section, show the remaining \textsc{rmae} distribution curves \citep{agarwal_precipice_2021} of each algorithm on all datasets. We reach the similar conclusion that on very difficult datasets, \textsc{rope} significantly mitigates the divergence of \textsc{fqe} and that to avoid \textsc{fqe} divergence it is necessary to clip the bootstrapping target. See Figures~\ref{fig:rope_abl_divergence_start} to \ref{fig:rope_abl_divergence_end}.

\begin{figure*}[hbtp]
    \centering
        \subfigure[Swimmer]{\includegraphics[scale=0.15]{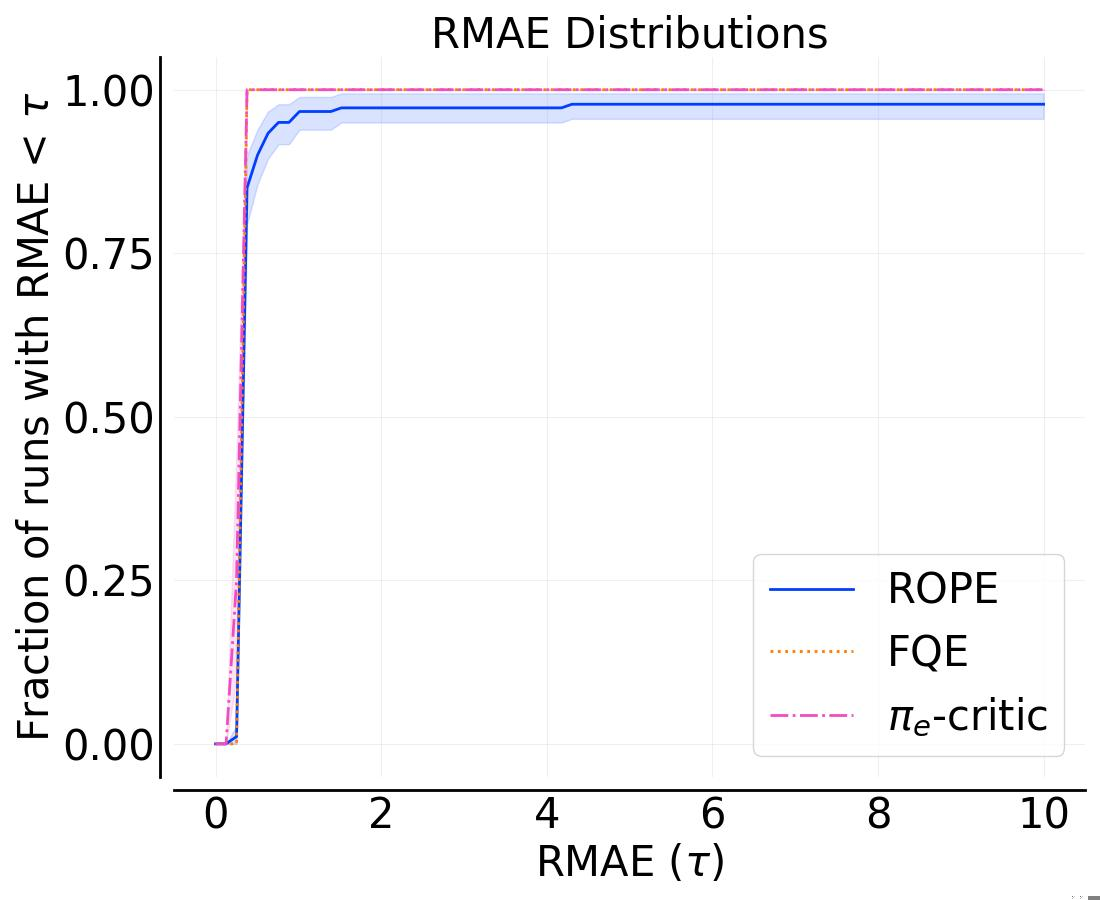}}
        \subfigure[HalfCheetah]{\includegraphics[scale=0.15]{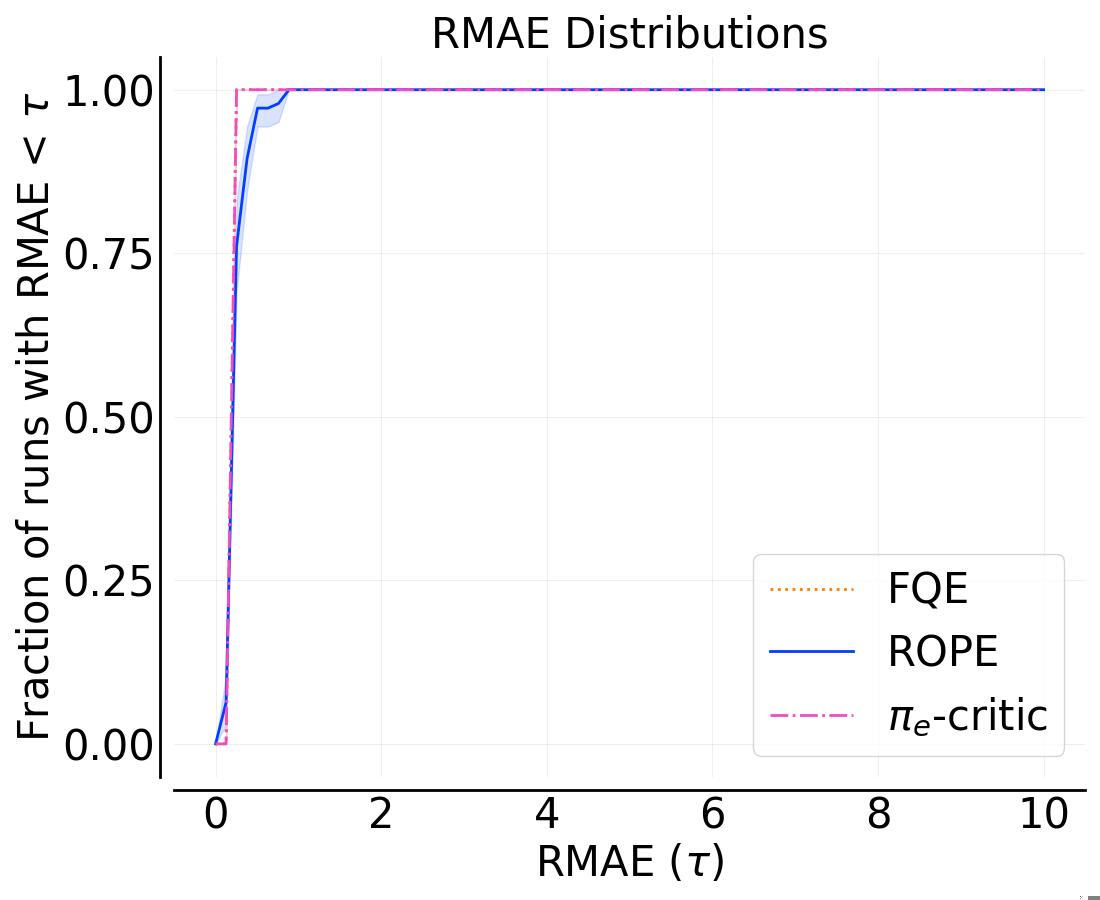}}
        \subfigure[HumanoidStandup]{\includegraphics[scale=0.15]{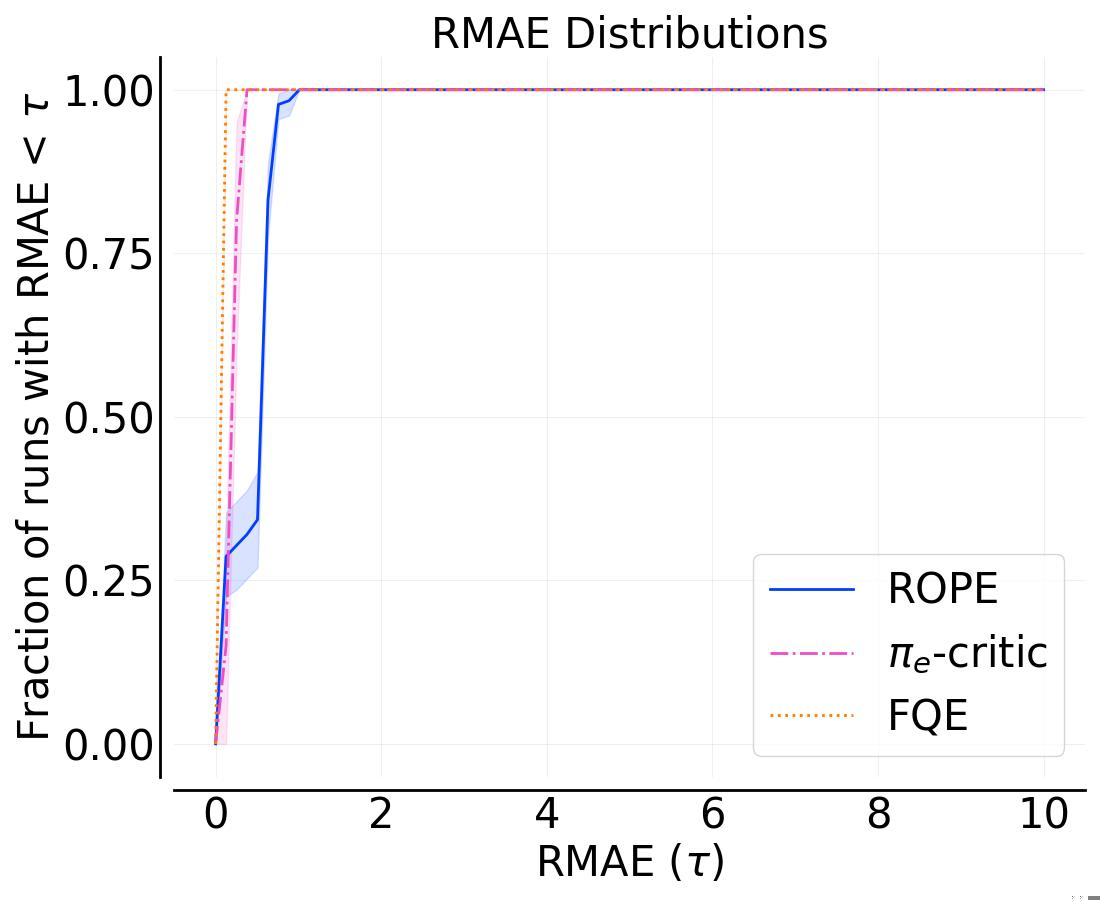}}
    \caption{\footnotesize \textsc{rmae} distributions across all runs and hyperparameters for each algorithm, resulting in $\geq 20$ runs for each algorithm. Shaded region is $95\%$ confidence interval. Larger area under the curve is better.}
    \label{fig:rope_abl_divergence_start}
\end{figure*}

\begin{figure*}[hbtp]
    \centering
        \subfigure[HalfCheetah-random]{\includegraphics[scale=0.15]{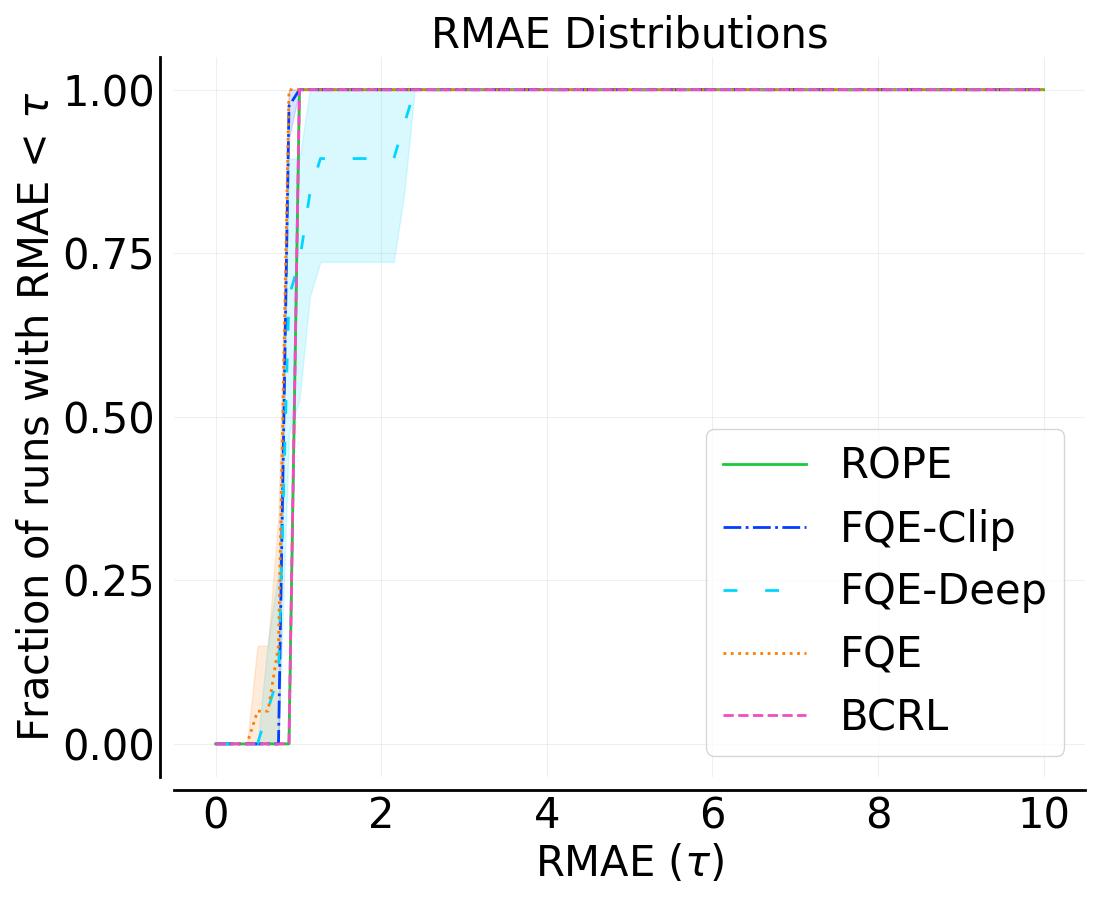}}
        \subfigure[HalfCheetah-medium]{\includegraphics[scale=0.15]{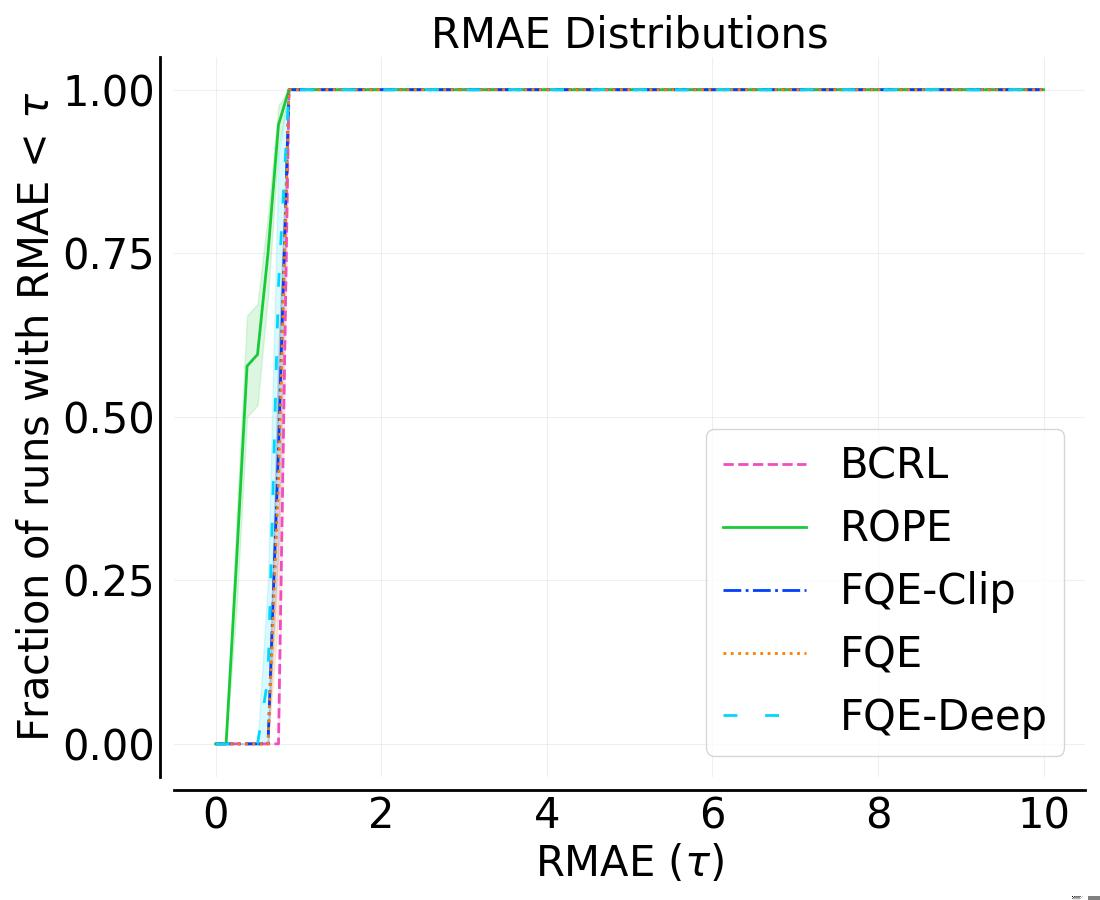}}
        \subfigure[HalfCheetah-medium-expert]{\includegraphics[scale=0.15]{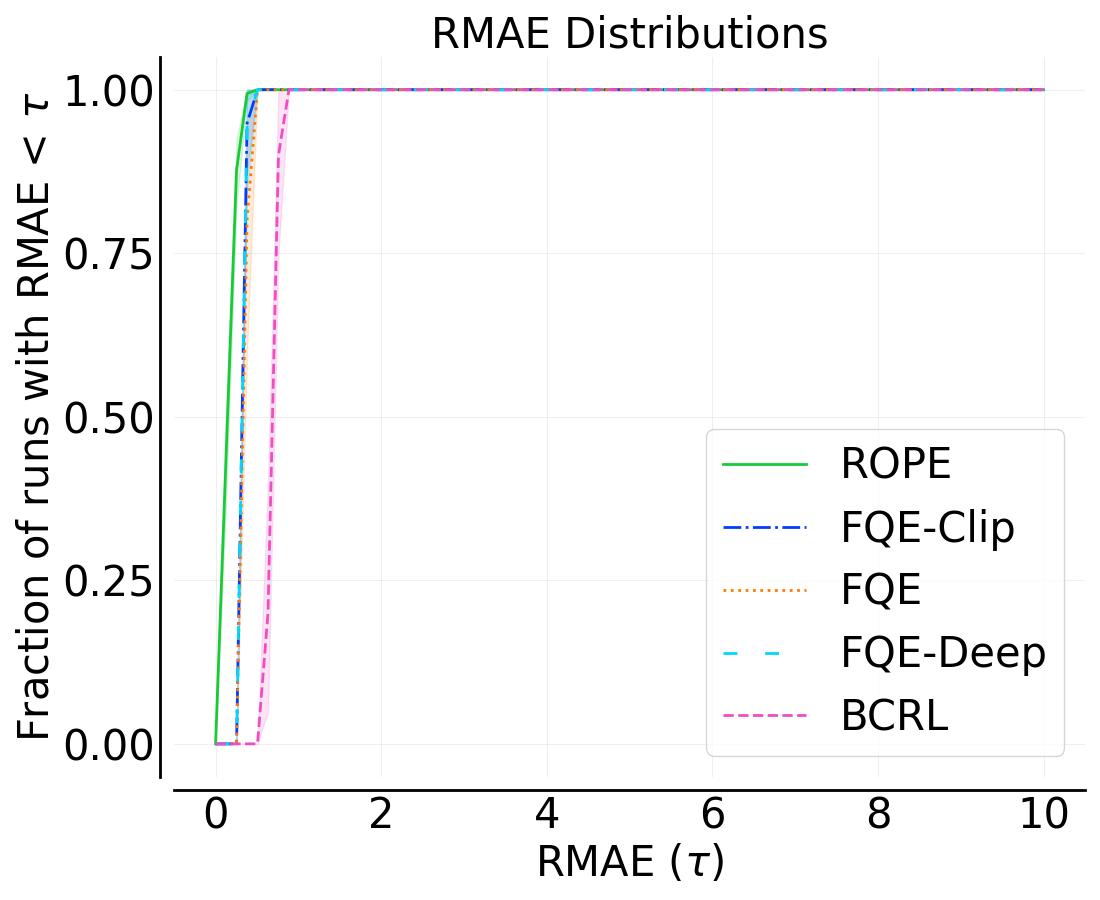}}
    \caption{\footnotesize \textsc{rmae} distributions across all runs and hyperparameters for each algorithm, resulting in $\geq 20$ runs for each algorithm. Shaded region is $95\%$ confidence interval. Larger area under the curve is better.}
\end{figure*}

\begin{figure*}[hbtp]
    \centering
        \subfigure[Walker2D-random]{\includegraphics[scale=0.15]{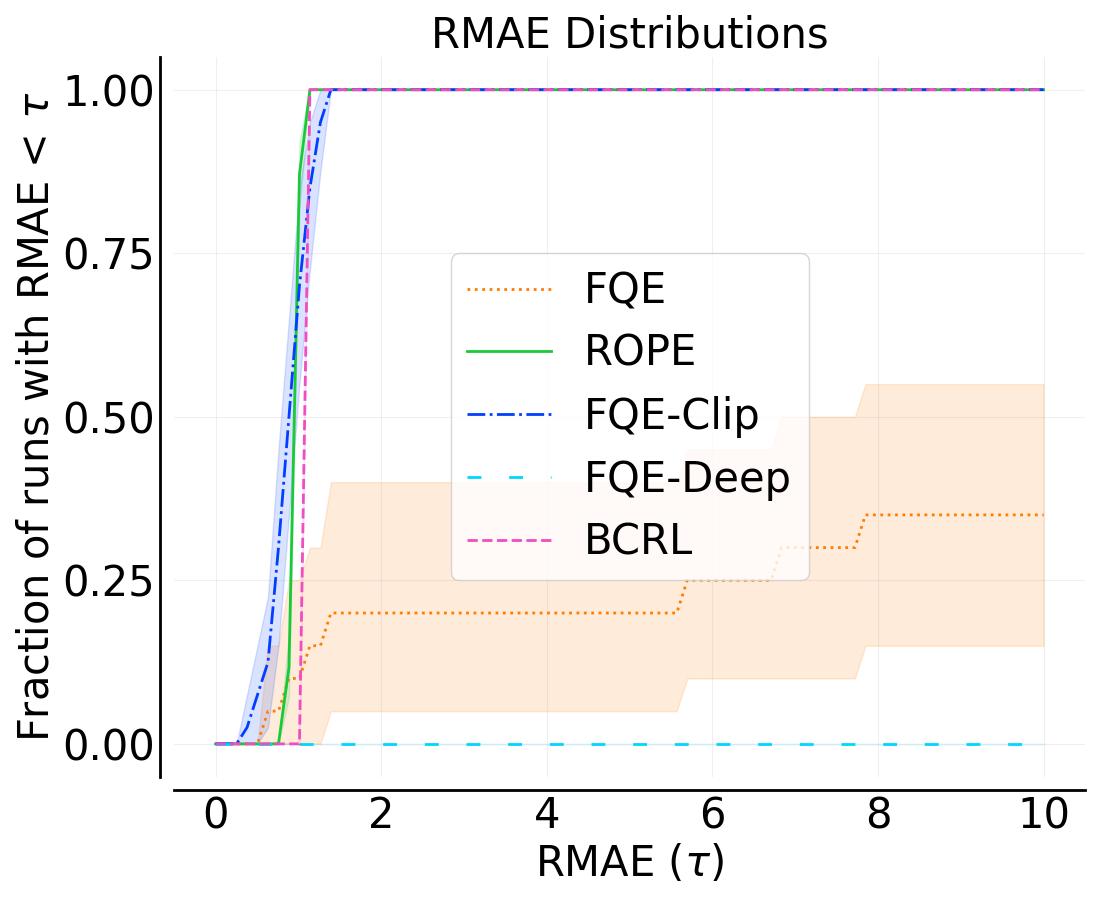}}
        \subfigure[Walker2D-medium]{\includegraphics[scale=0.15]{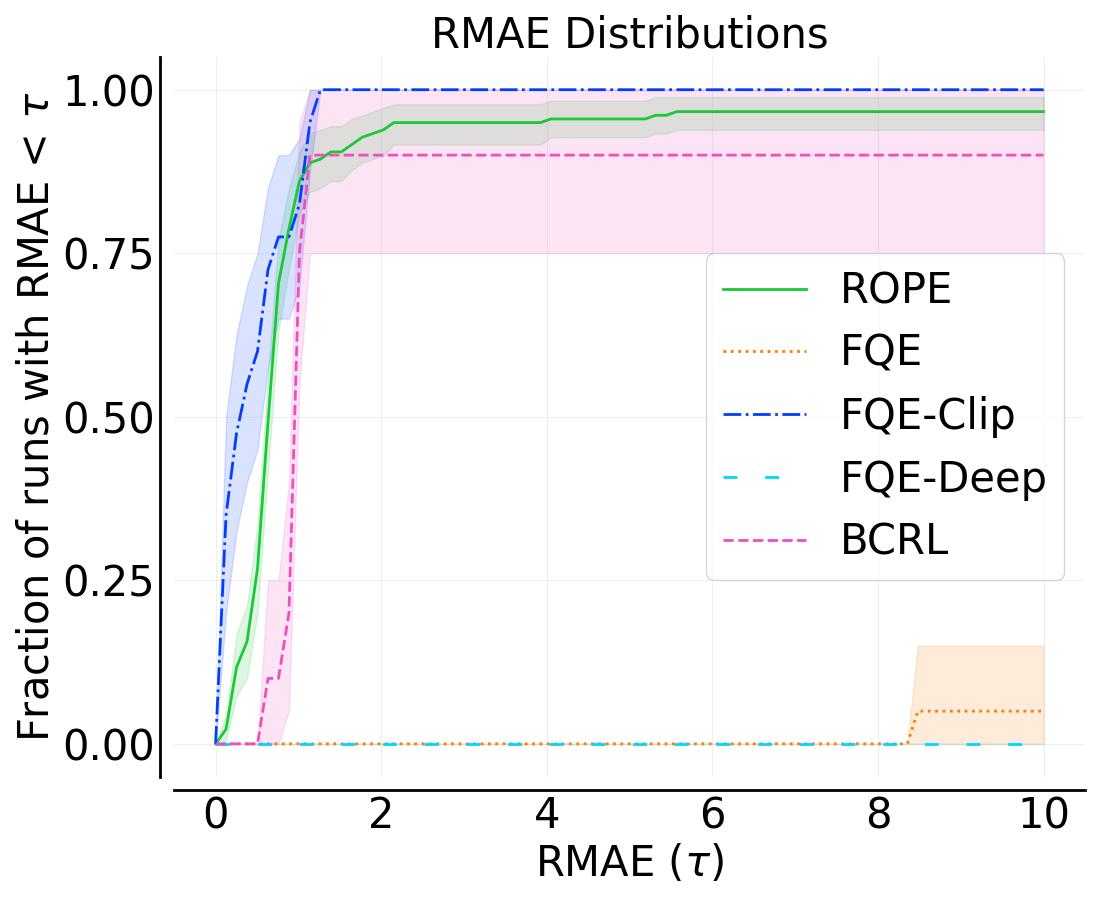}}
        \subfigure[Walker2D-medium-expert]{\includegraphics[scale=0.15]{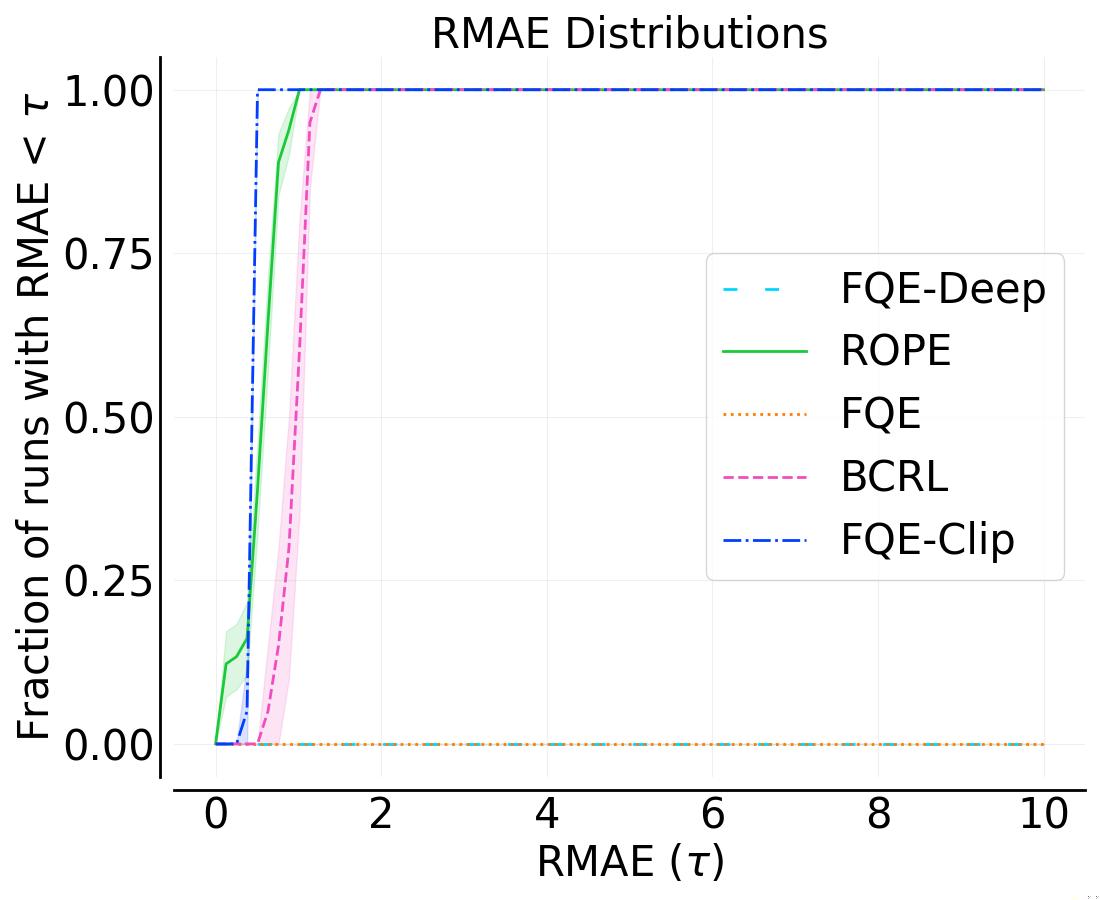}}
    \caption{\footnotesize \textsc{rmae} distributions across all runs and hyperparameters for each algorithm, resulting in $\geq 20$ runs for each algorithm. Shaded region is $95\%$ confidence interval. Larger area under the curve is better.}
\end{figure*}

\begin{figure*}[hbtp]
    \centering
        \subfigure[Hopper-random]{\includegraphics[scale=0.15]{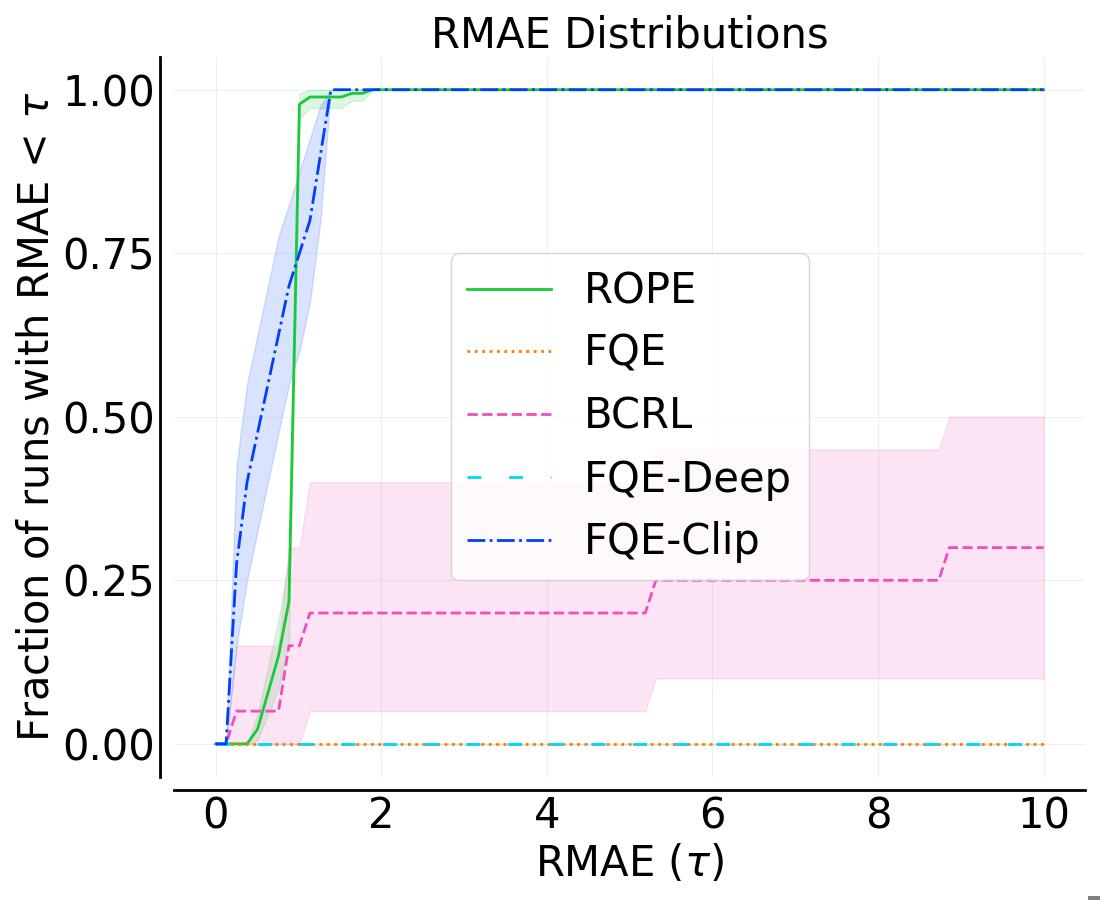}}
        \subfigure[Hopper-medium]{\includegraphics[scale=0.15]{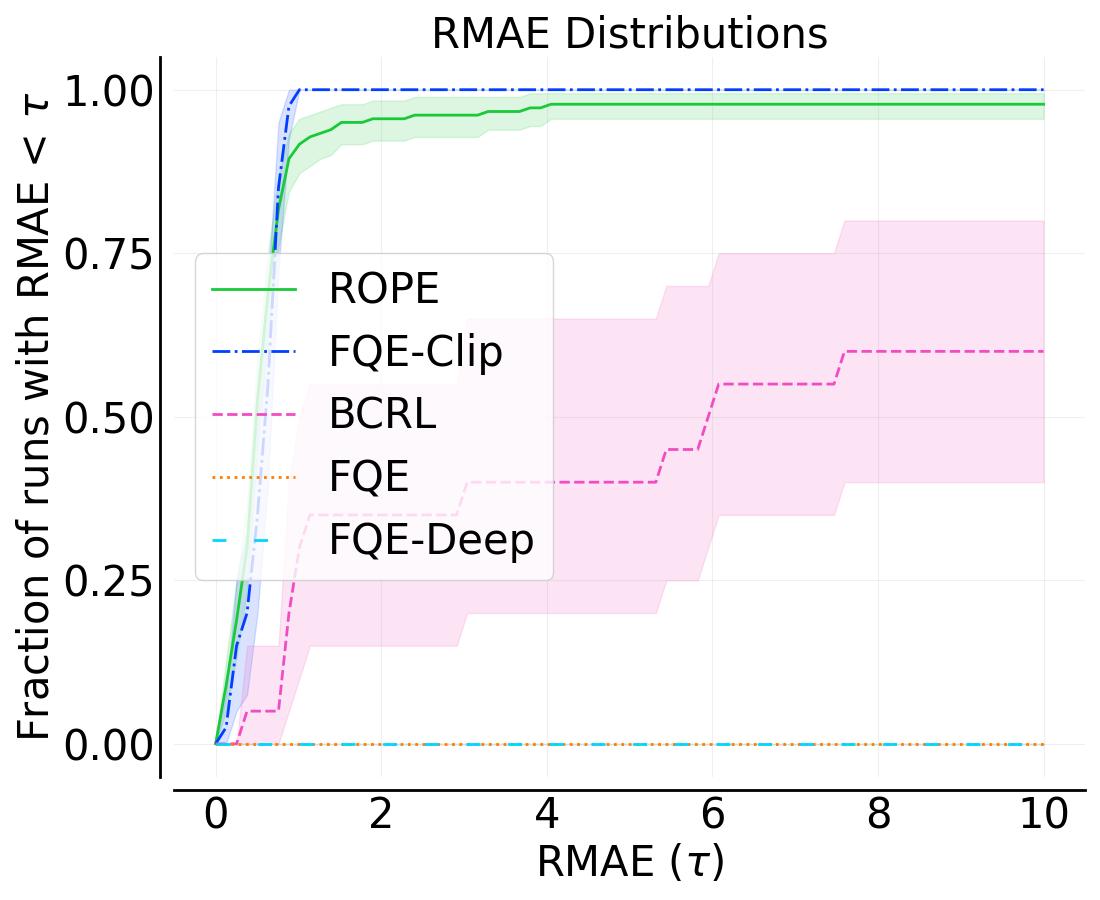}}
        \subfigure[Hopper-medium-expert]{\includegraphics[scale=0.15]{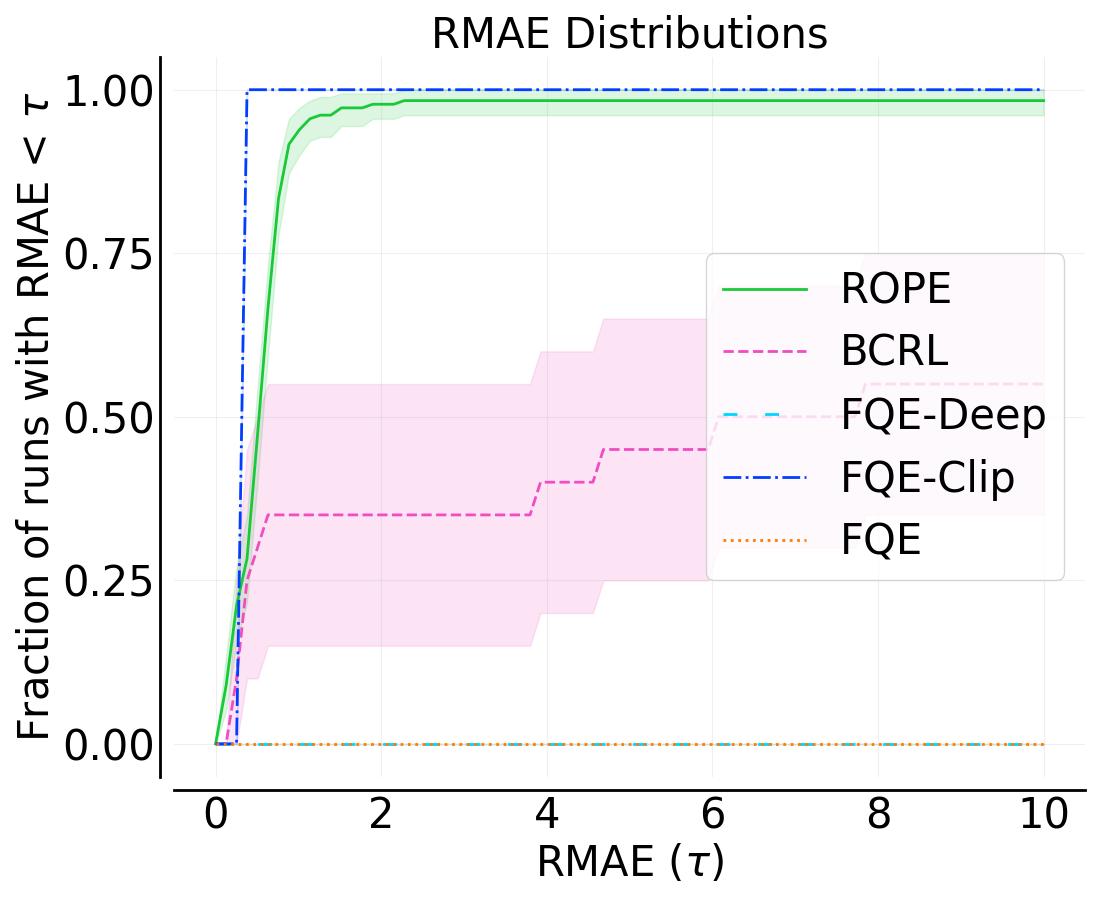}}
    \caption{\footnotesize \textsc{rmae} distributions across all runs and hyperparameters for each algorithm, resulting in $\geq 20$ runs for each algorithm. Shaded region is $95\%$ confidence interval. Larger area under the curve is better.}
    \label{fig:rope_abl_divergence_end}
\end{figure*}

\subsubsection{Training Loss Curves for ROPE and FQE}

In this section, we include the training loss curves for \textsc{rope}'s training, \textsc{fqe}'s training using \textsc{rope} representations as input, and normal \textsc{fqe} and \textsc{fqe-clip}. The training curves are a function of the algorithms hyperparameters (learning rate for \textsc{fqe}, $\beta$ and representation output dimension for \textsc{rope}). We can see that on difficult datasets, the loss of \textsc{fqe} diverges. On the other hand, with \textsc{rope}, \textsc{fqe}'s divergence is significantly mitigated. Note that \textsc{rope} does not eliminate the divergence. See Figures~\ref{fig:tr_losses_start} to \ref{fig:tr_losses_end}.

\begin{figure*}[hbtp]
    \centering
        \subfigure[\textsc{fqe} HalfCheetah-random]{\includegraphics[scale=0.125]{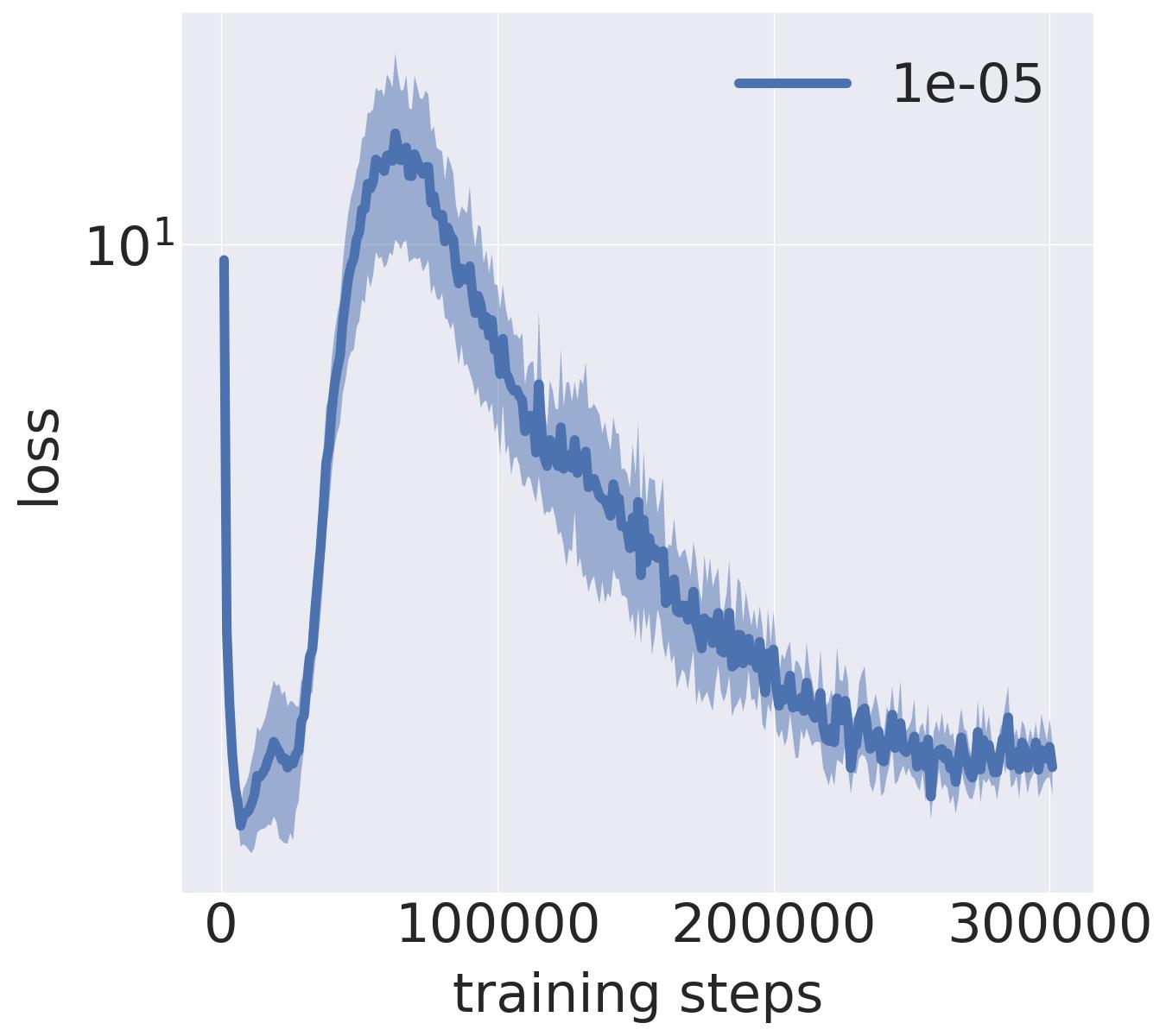}}
        \subfigure[\textsc{fqe-clip} HalfCheetah-random]{\includegraphics[scale=0.125]{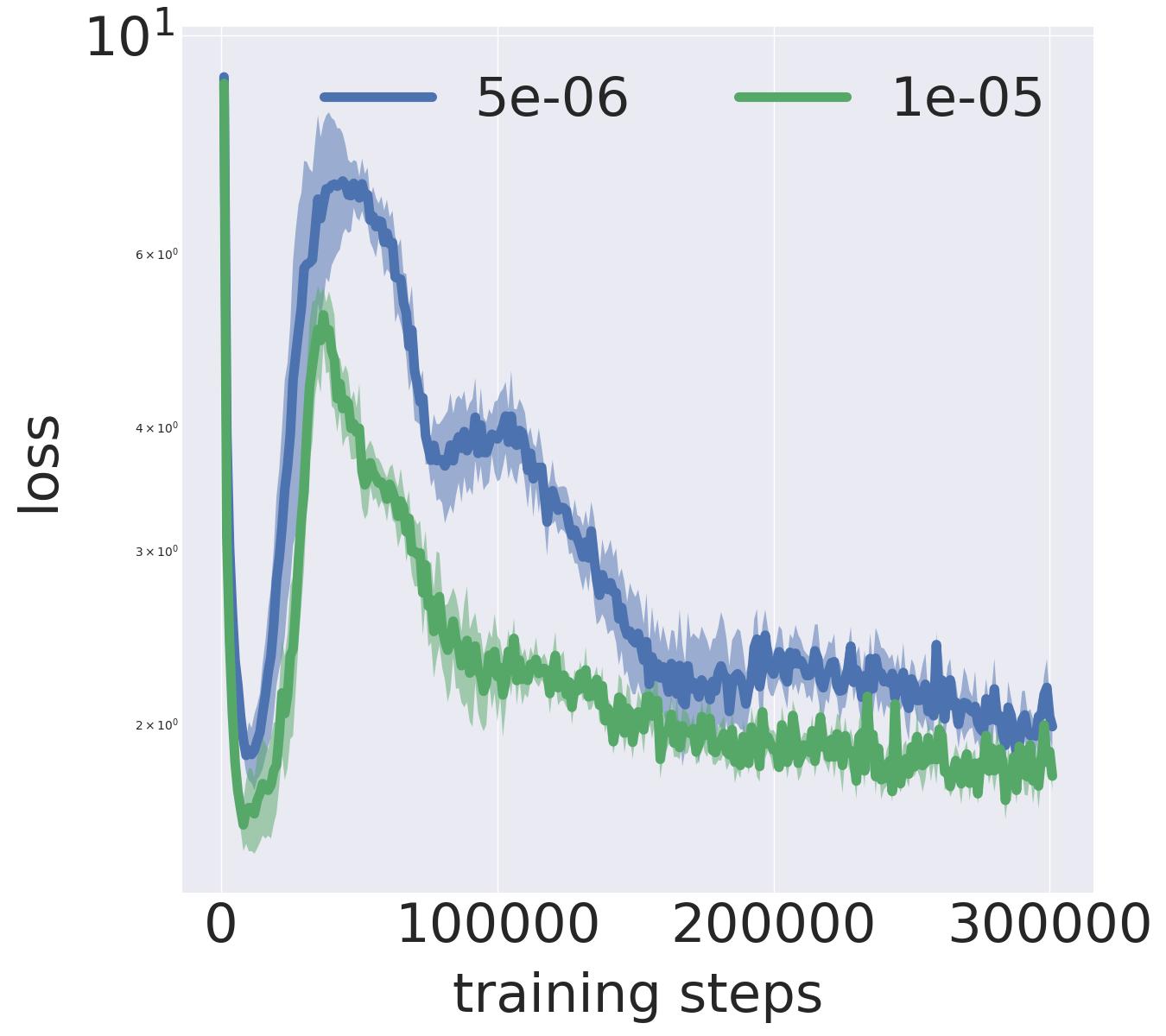}}
        \subfigure[\textsc{fqe} w/ \textsc{rope} HalfCheetah-random]{\includegraphics[scale=0.125]{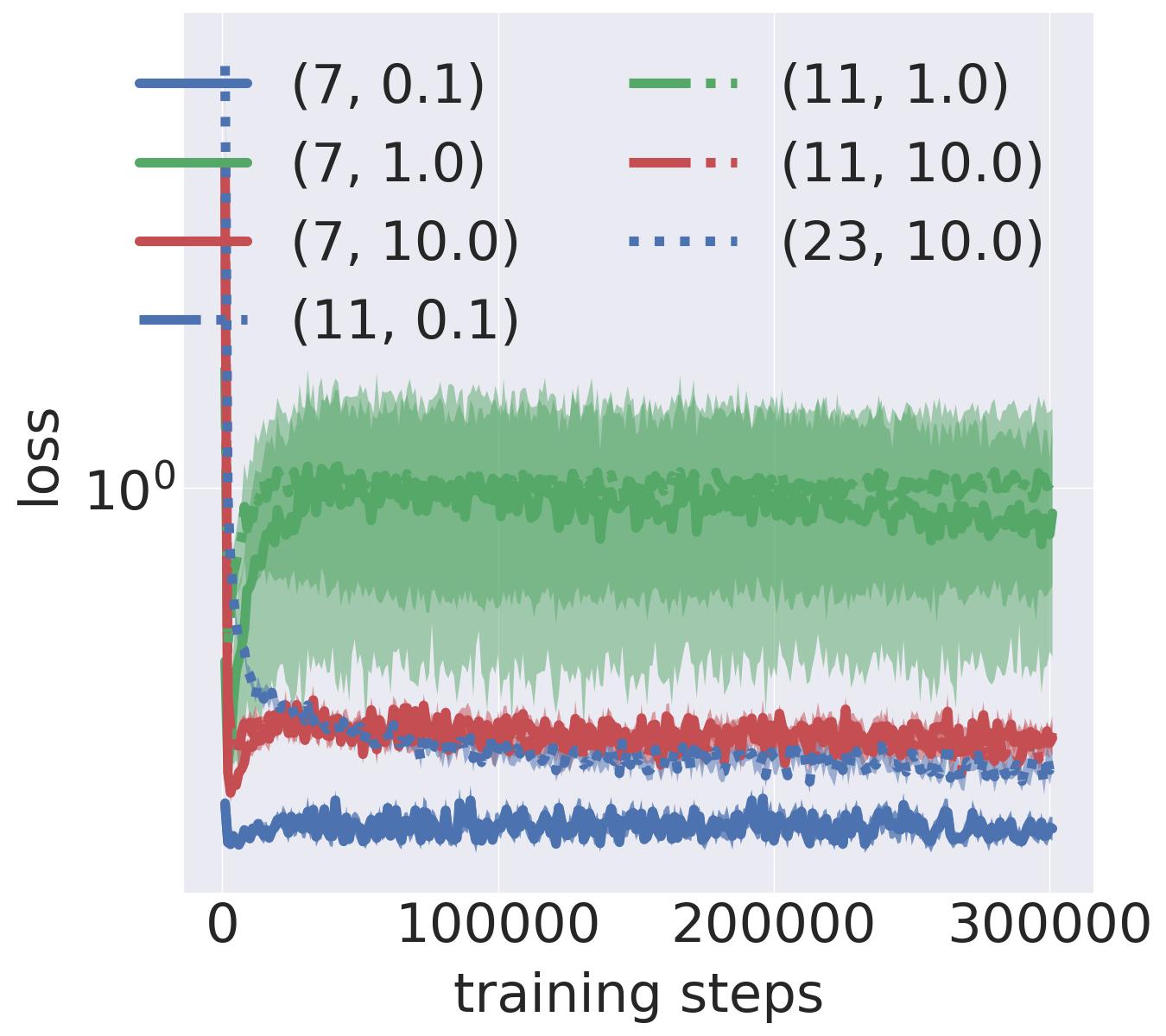}}\\
        \subfigure[\textsc{fqe} HalfCheetah-medium]{\includegraphics[scale=0.125]{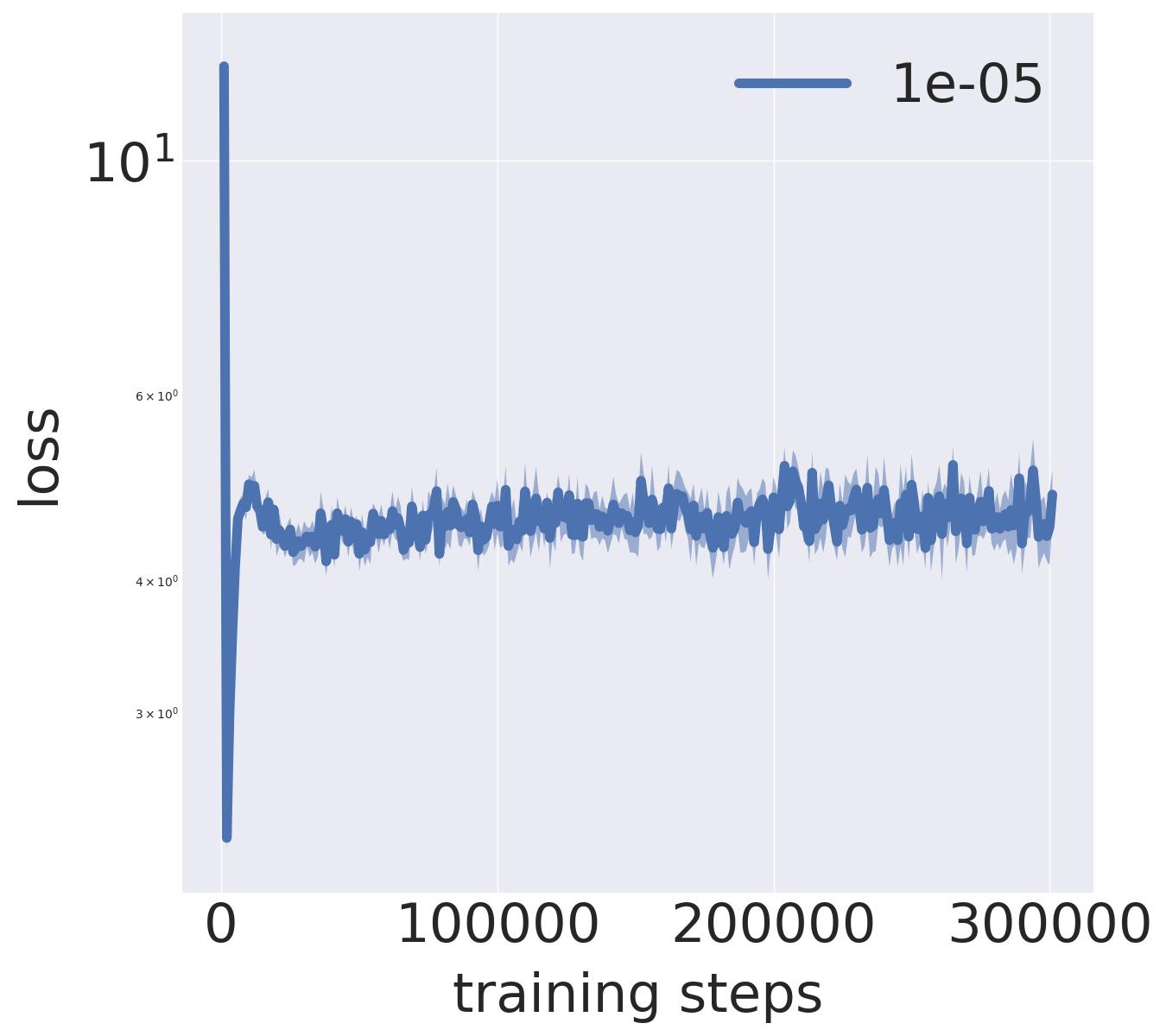}}
        \subfigure[\textsc{fqe-clip} HalfCheetah-medium]{\includegraphics[scale=0.125]{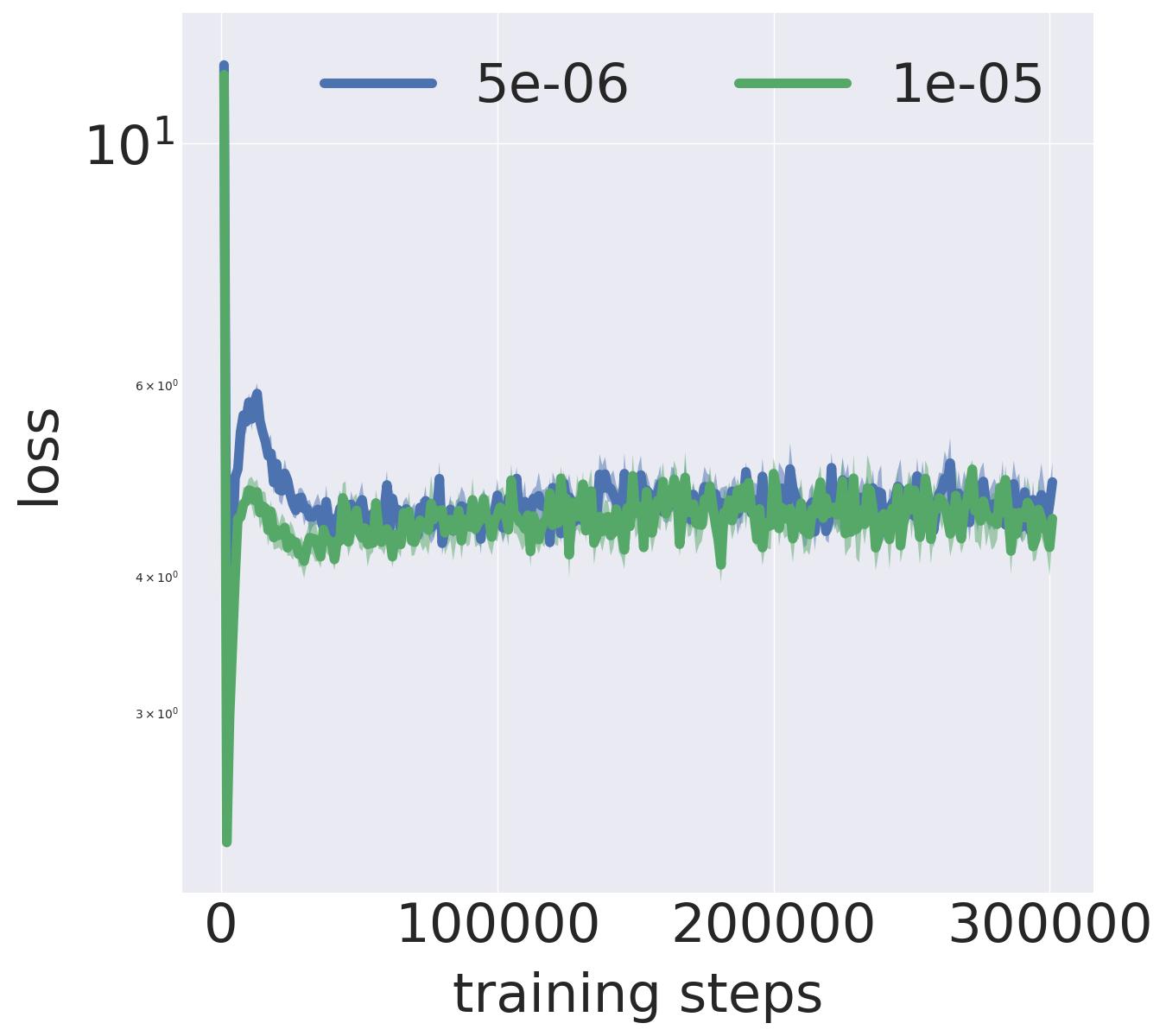}}
        \subfigure[\textsc{fqe} w/ \textsc{rope} HalfCheetah-medium]{\includegraphics[scale=0.125]{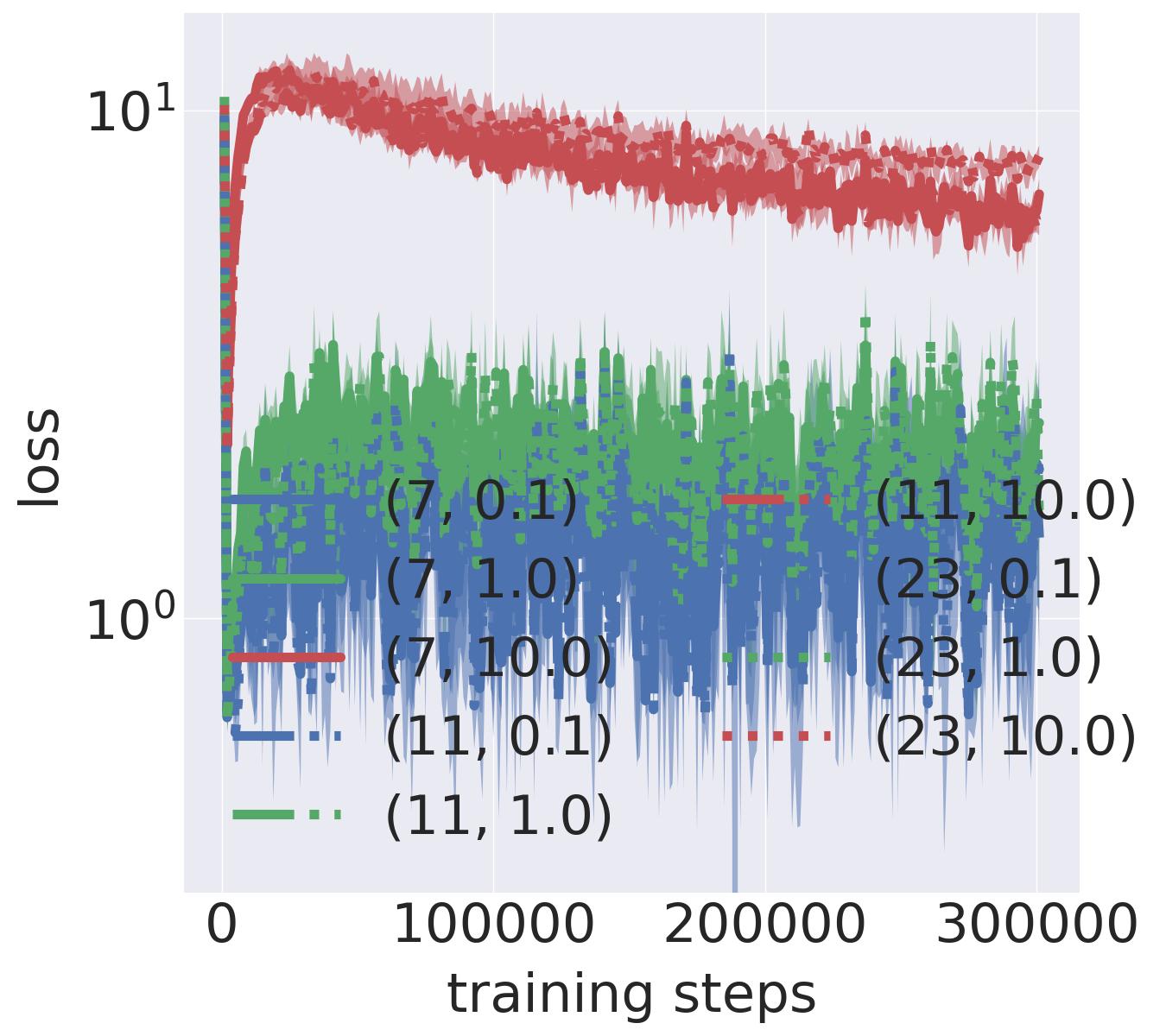}}\\
        \subfigure[\textsc{fqe} HalfCheetah-medium-expert]{\includegraphics[scale=0.125]{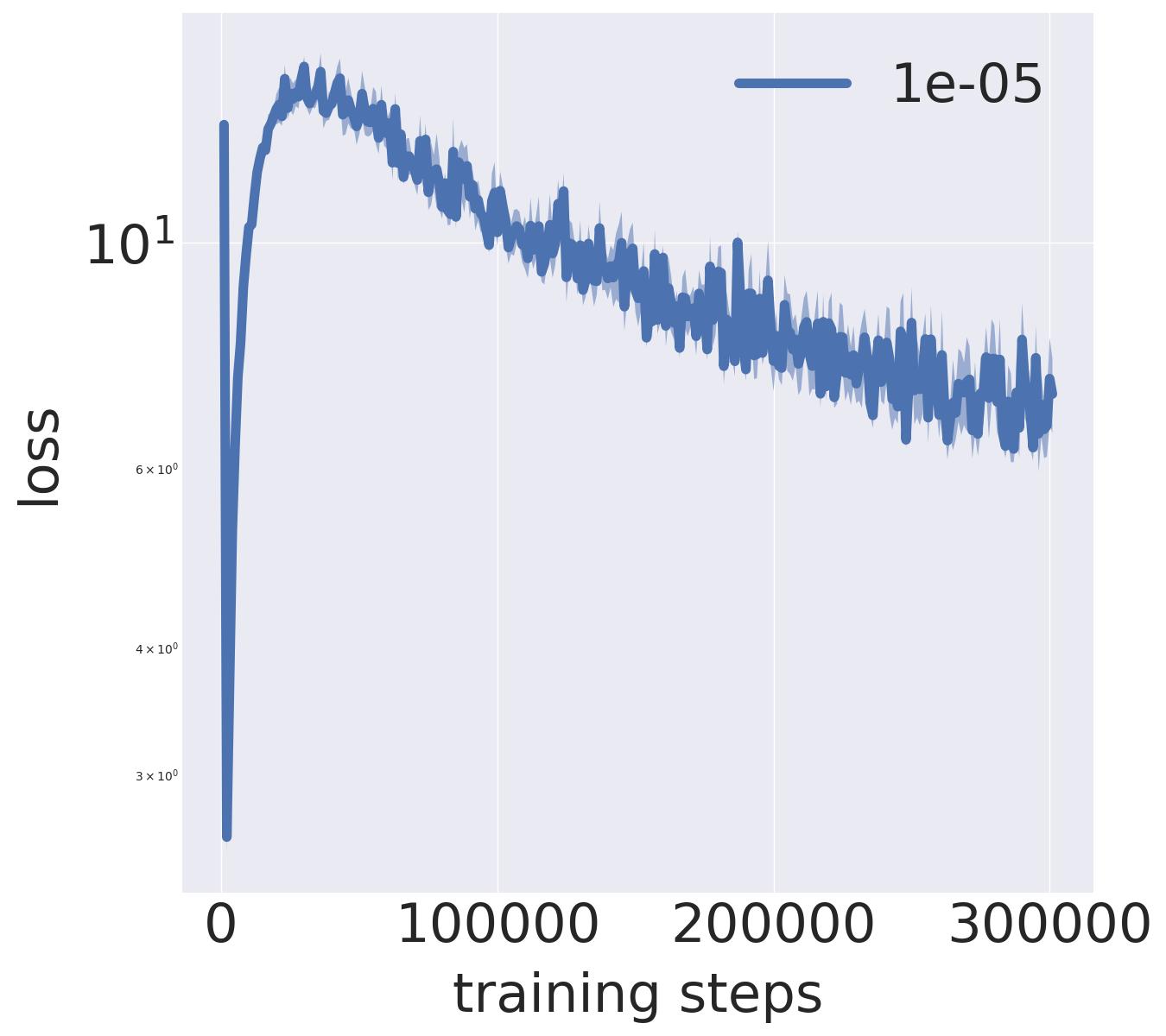}}
        \subfigure[\textsc{fqe-clip} HalfCheetah-medium-expert]{\includegraphics[scale=0.125]{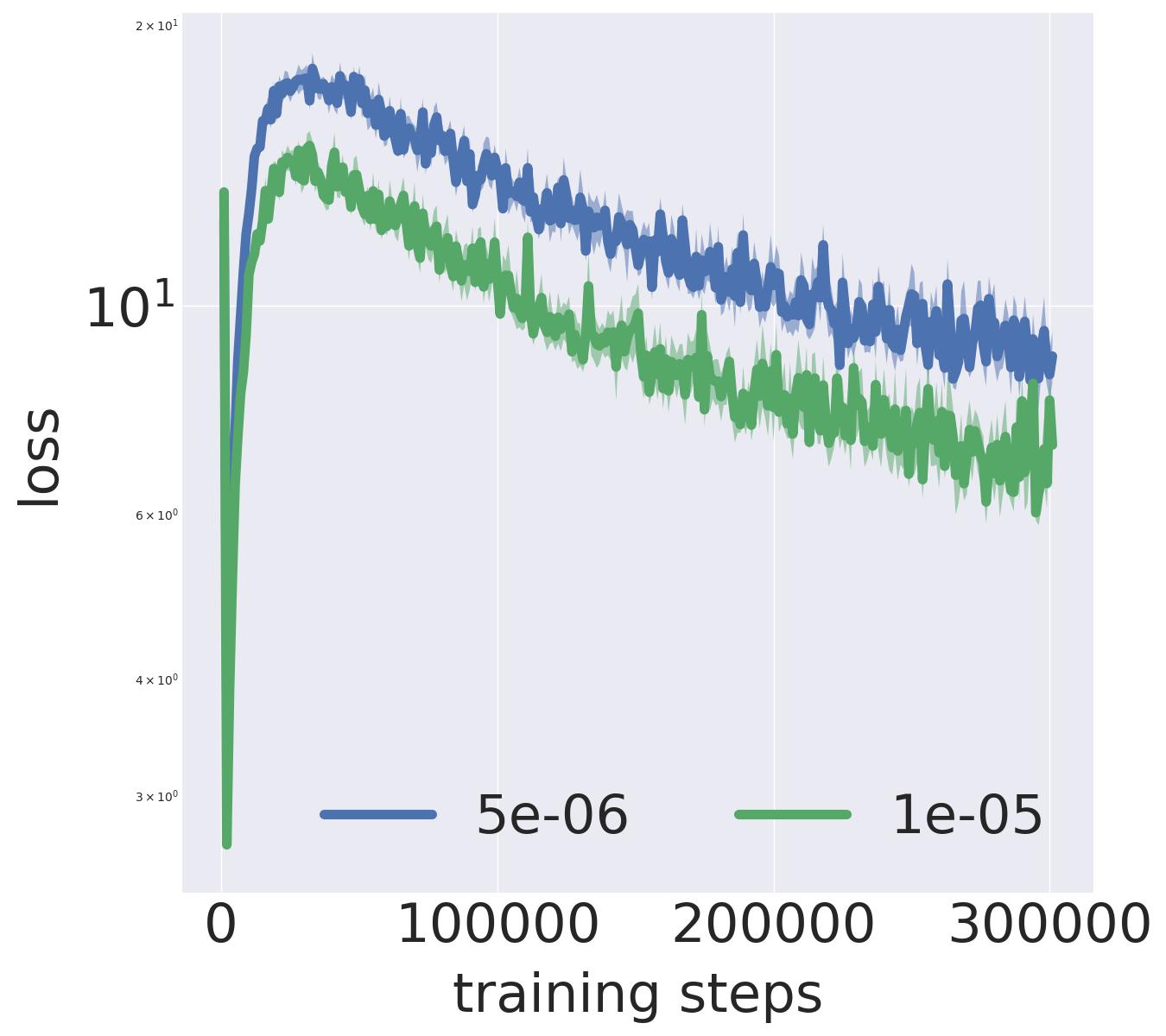}}
        \subfigure[\textsc{fqe} w/ \textsc{rope} HalfCheetah-medium-expert]{\includegraphics[scale=0.125]{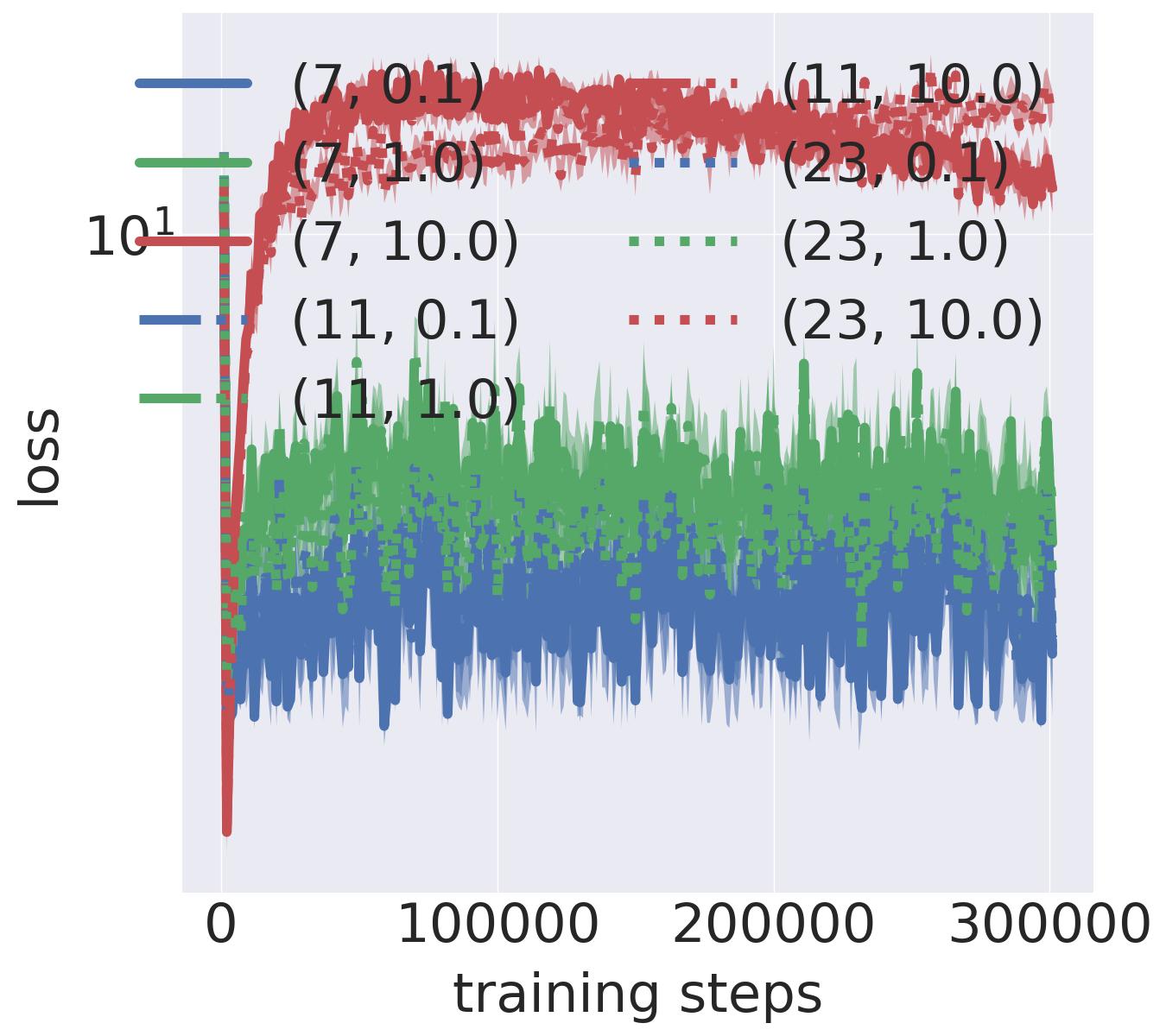}}\\
    \caption{\footnotesize \textsc{fqe} training loss vs. training iterations on the \textsc{d4rl} datasets. \textsc{iqm} of errors for each domain were computed over $20$ trials with $95\%$ confidence intervals. Lower is better. Vertical axis is log-scaled.}
    \label{fig:tr_losses_start}
\end{figure*}

\begin{figure*}[hbtp]
    \centering
        \subfigure[\textsc{fqe} Walker2D-random]{\includegraphics[scale=0.125]{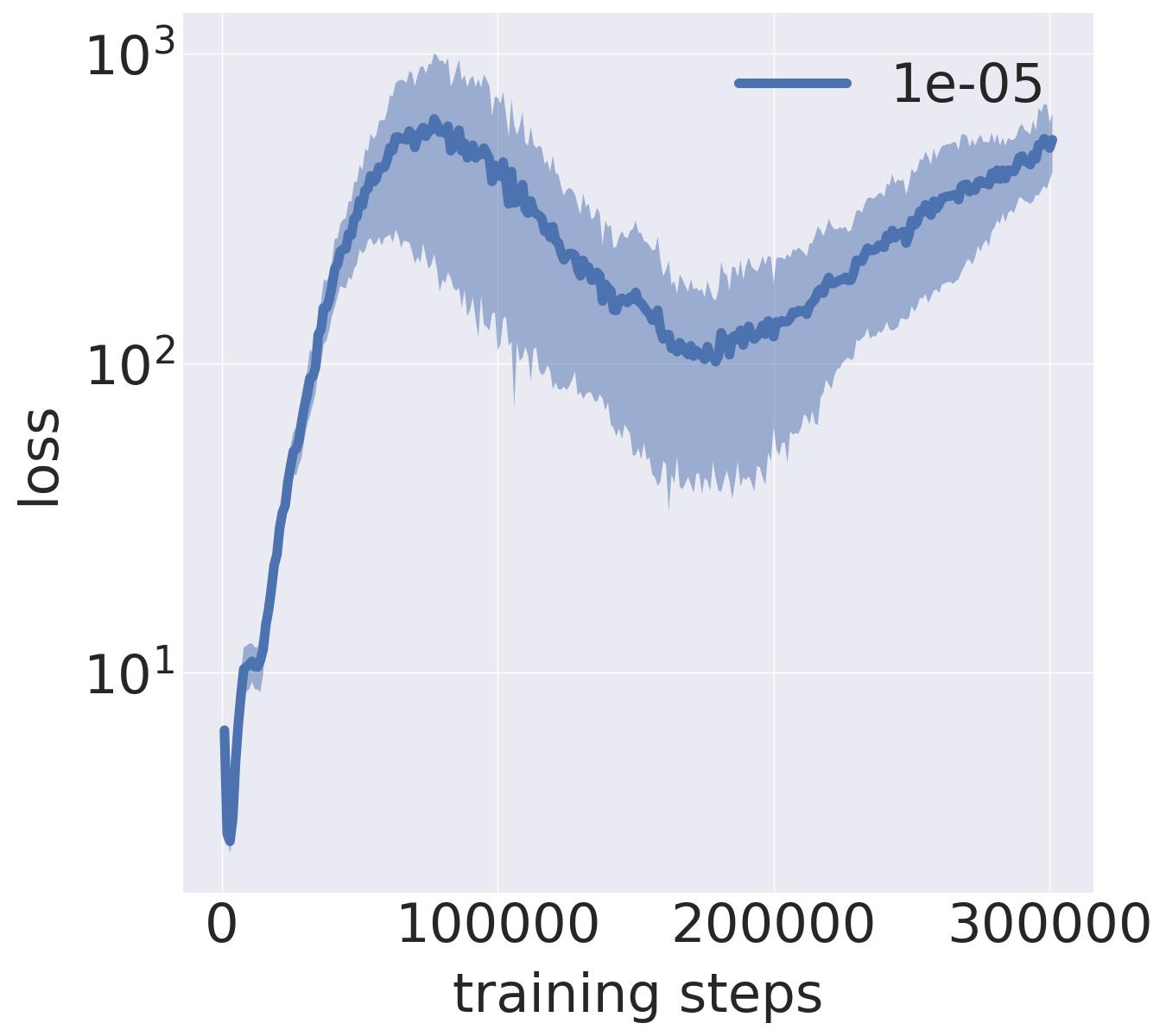}}
        \subfigure[\textsc{fqe-clip} Walker2D-random]{\includegraphics[scale=0.125]{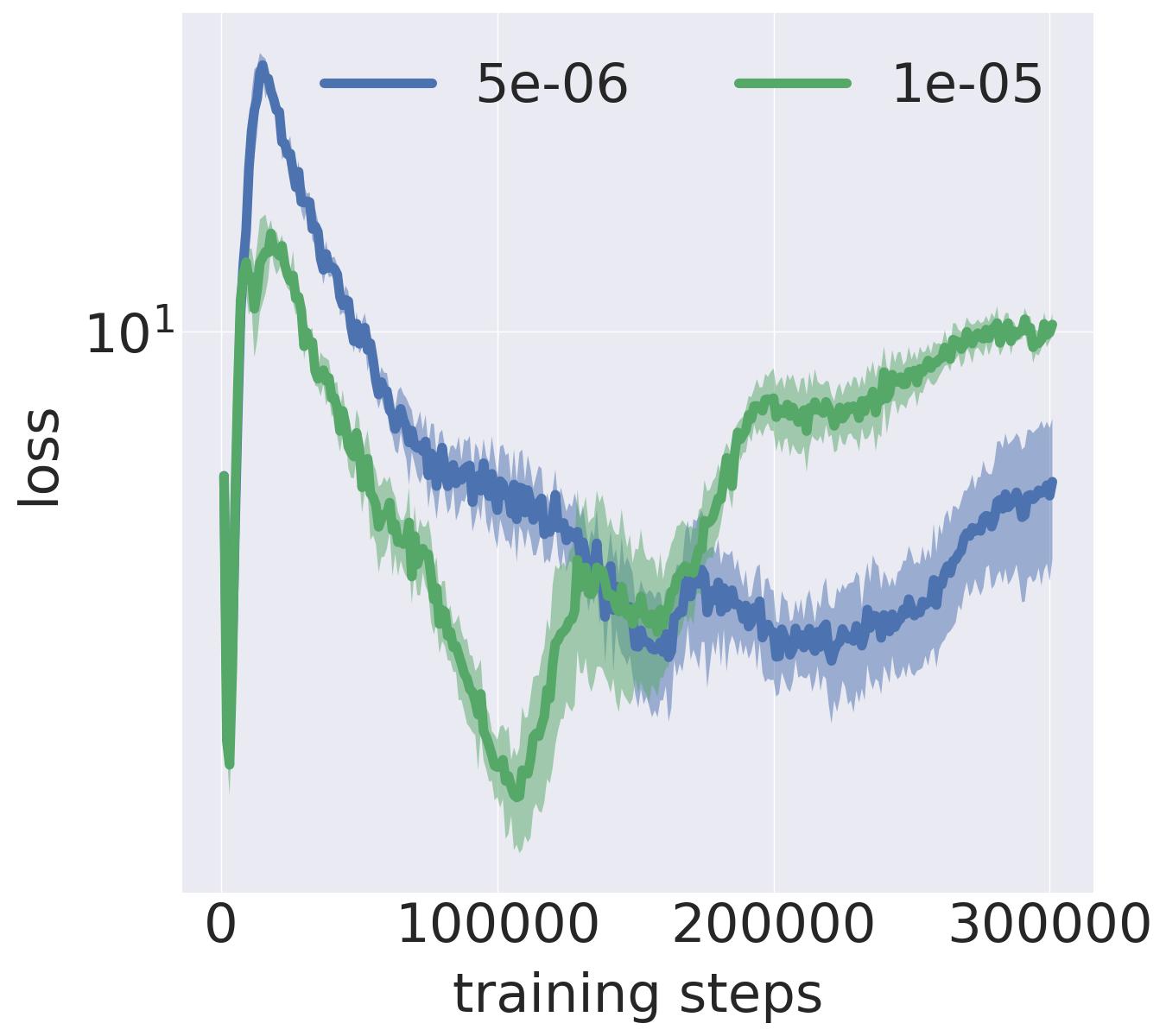}}
        \subfigure[\textsc{fqe} w/ \textsc{rope} Walker2D-random]{\includegraphics[scale=0.125]{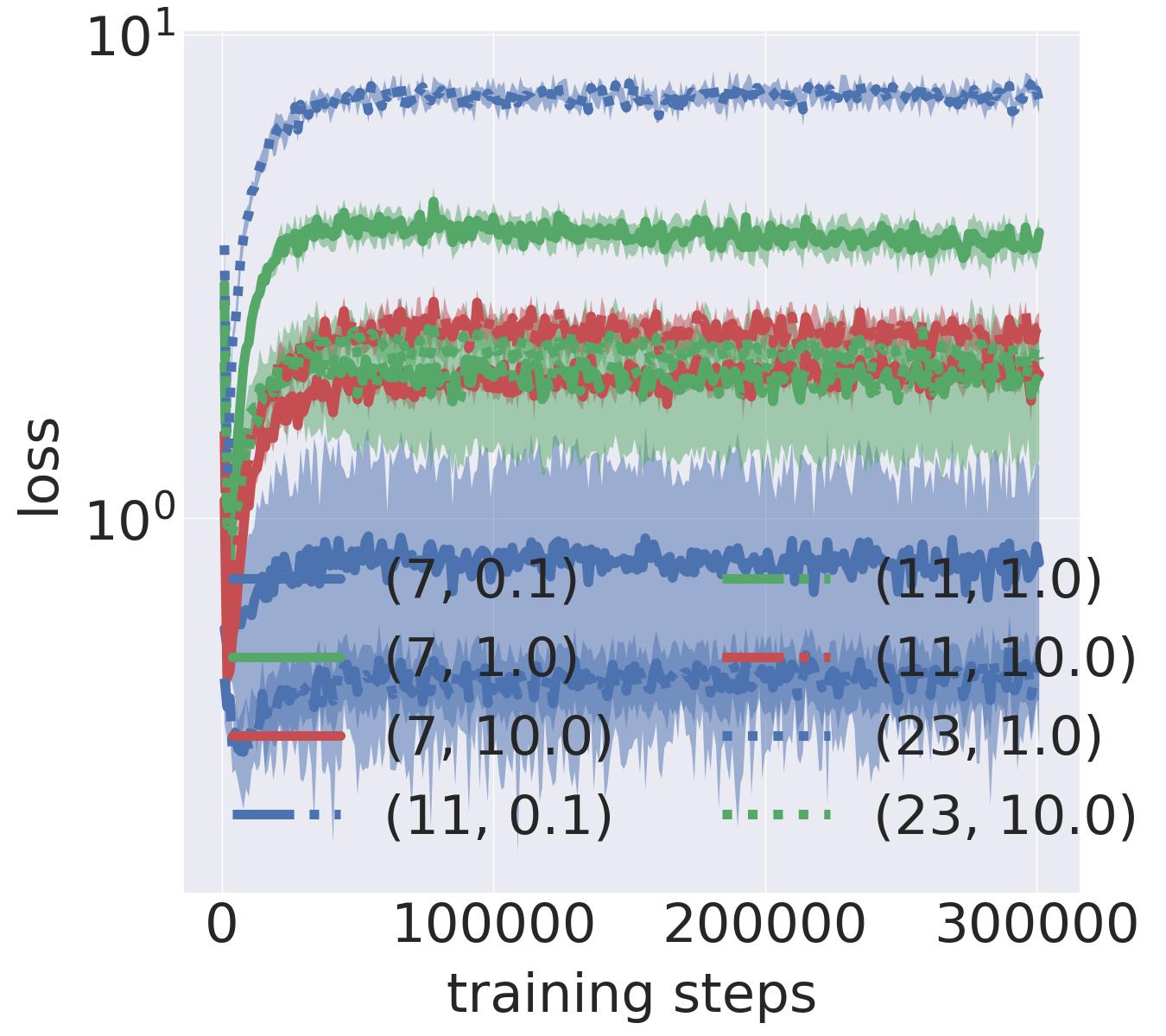}}\\
        \subfigure[\textsc{fqe} Walker2D-medium]{\includegraphics[scale=0.125]{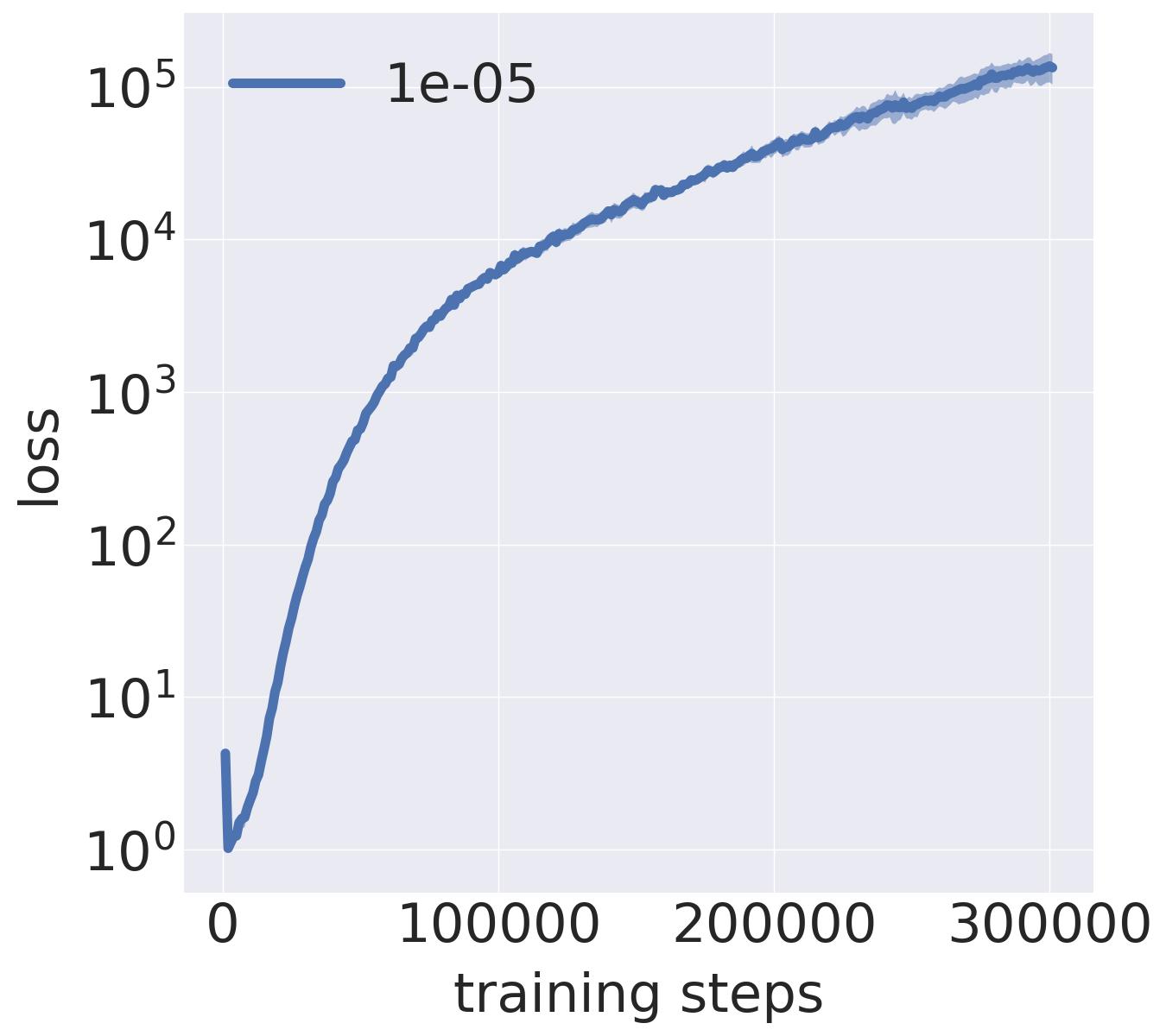}}
        \subfigure[\textsc{fqe-clip} Walker2D-medium]{\includegraphics[scale=0.125]{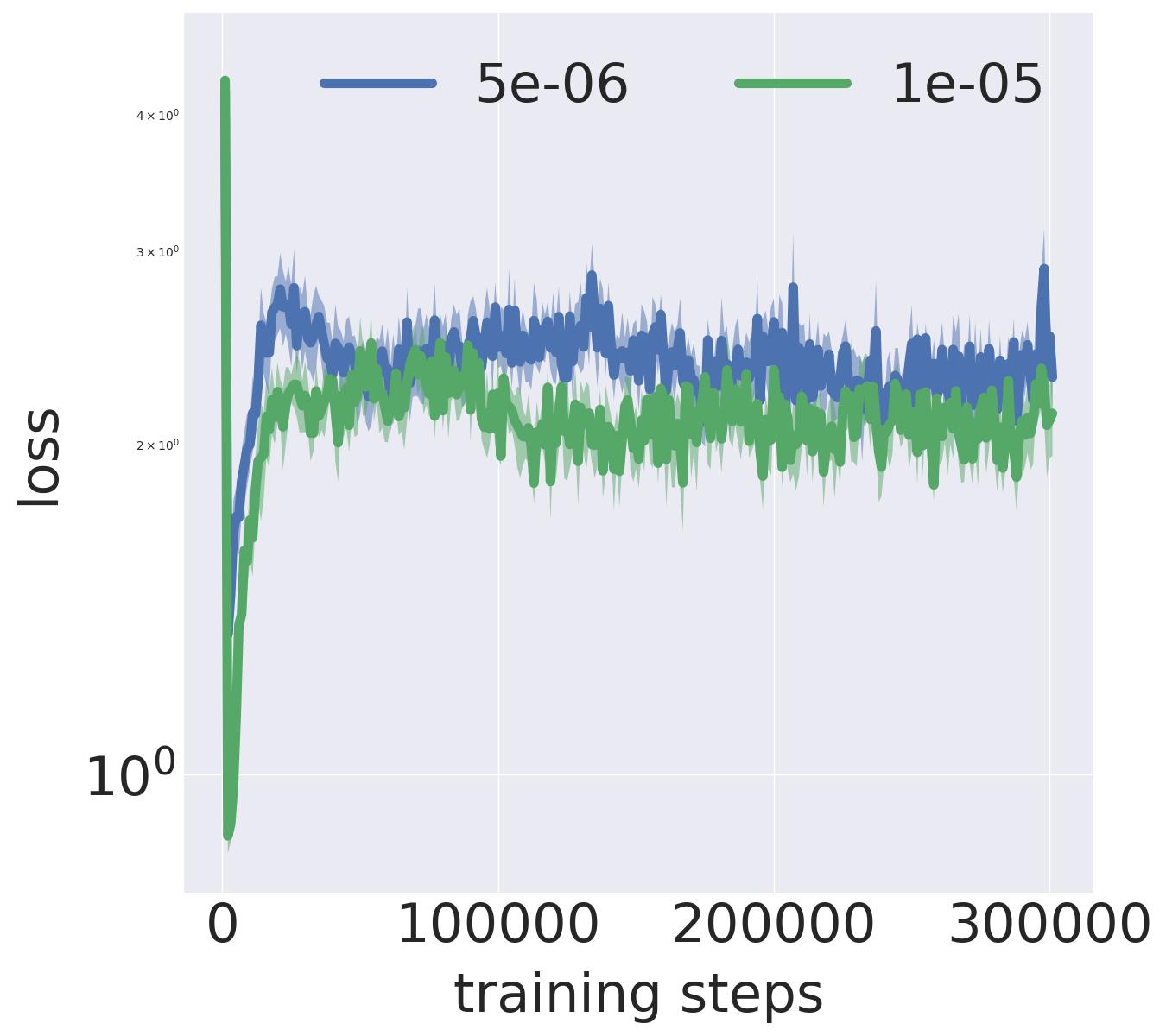}}
        \subfigure[\textsc{fqe} w/ \textsc{rope} Walker2D-medium]{\includegraphics[scale=0.125]{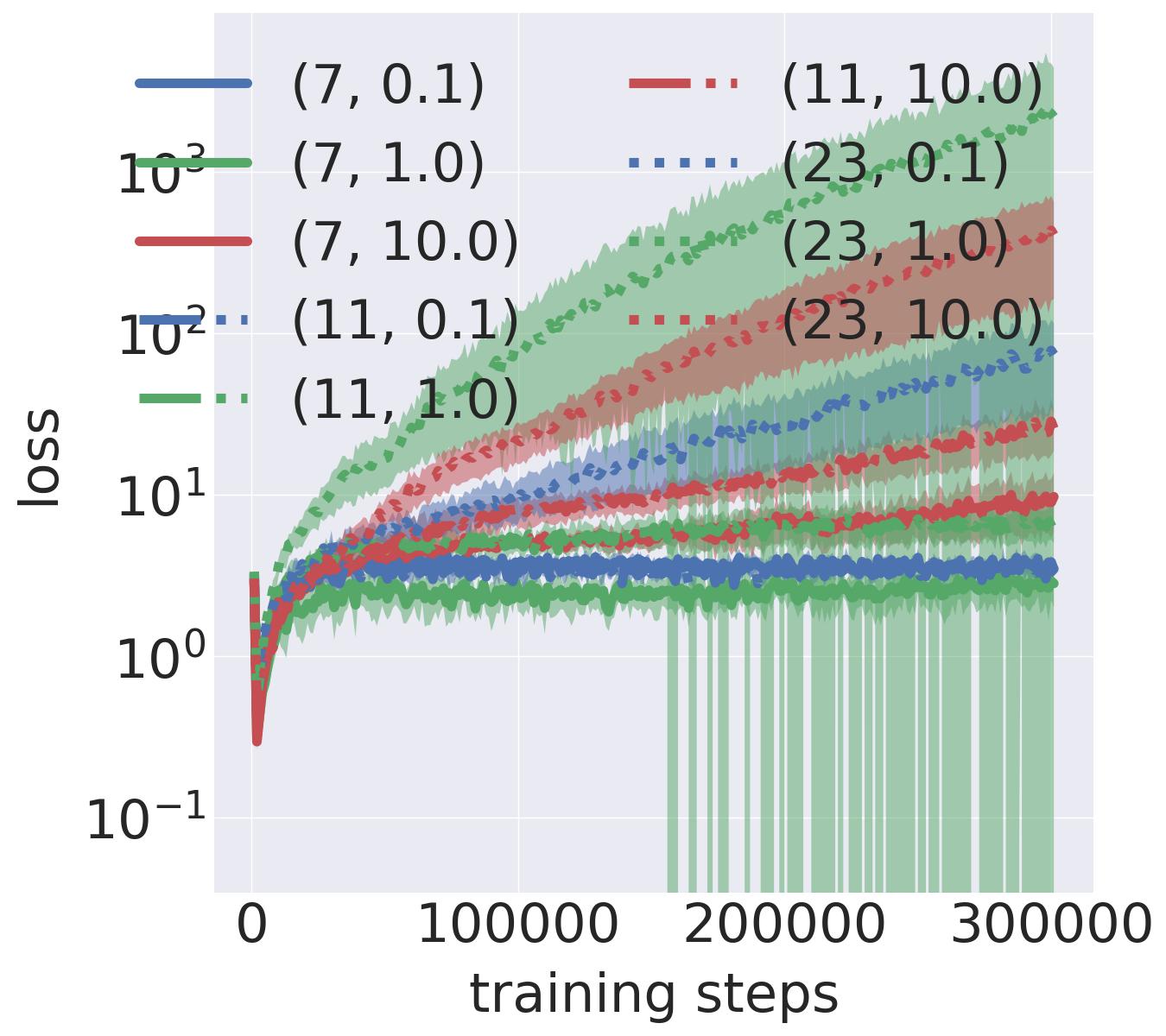}}\\
        \subfigure[\textsc{fqe} Walker2D-medium-expert]{\includegraphics[scale=0.125]{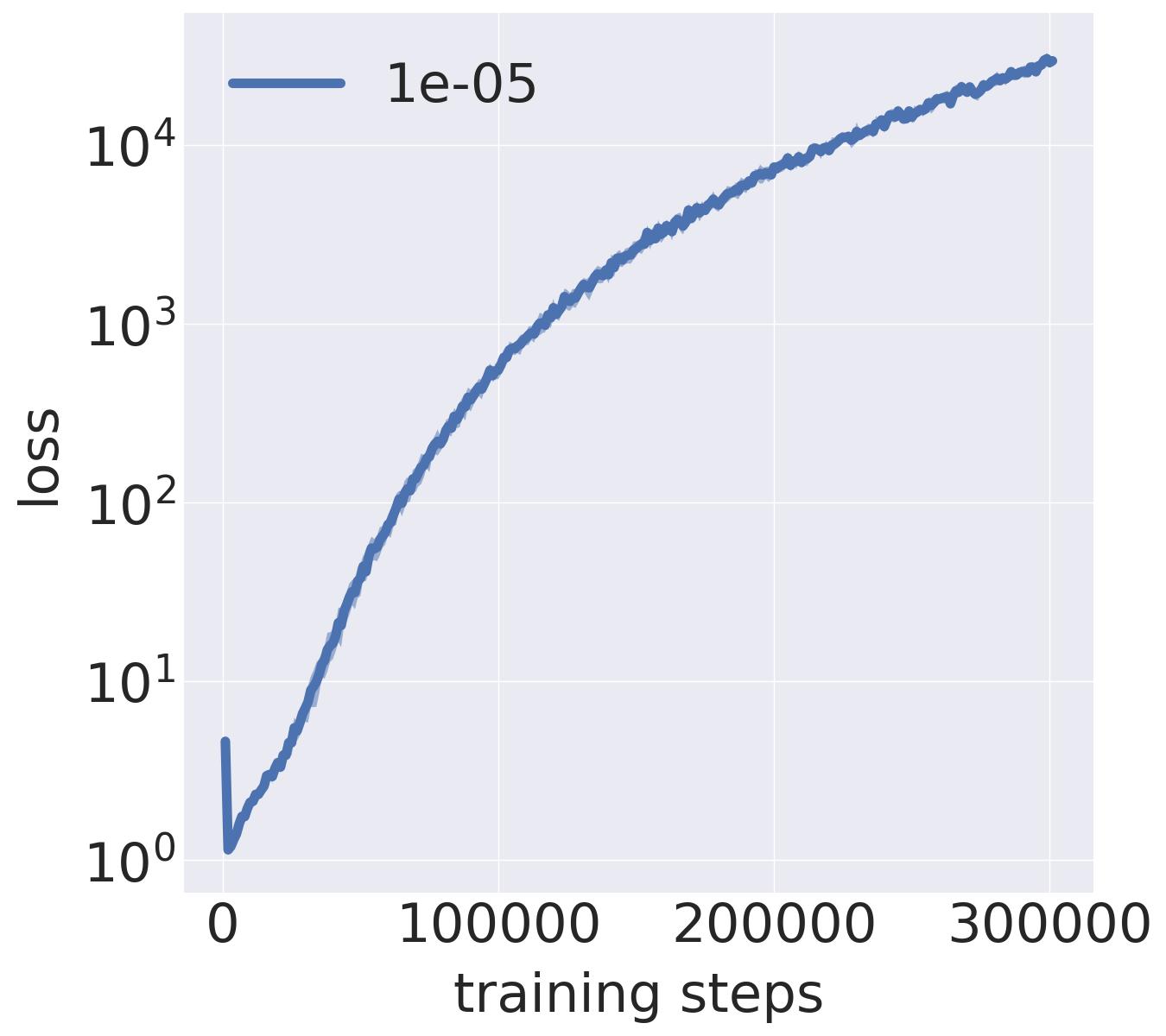}}
        \subfigure[\textsc{fqe-clip} Walker2D-medium-expert]{\includegraphics[scale=0.125]{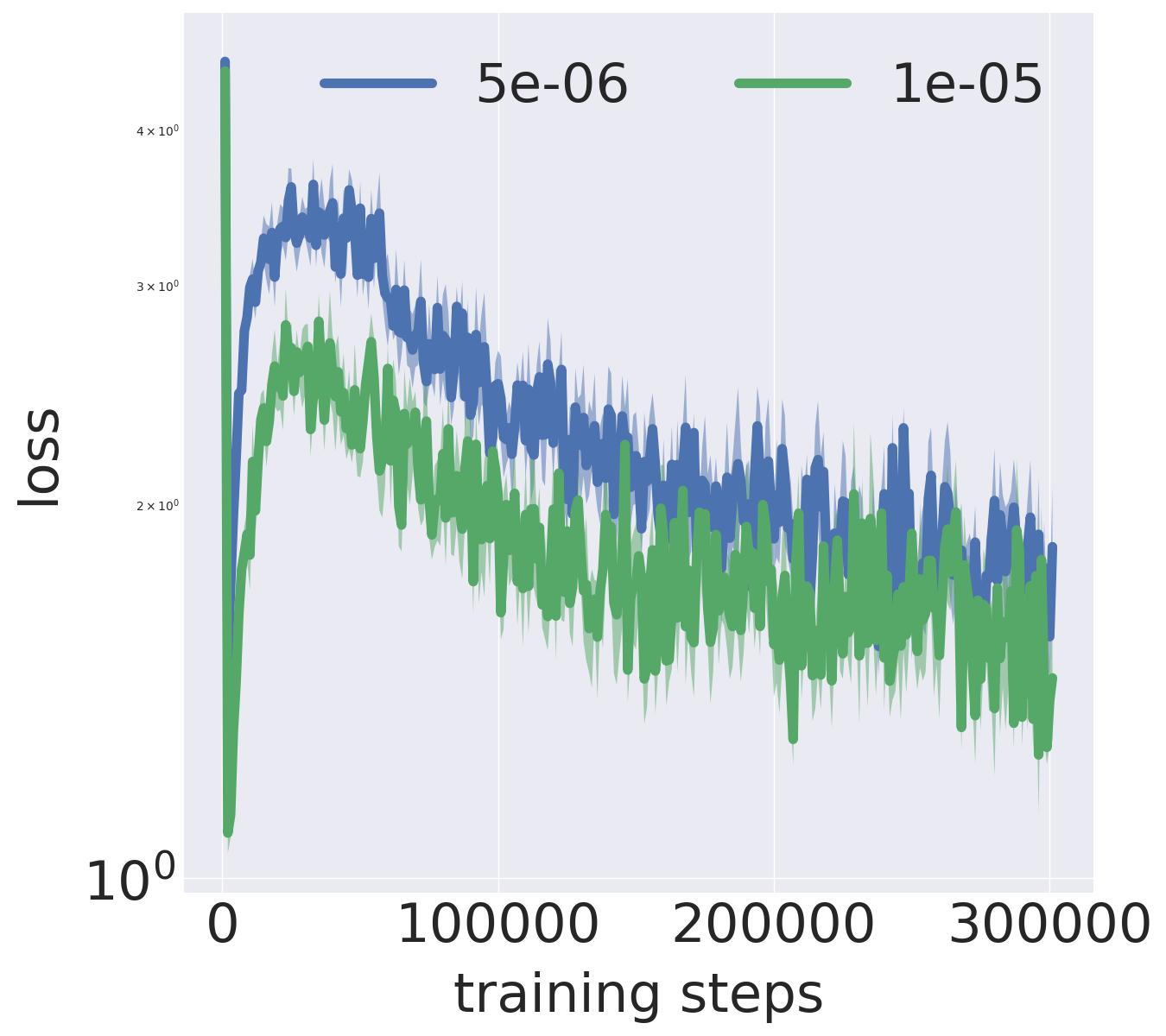}}
        \subfigure[\textsc{fqe} w/ \textsc{fqe} Walker2D-medium-expert]{\includegraphics[scale=0.125]{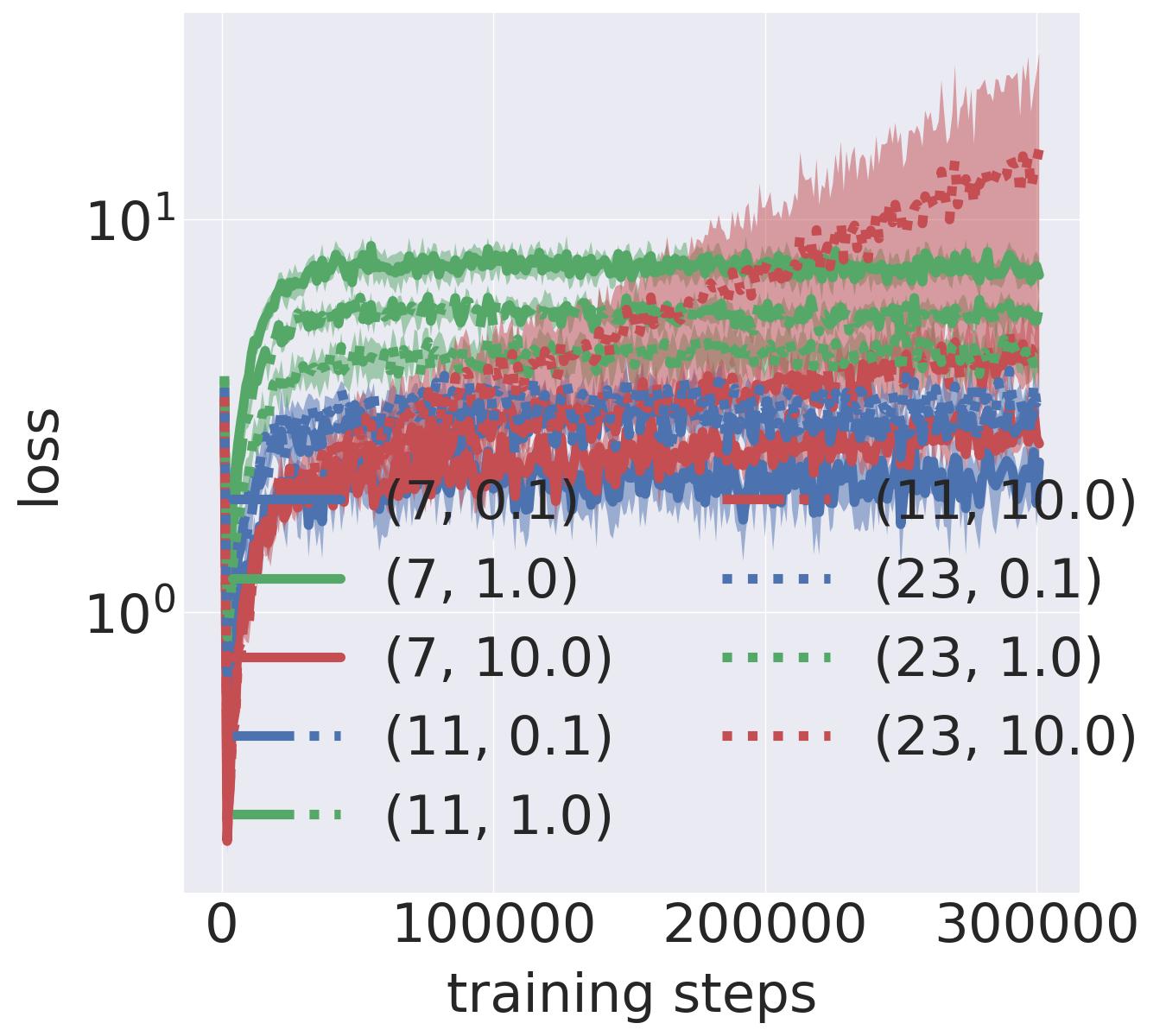}}\\
    \caption{\footnotesize \textsc{fqe} training loss vs. training iterations on the \textsc{d4rl} datasets. \textsc{iqm} of errors for each domain were computed over $20$ trials with $95\%$ confidence intervals. Lower is better. Vertical axis is log-scaled.}
\end{figure*}

\begin{figure*}[hbtp]
    \centering
        \subfigure[\textsc{fqe} Hopper-random]{\includegraphics[scale=0.125]{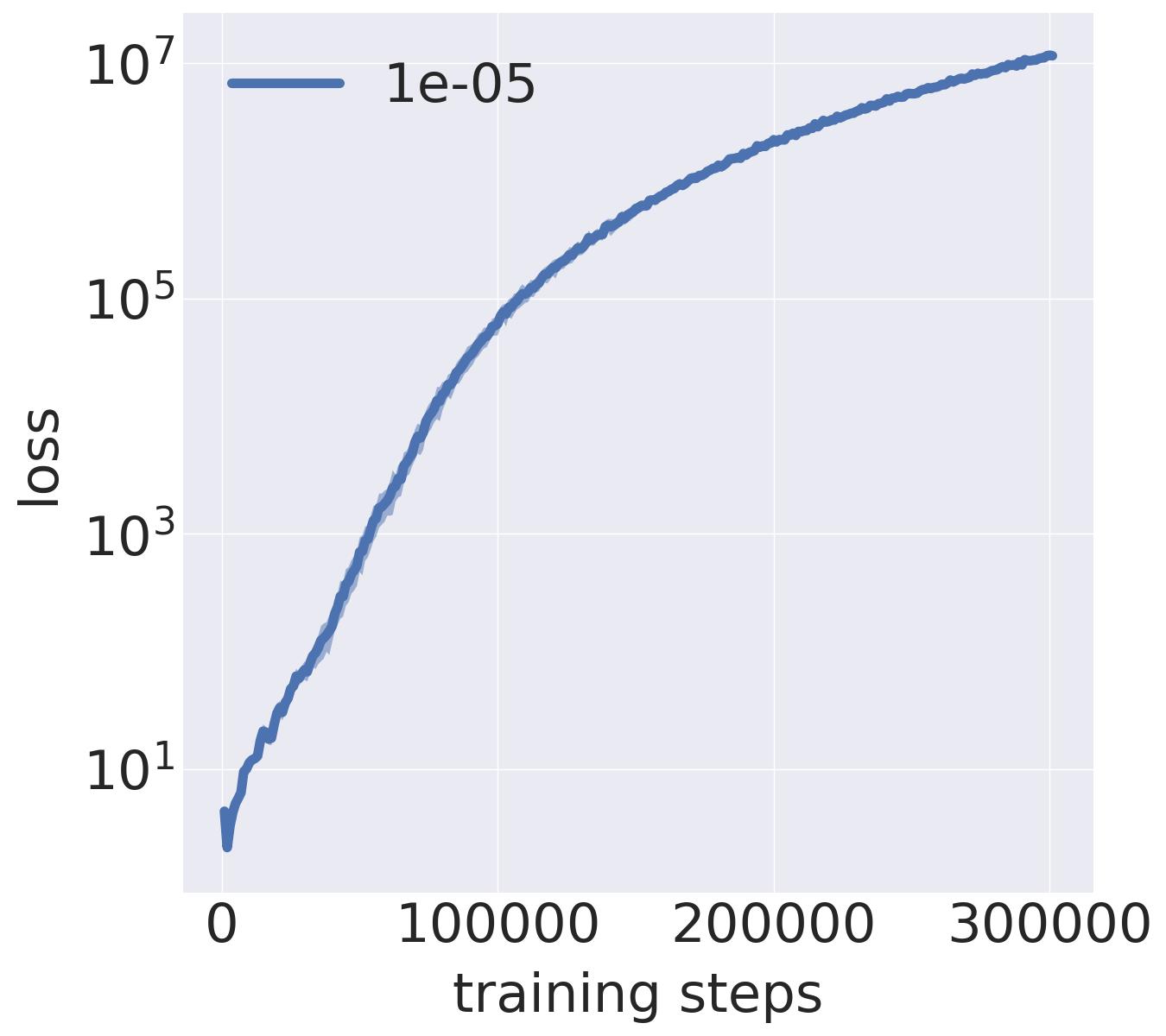}}
        \subfigure[\textsc{fqe-clip} Hopper-random]{\includegraphics[scale=0.125]{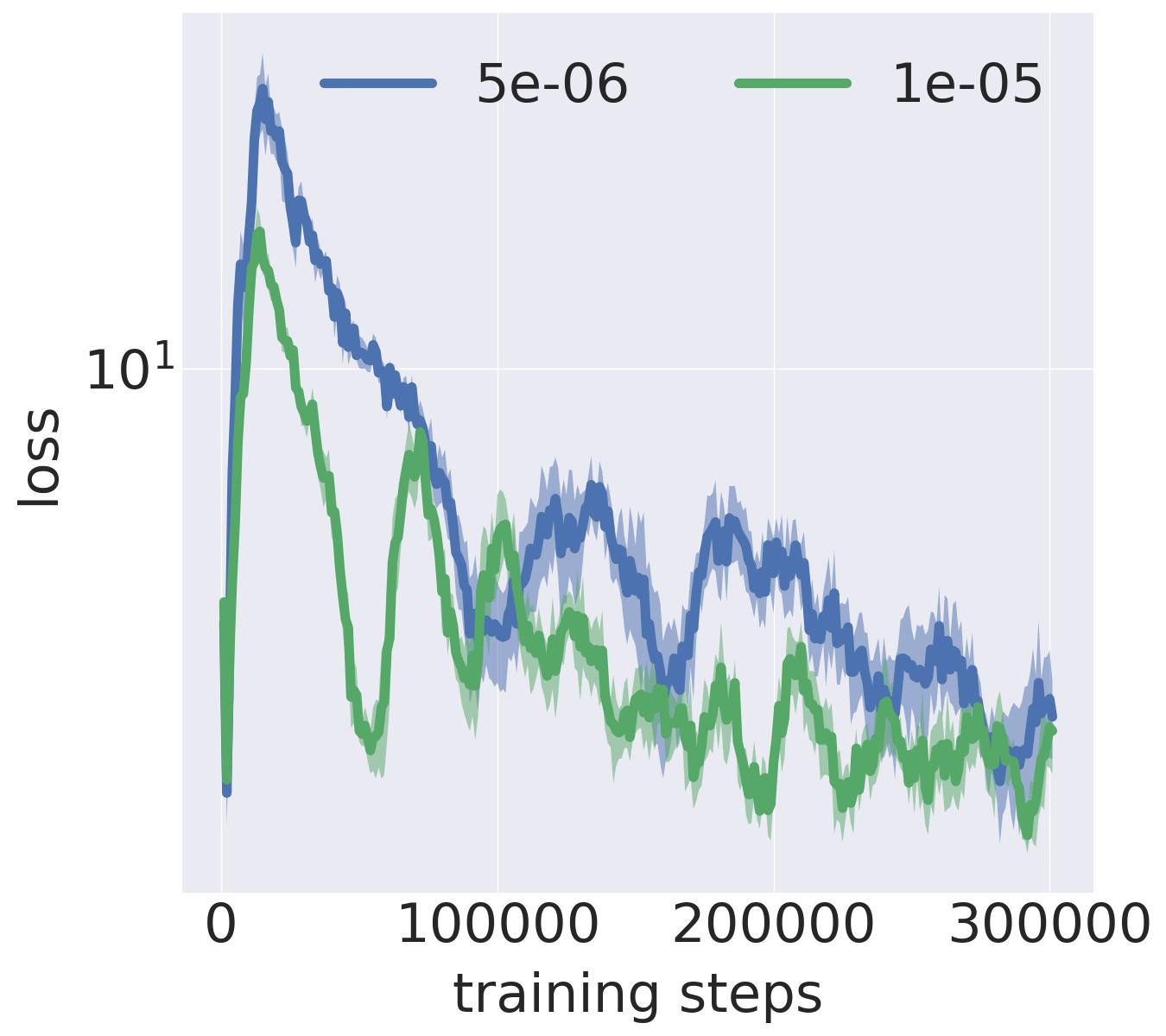}}
        \subfigure[\textsc{fqe} w/ \textsc{rope} Hopper-random]{\includegraphics[scale=0.125]{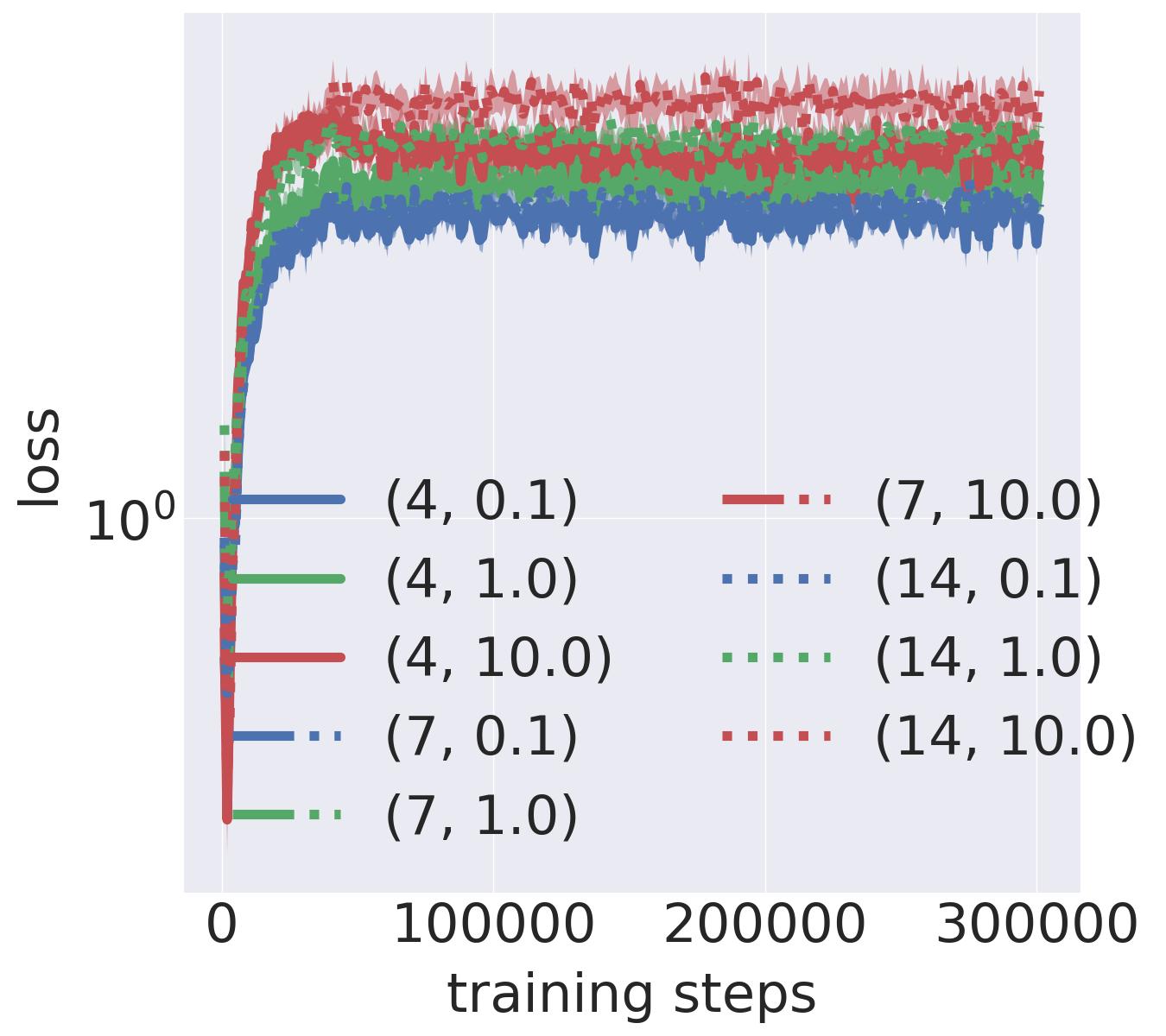}}\\
        \subfigure[\textsc{fqe} Hopper-medium]{\includegraphics[scale=0.125]{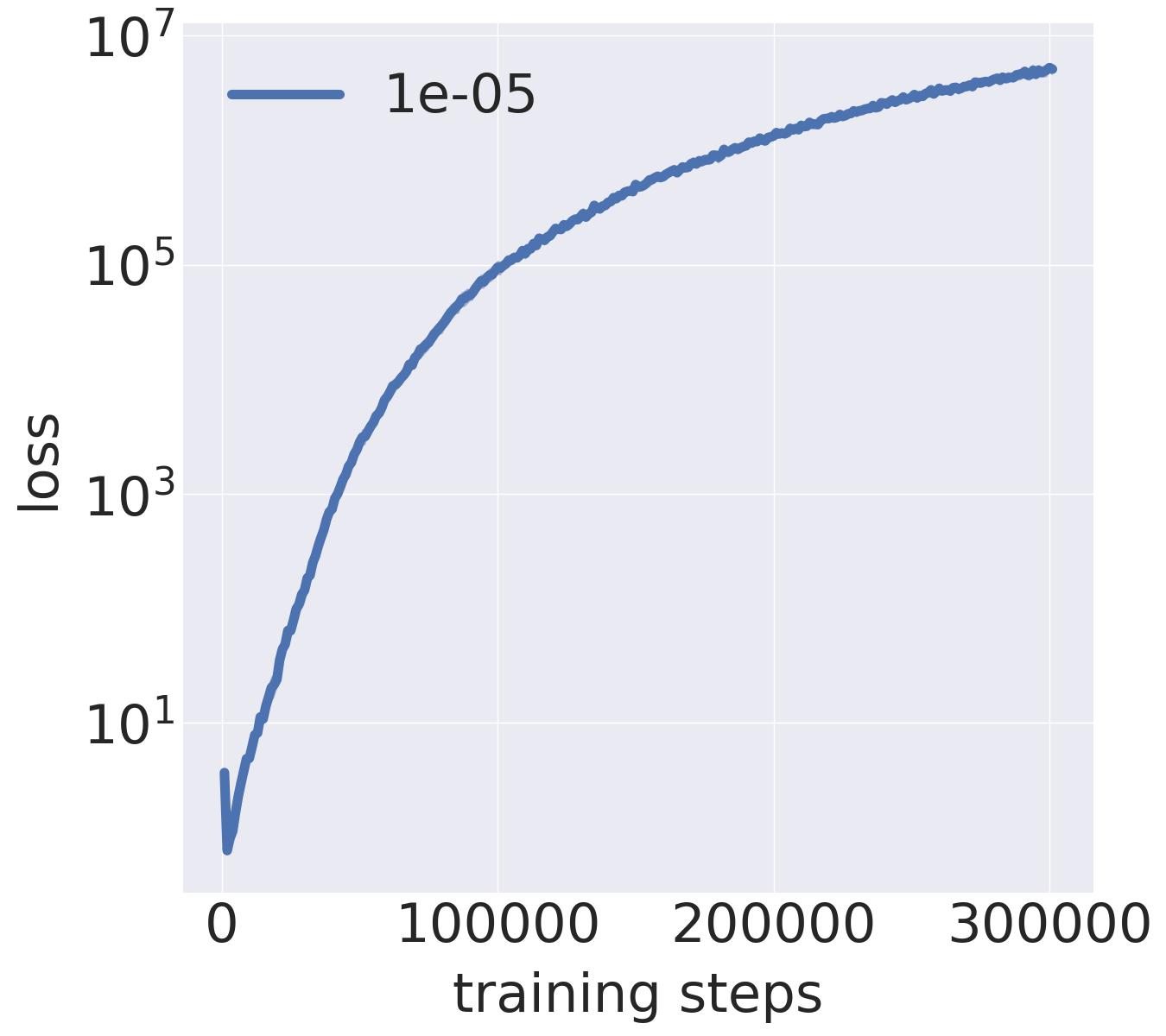}}
        \subfigure[\textsc{fqe-clip} Hopper-medium]{\includegraphics[scale=0.125]{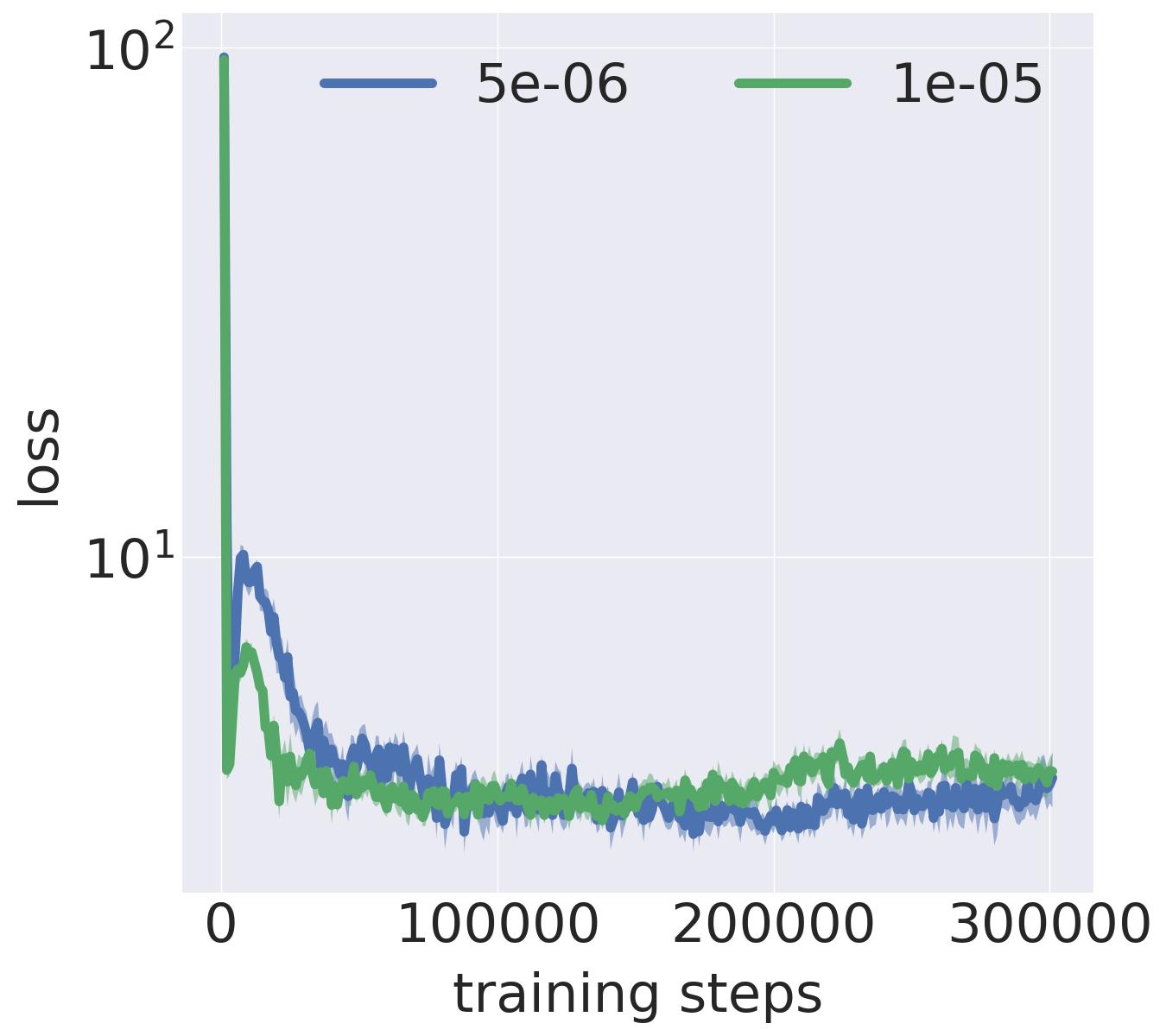}}
        \subfigure[\textsc{fqe} w/ \textsc{rope} Hopper-medium]{\includegraphics[scale=0.125]{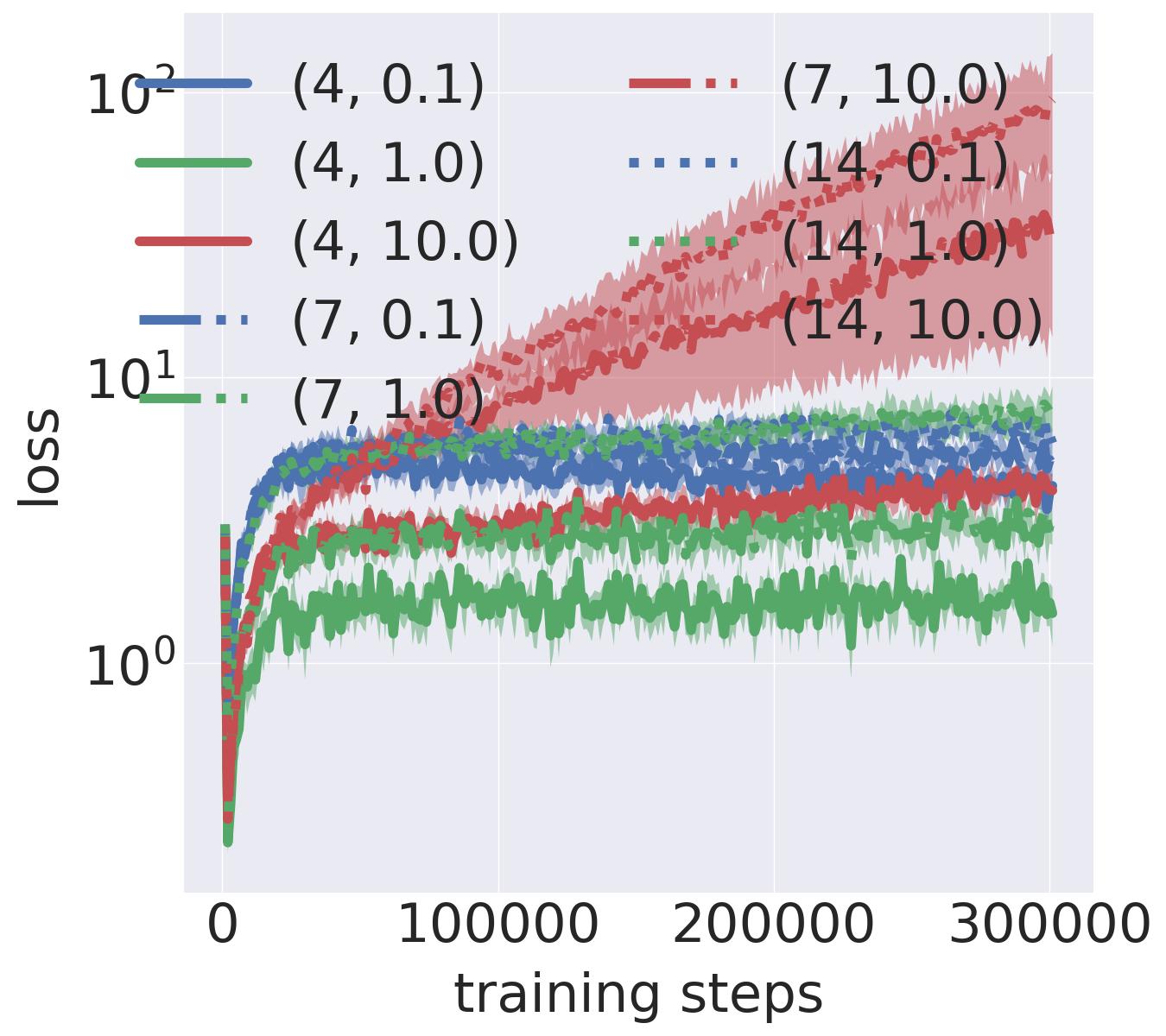}}\\
        \subfigure[\textsc{fqe} Hopper-medium-expert]{\includegraphics[scale=0.125]{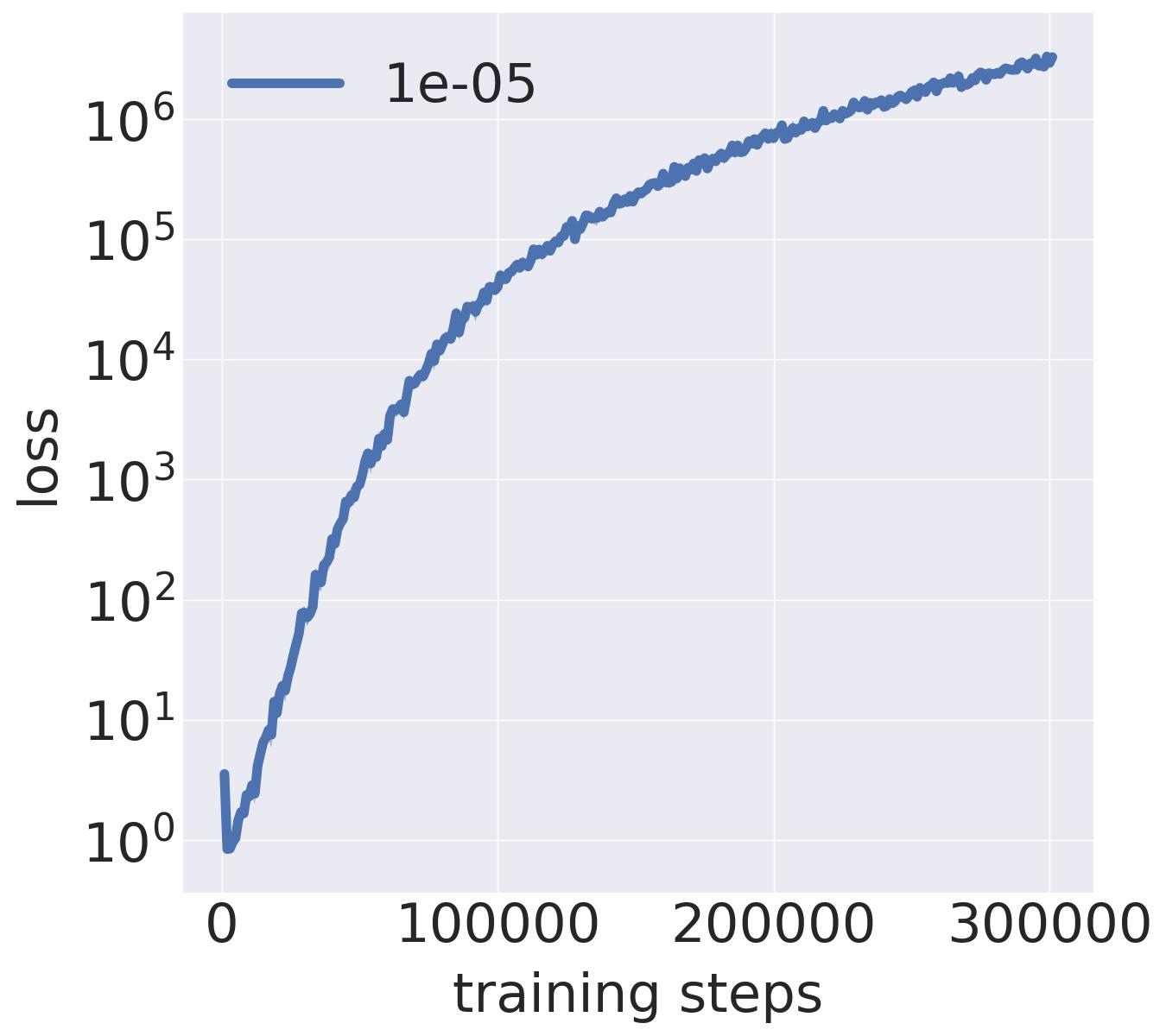}}
        \subfigure[\textsc{fqe-clip} Hopper-medium-expert]{\includegraphics[scale=0.125]{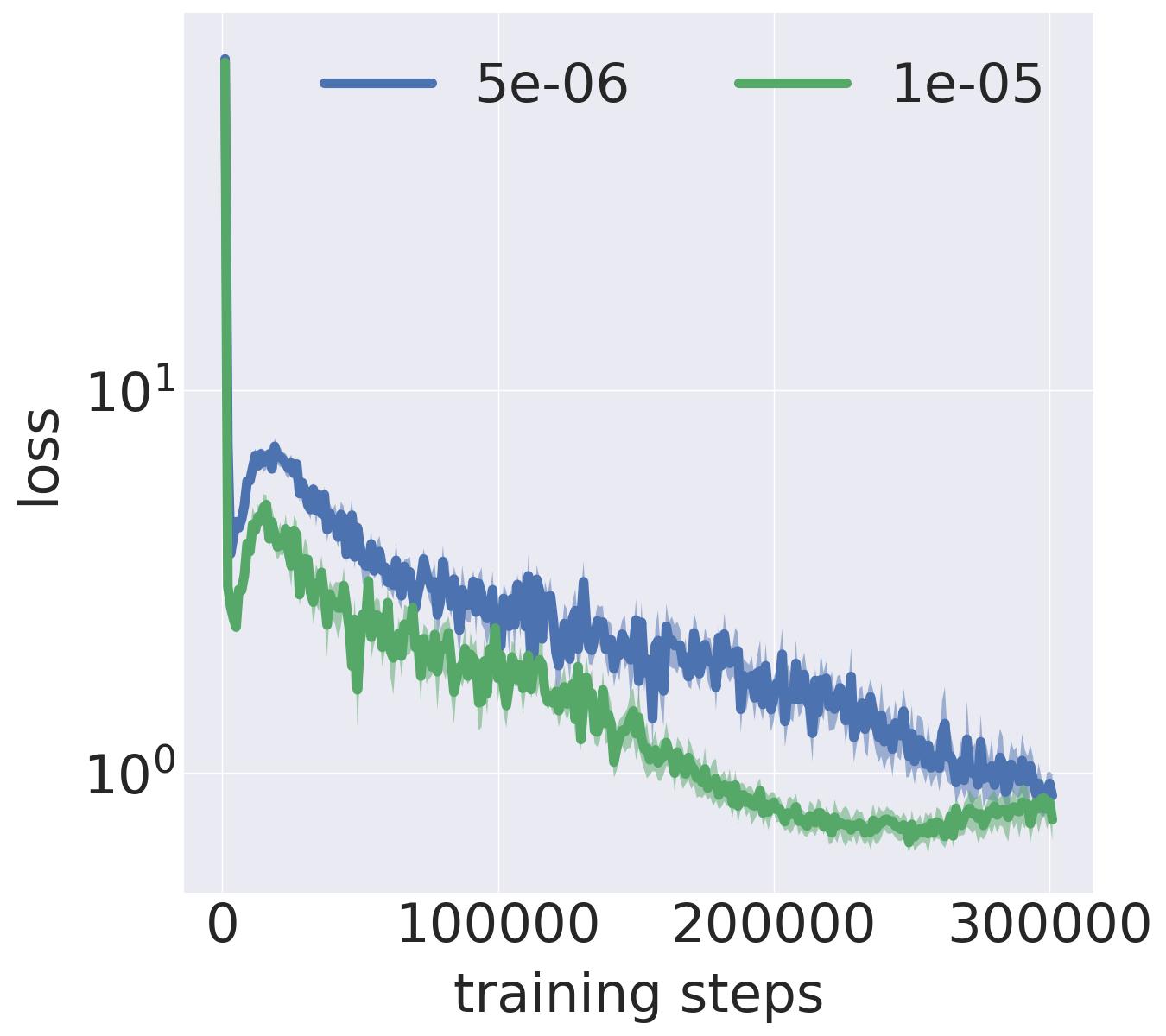}}
        \subfigure[\textsc{fqe} w/ \textsc{rope} Hopper-medium-expert]{\includegraphics[scale=0.125]{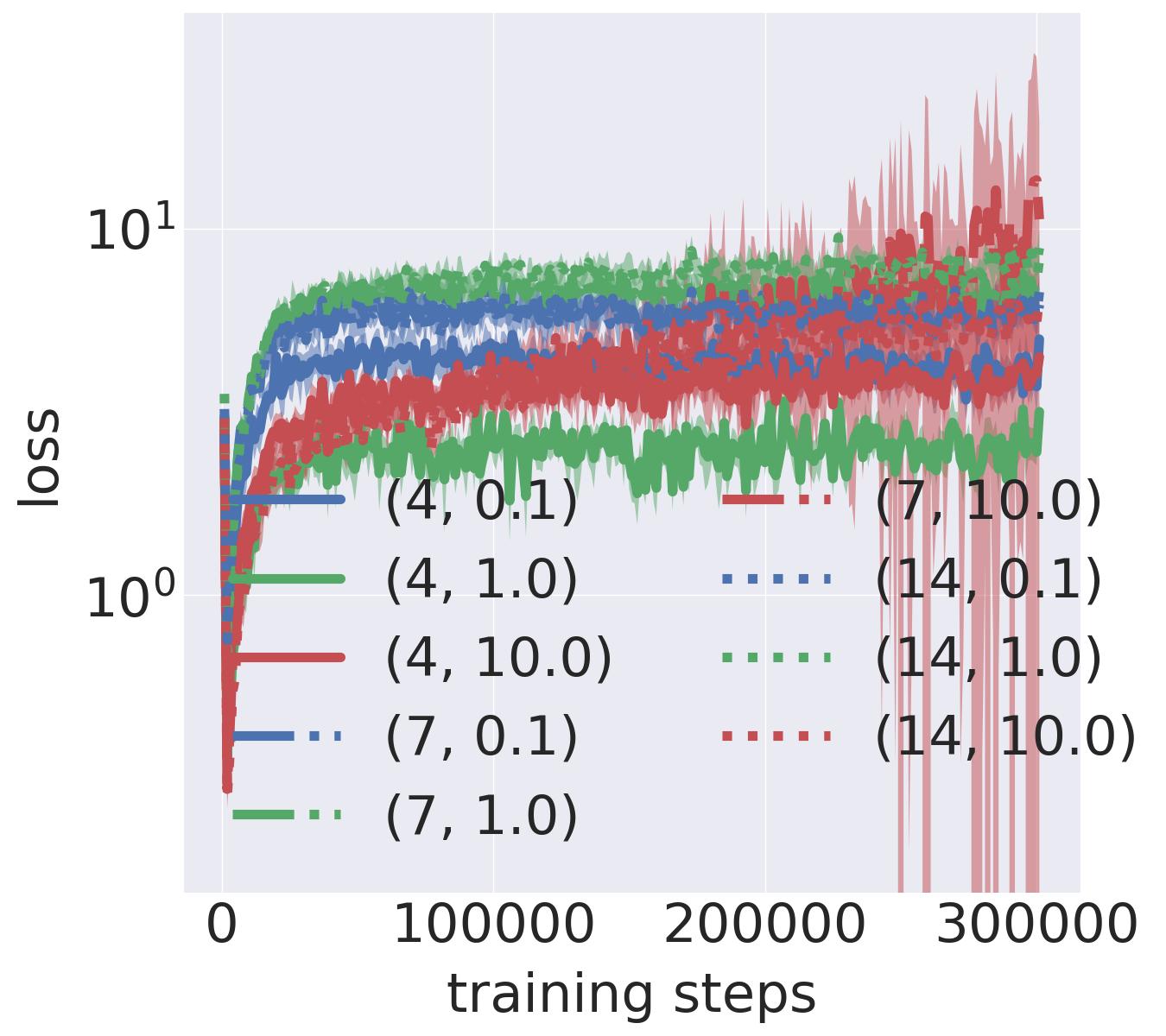}}\\
    \caption{\footnotesize \textsc{fqe} training loss vs. training iterations on the \textsc{d4rl} datasets. \textsc{iqm} of errors for each domain were computed over $20$ trials with $95\%$ confidence intervals. Lower is better. Vertical axis is log-scaled.}
\end{figure*}

\begin{figure*}[hbtp]
    \centering
        \subfigure[HalfCheetah-random]{\includegraphics[scale=0.125]{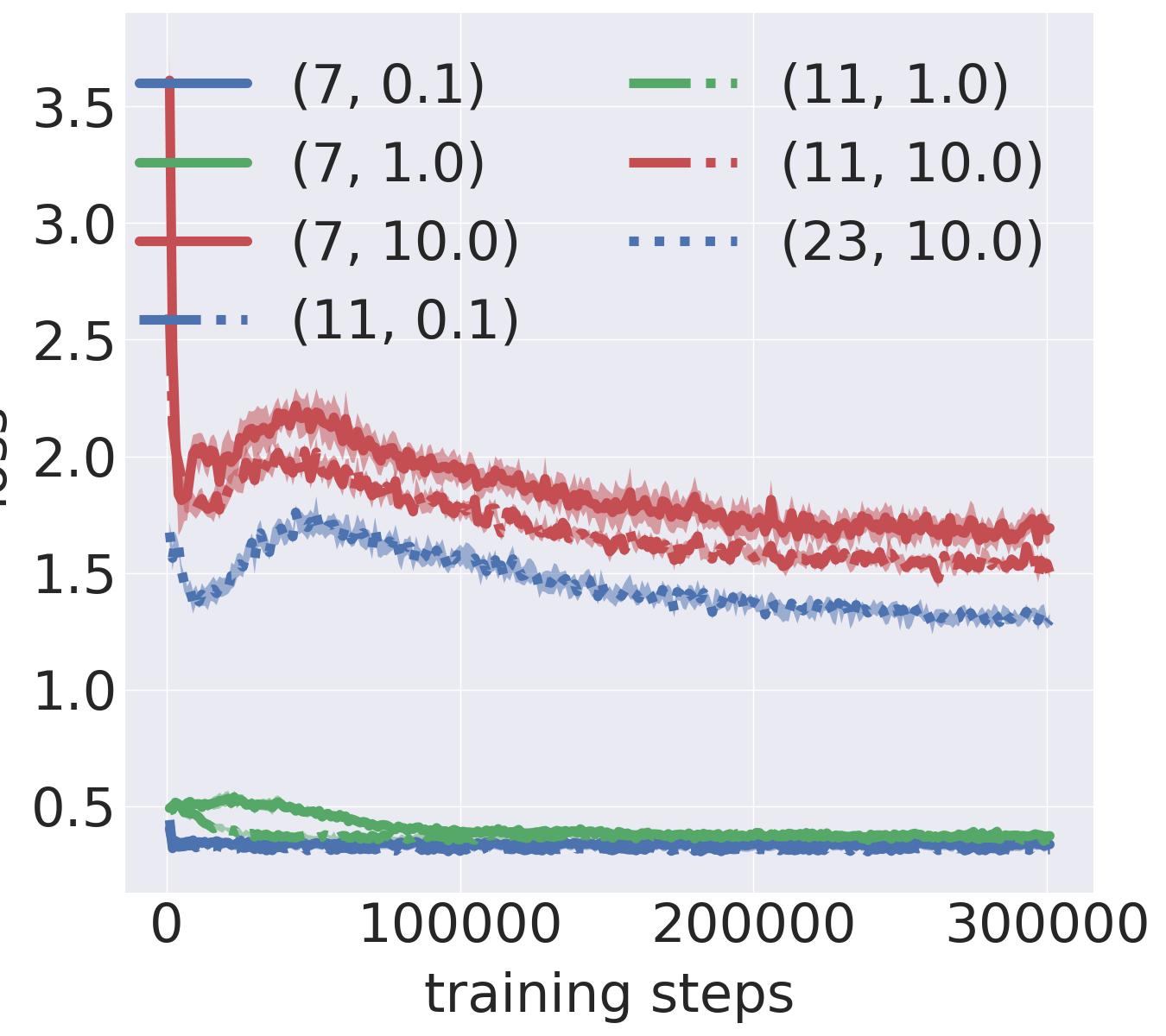}}
        \subfigure[HalfCheetah-medium]{\includegraphics[scale=0.125]{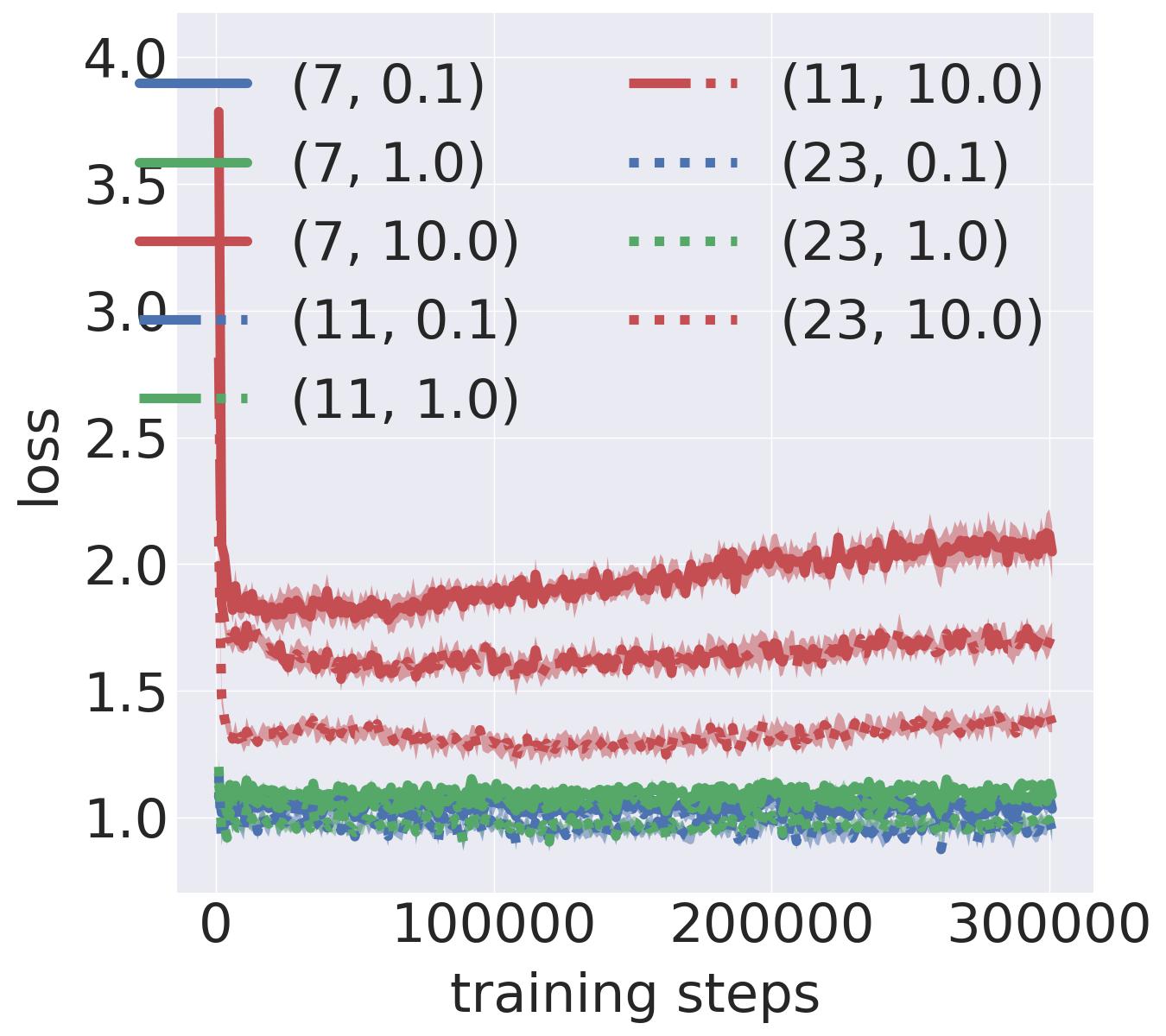}}
        \subfigure[HalfCheetah-medium-expert]{\includegraphics[scale=0.125]{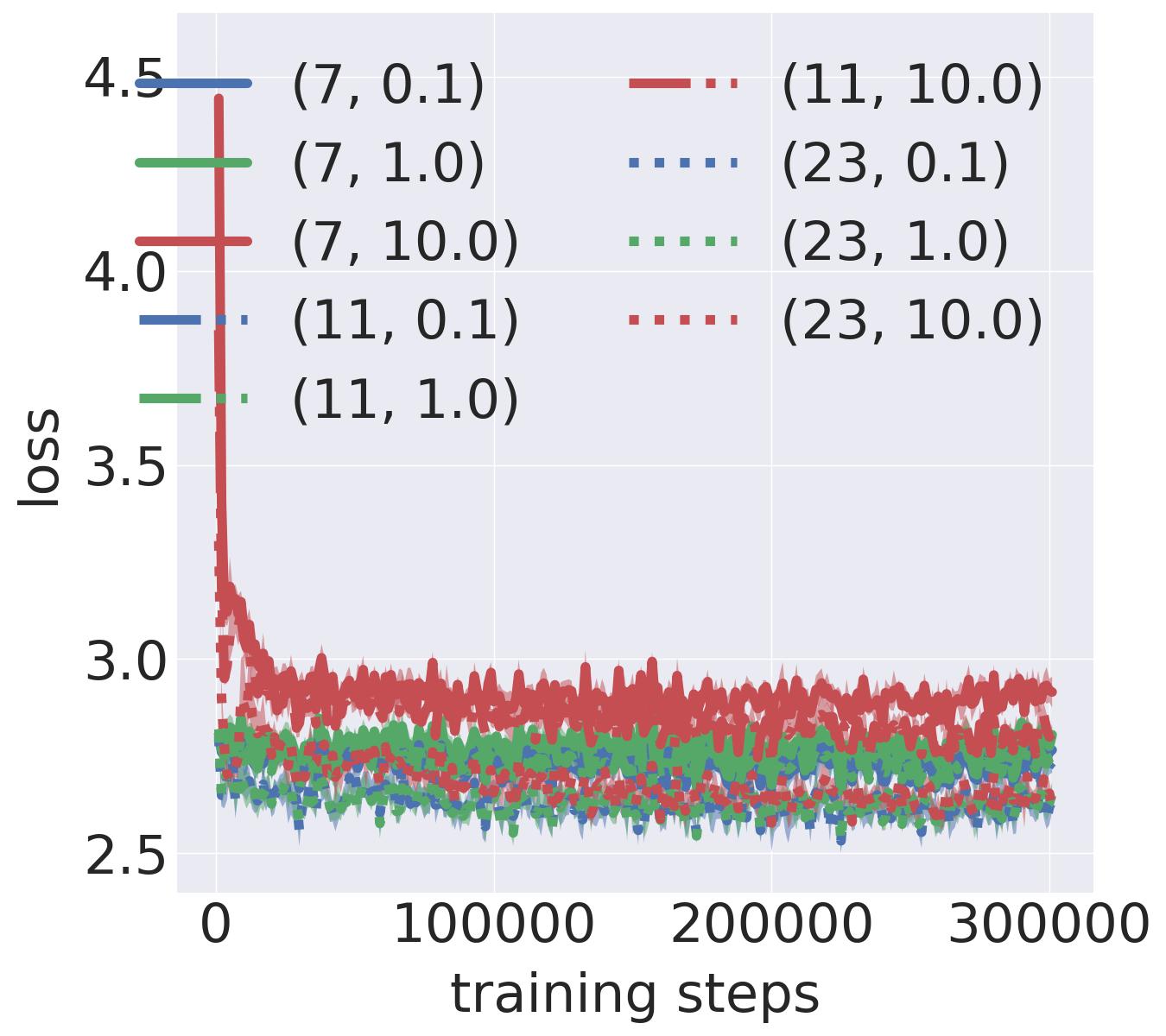}}\\
        \subfigure[Walker2D-random]{\includegraphics[scale=0.125]{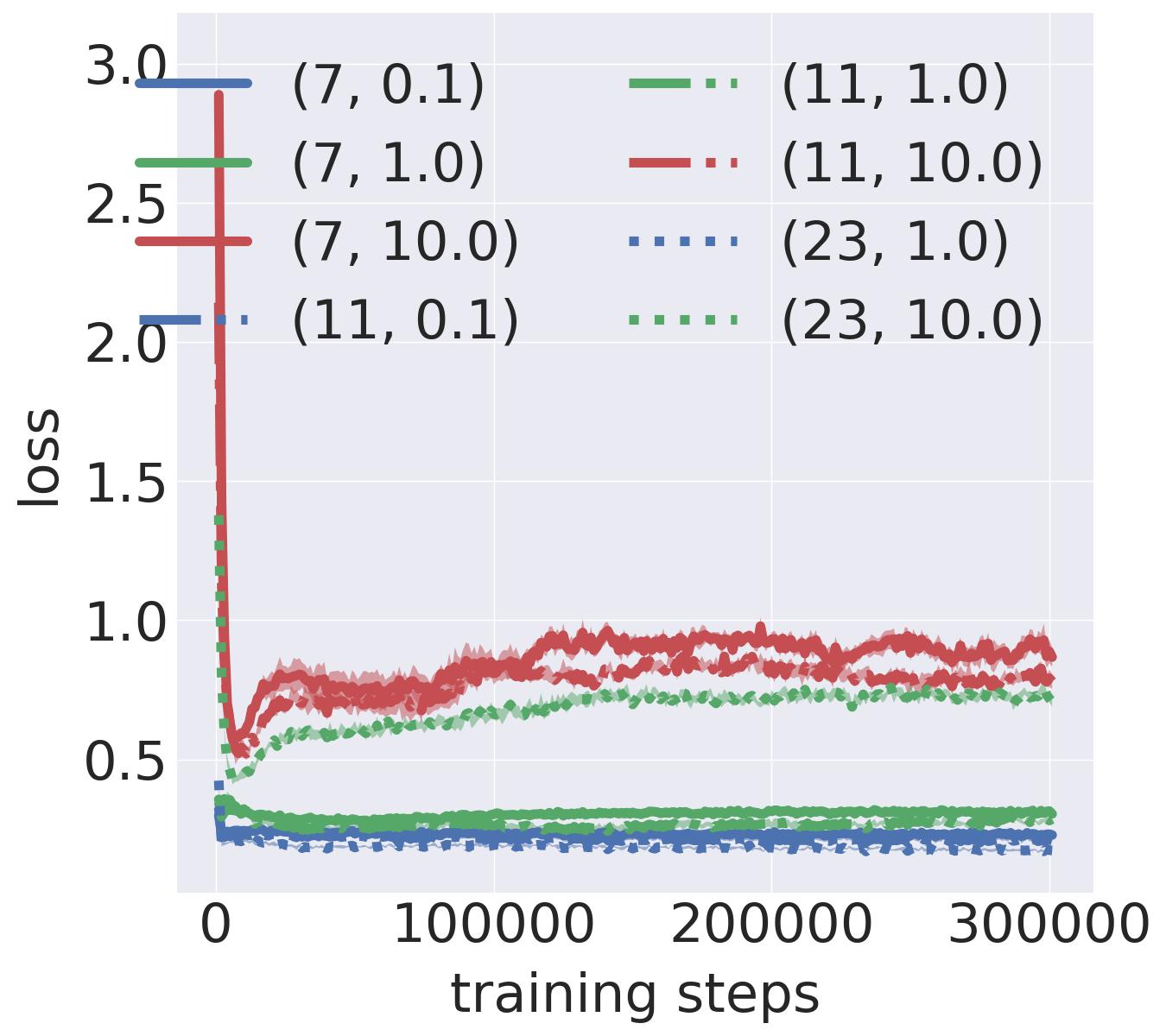}}
        \subfigure[Walker2D-medium]{\includegraphics[scale=0.125]{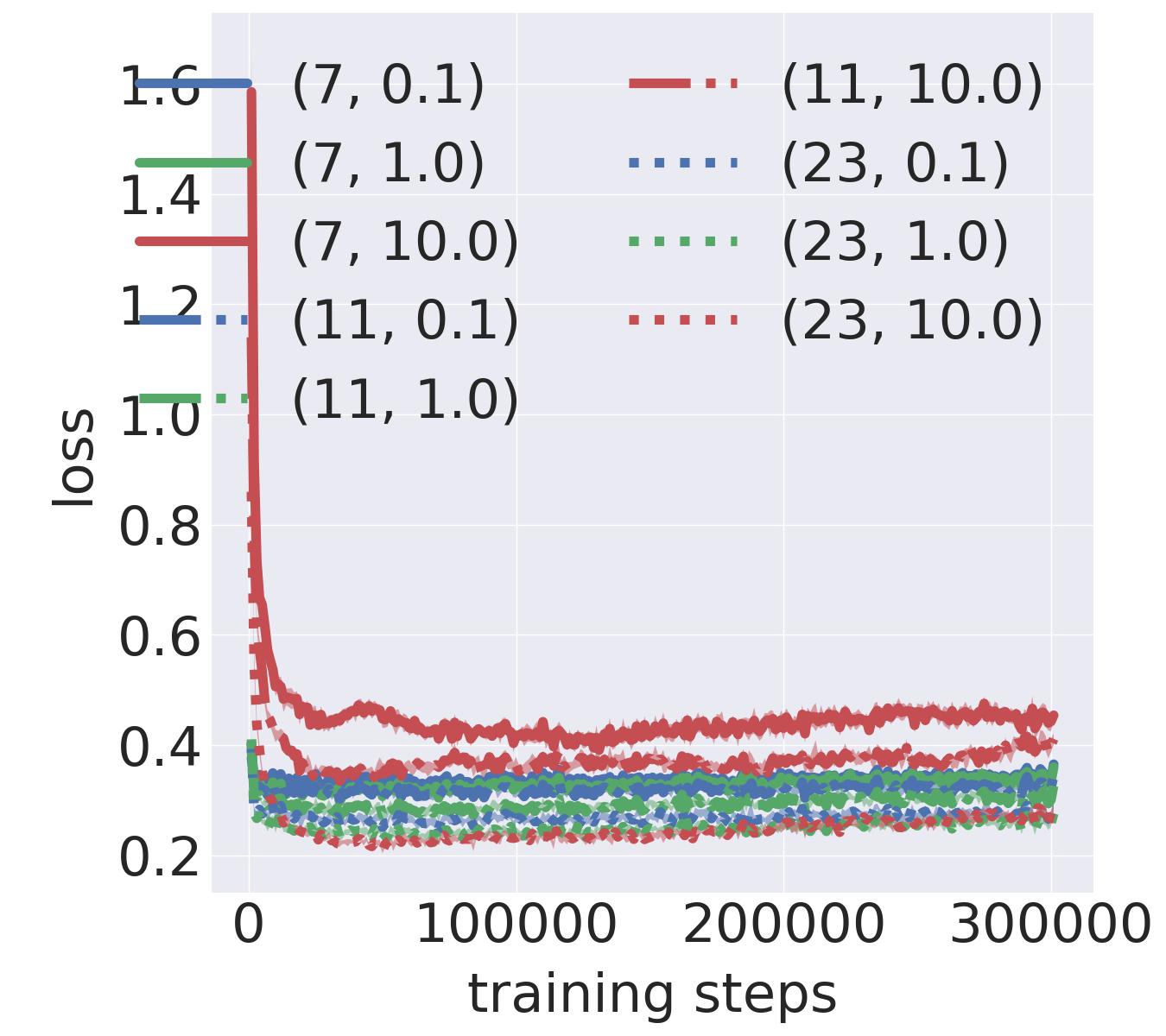}}
        \subfigure[Walker2D-medium-expert]{\includegraphics[scale=0.125]{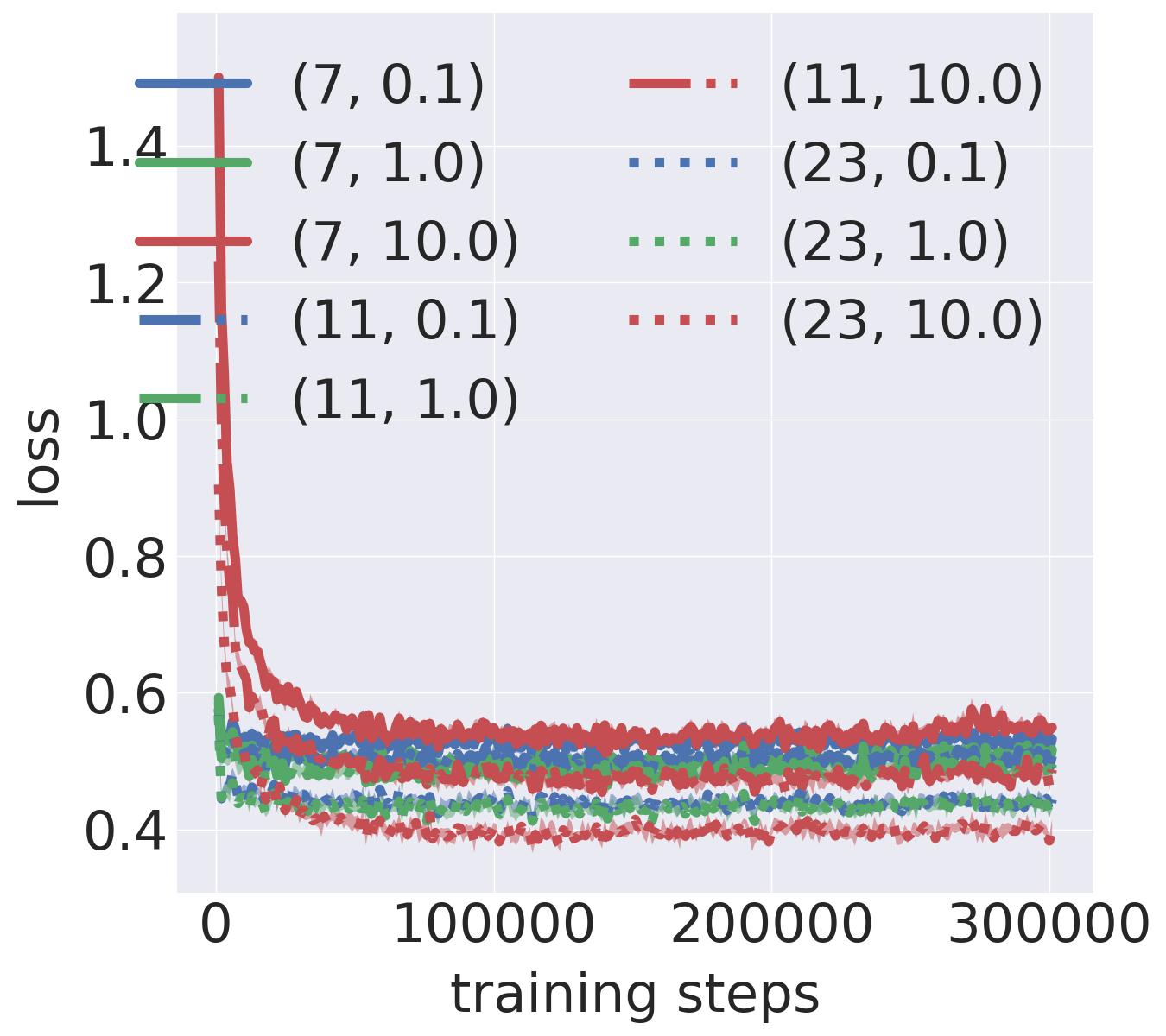}}\\
        \subfigure[Hopper-random]{\includegraphics[scale=0.125]{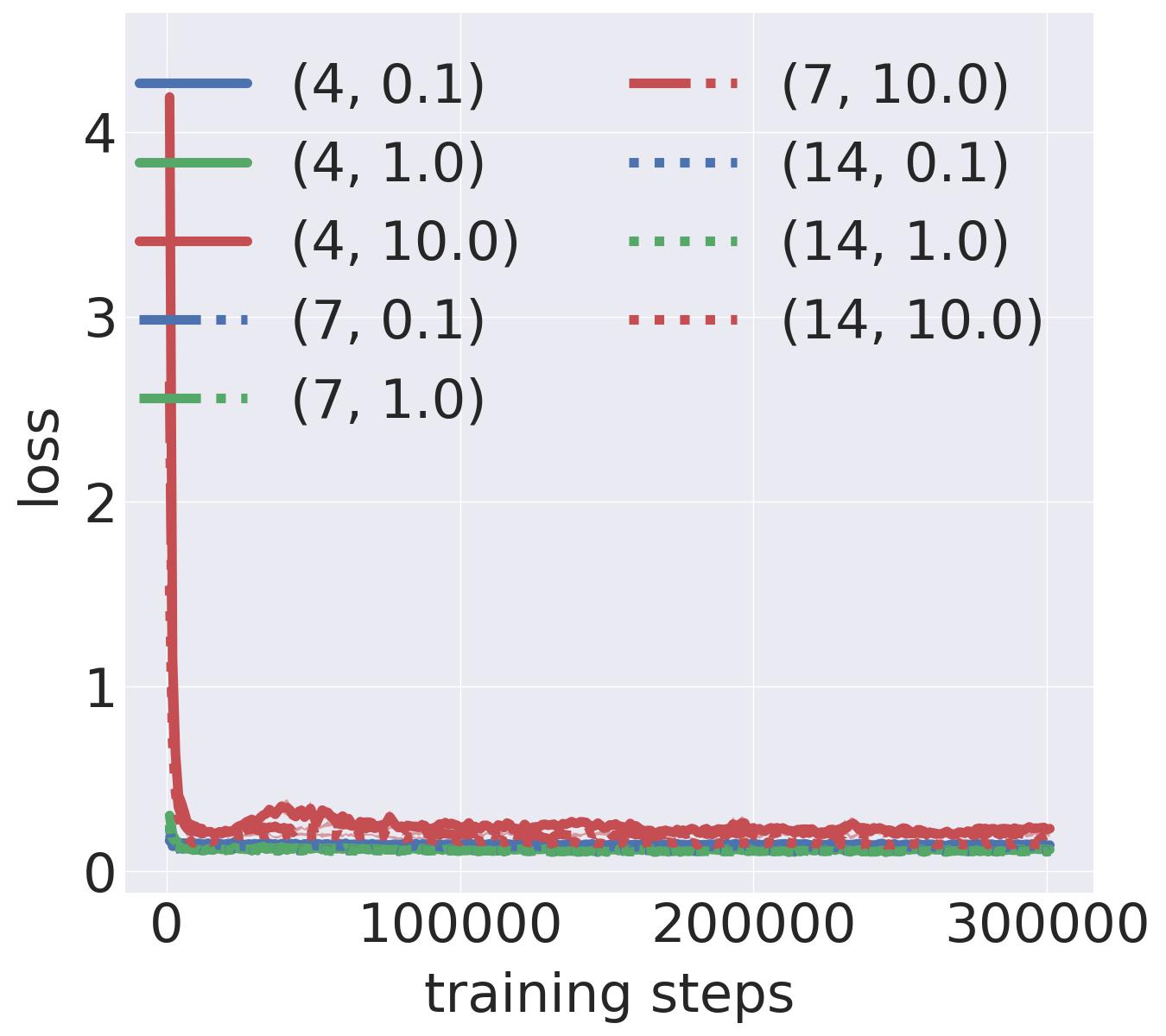}}
        \subfigure[Hopper-medium]{\includegraphics[scale=0.125]{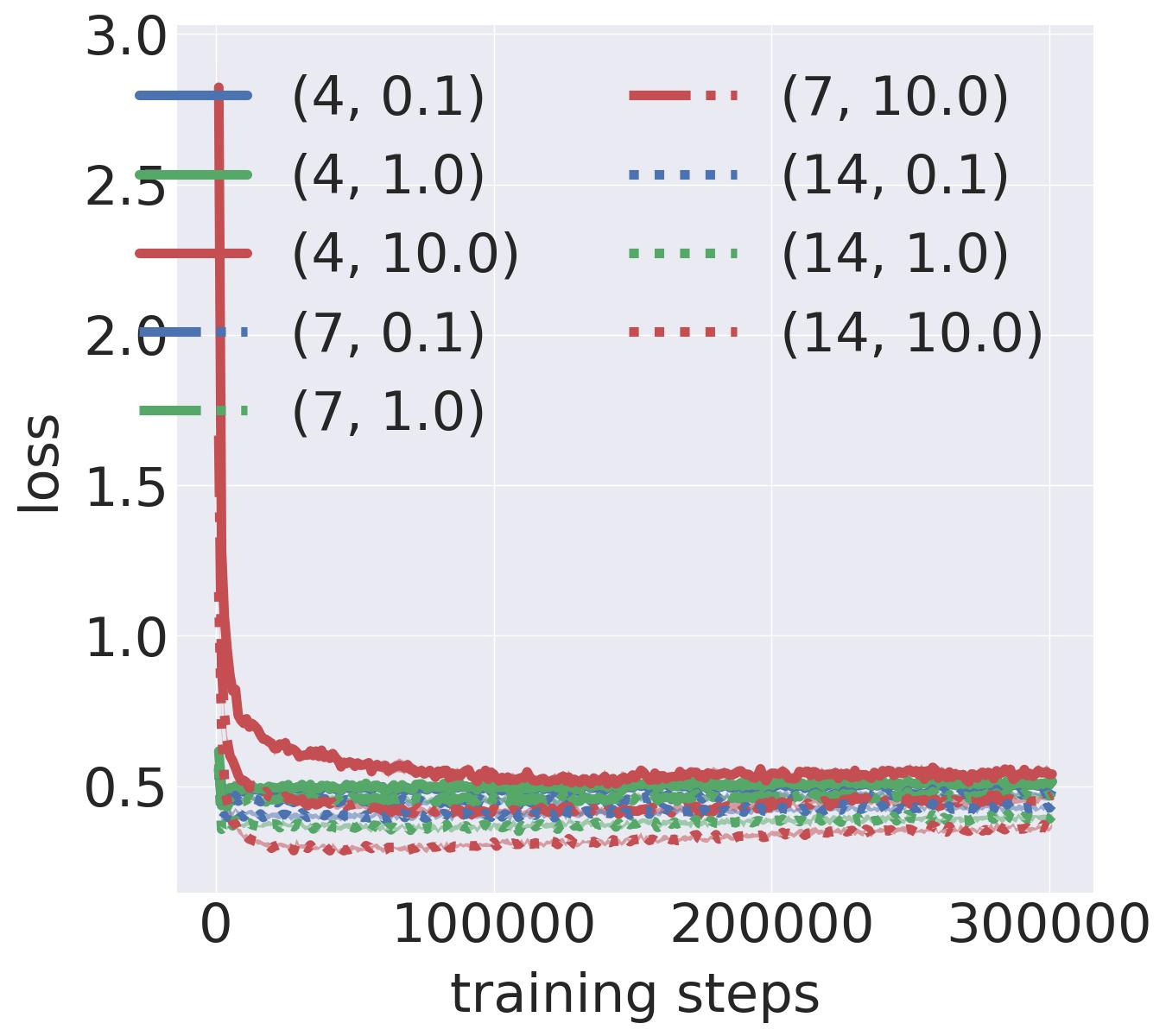}}
        \subfigure[Hopper-medium-expert]{\includegraphics[scale=0.125]{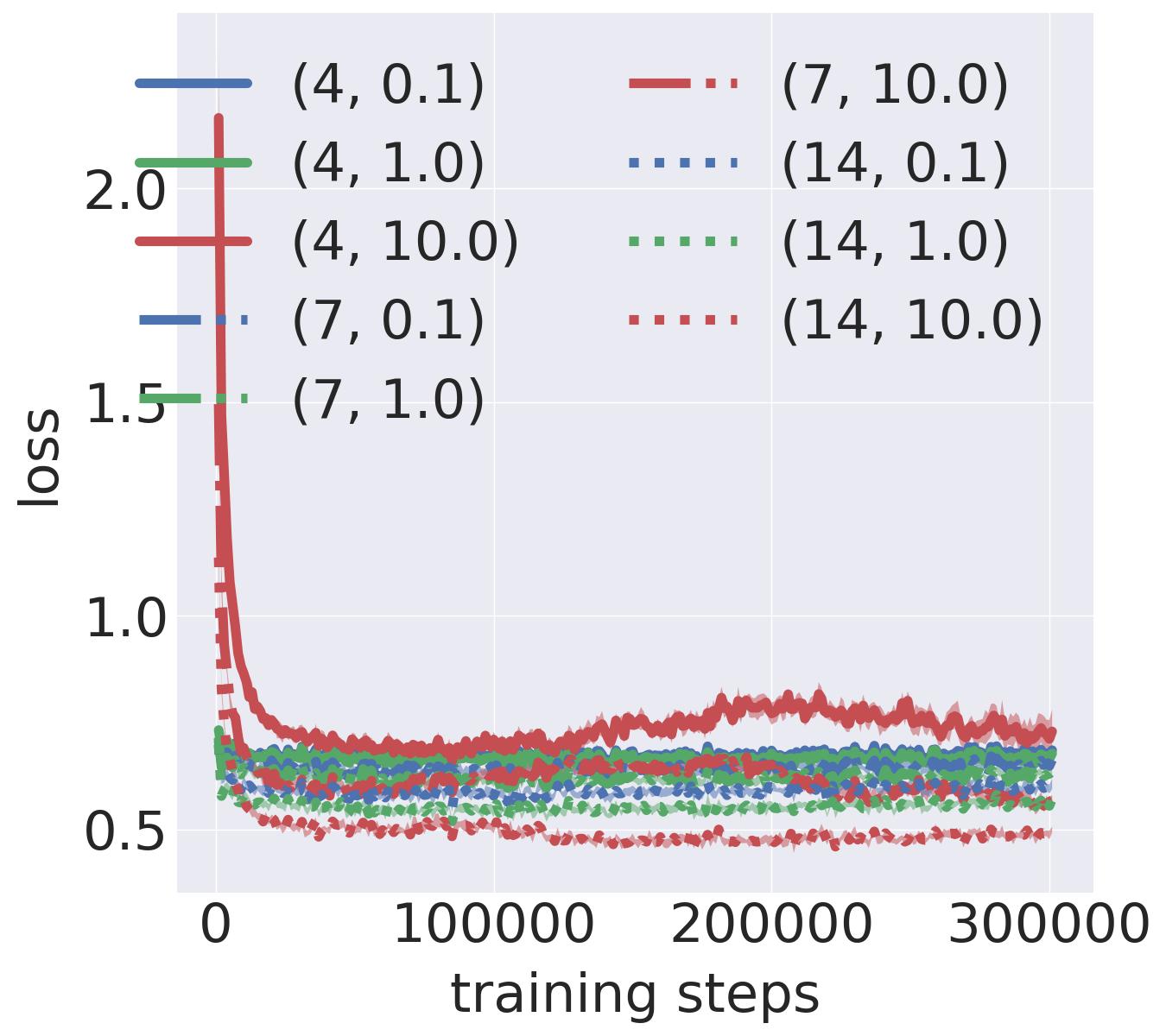}}\\
    \caption{\footnotesize \textsc{rope} training loss vs. training iterations on the \textsc{d4rl} datasets. \textsc{iqm} of errors for each domain were computed over $20$ trials with $95\%$ confidence intervals. Lower is better.}
    \label{fig:tr_losses_end}
\end{figure*}

\subsubsection{Understanding the ROPE Representations}
In this section, we try to understand the nature of the \textsc{rope} representations. We do so by plotting the mean of the: 1) mean feature dimension and 2) standard deviation feature dimension. For example, if there $N$ state-action pairs, each with dimension $D$, we compute the mean and standard deviation feature dimension for each of the $D$ dimensions across the $N$ examples, and then compute the mean along the $D$ dimensions. If the standard deviation value is close $0$, it indicates that there may be a representation collapse. See Figure~\ref{fig:rope_feature_stats}.

\begin{figure*}[hbtp]
    \centering
        \subfigure[Mean random]{\includegraphics[scale=0.125]{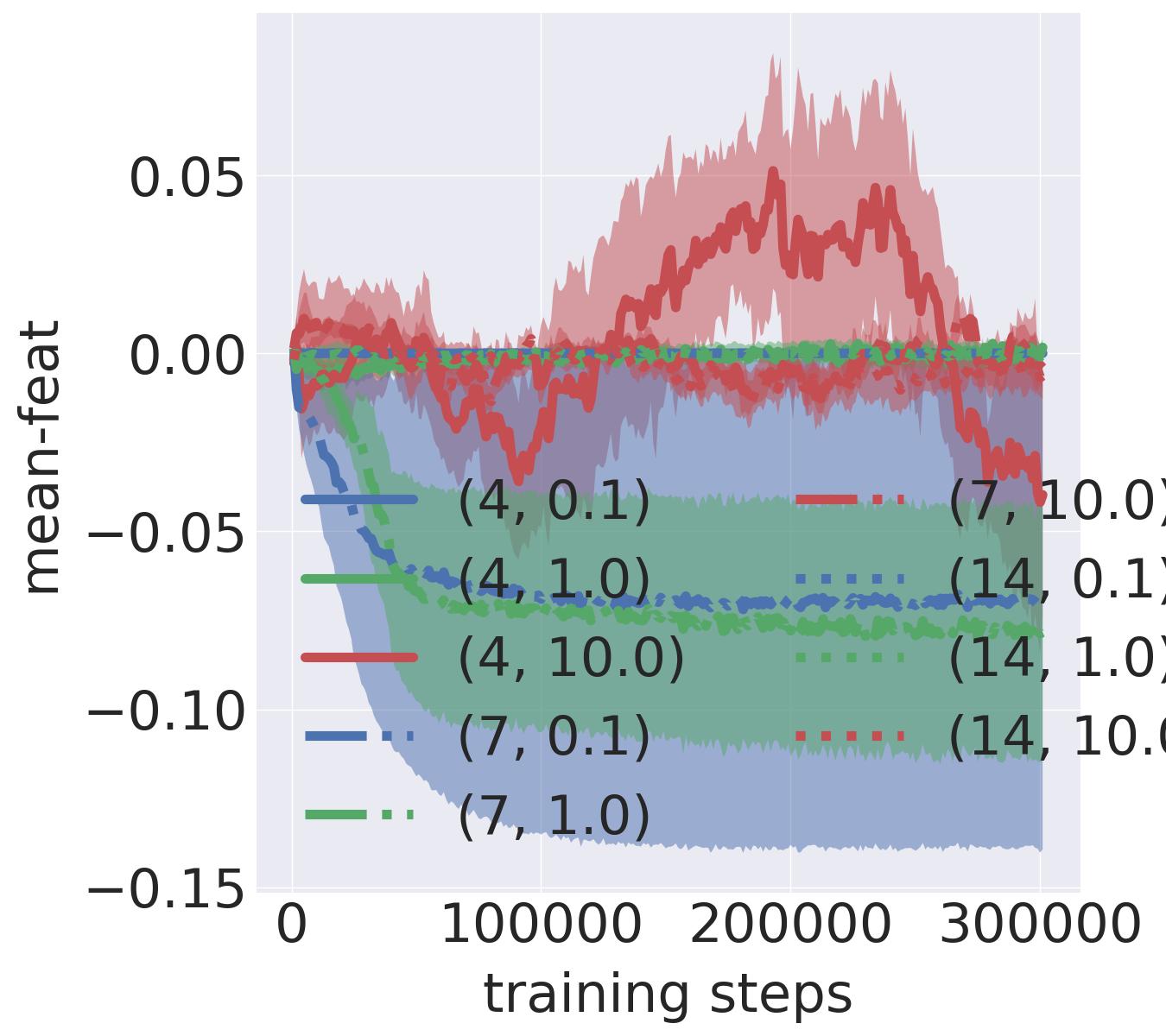}}
        \subfigure[Mean medium]{\includegraphics[scale=0.125]{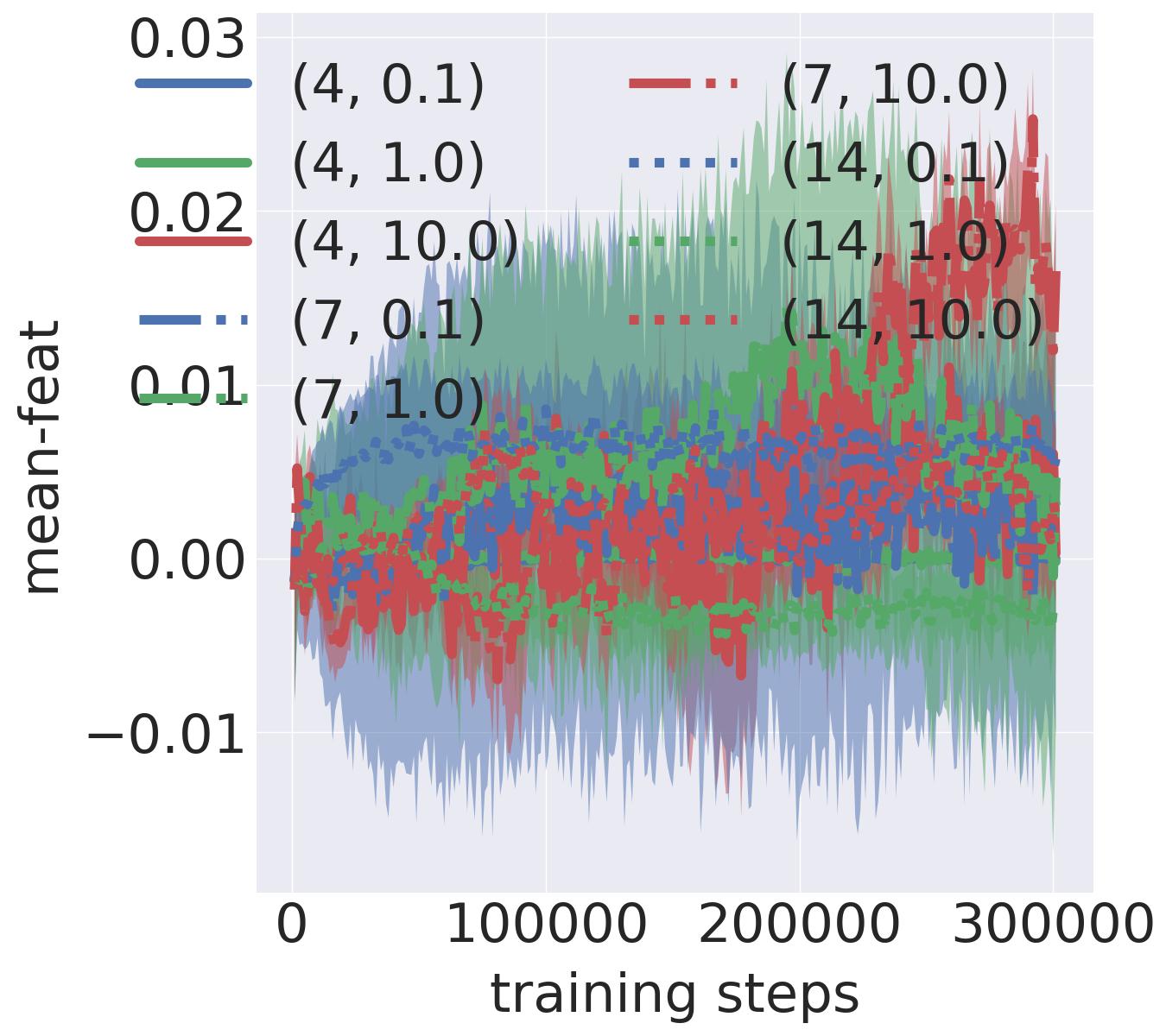}}
        \subfigure[Mean medium-expert]{\includegraphics[scale=0.125]{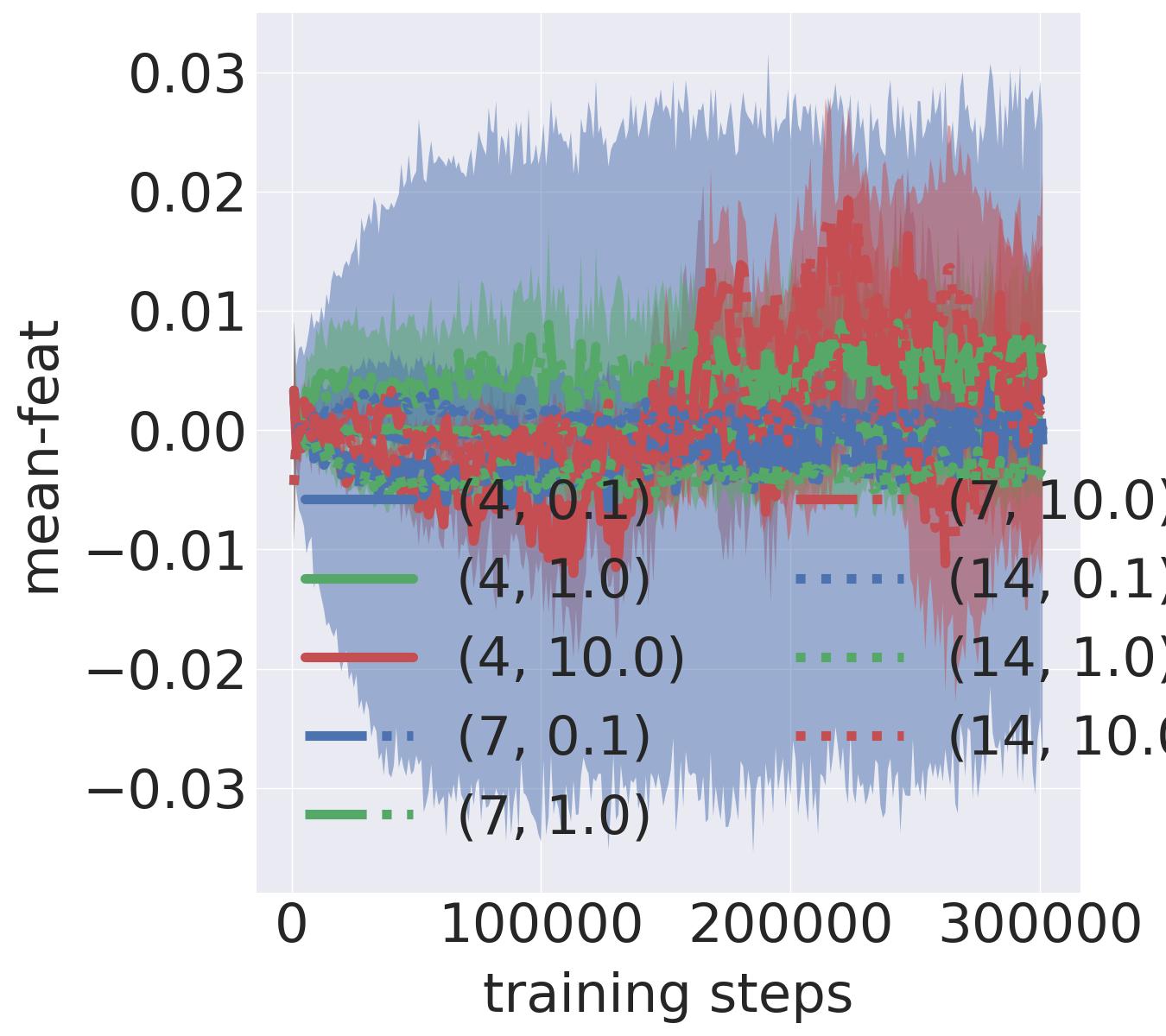}}\\
        \subfigure[Std random]{\includegraphics[scale=0.125]{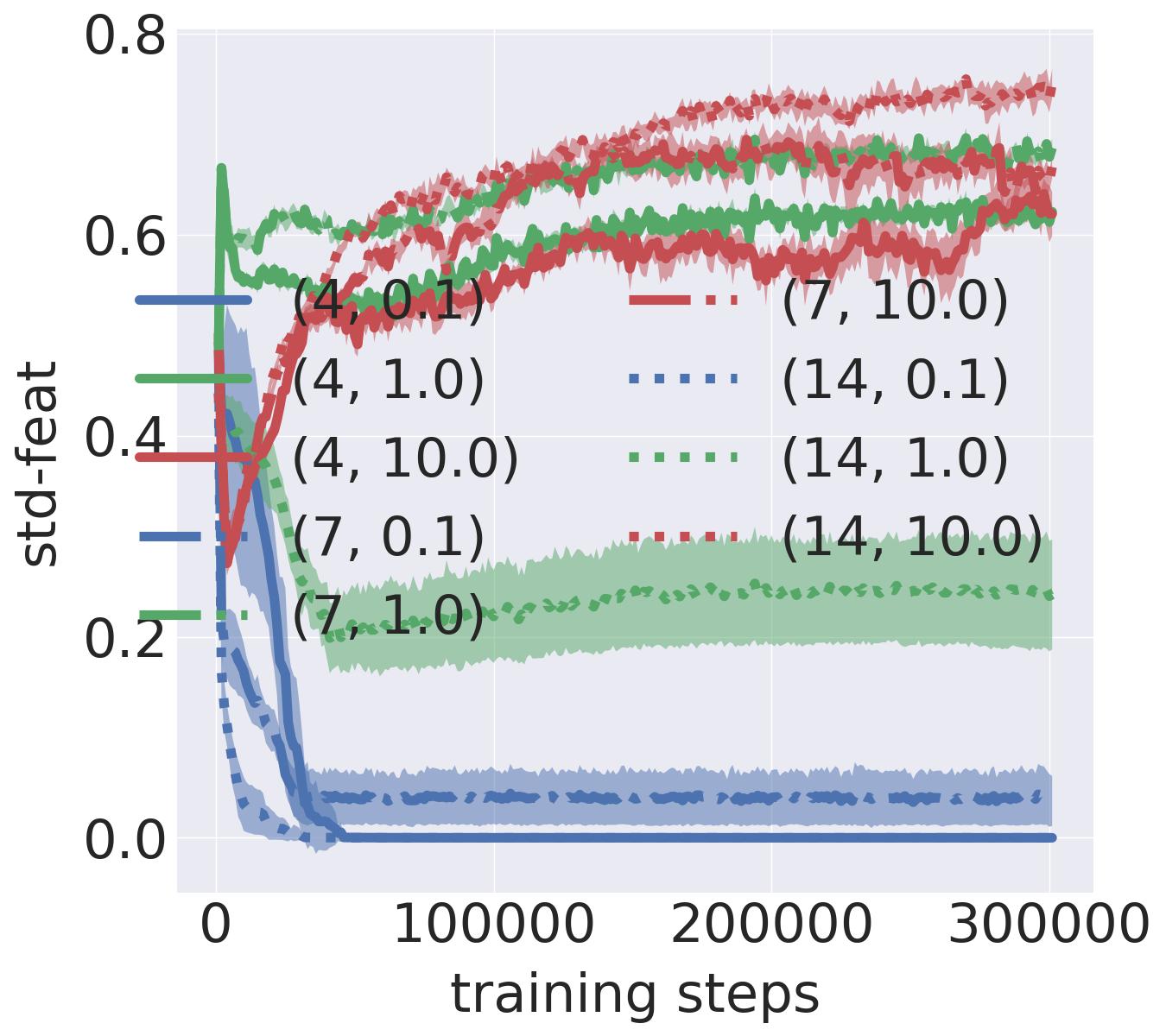}}
        \subfigure[Std medium]{\includegraphics[scale=0.125]{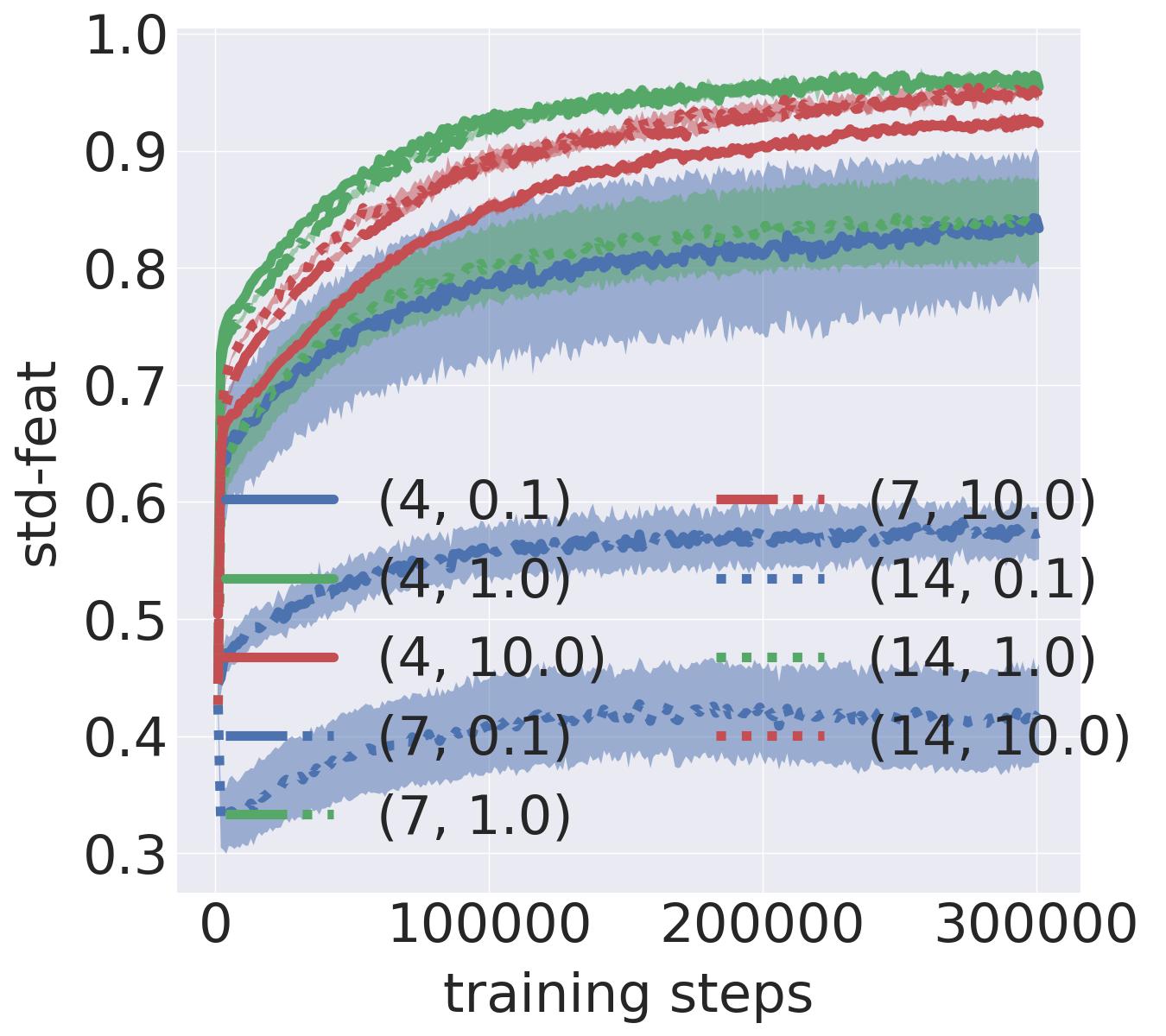}}
        \subfigure[Std medium-expert]{\includegraphics[scale=0.125]{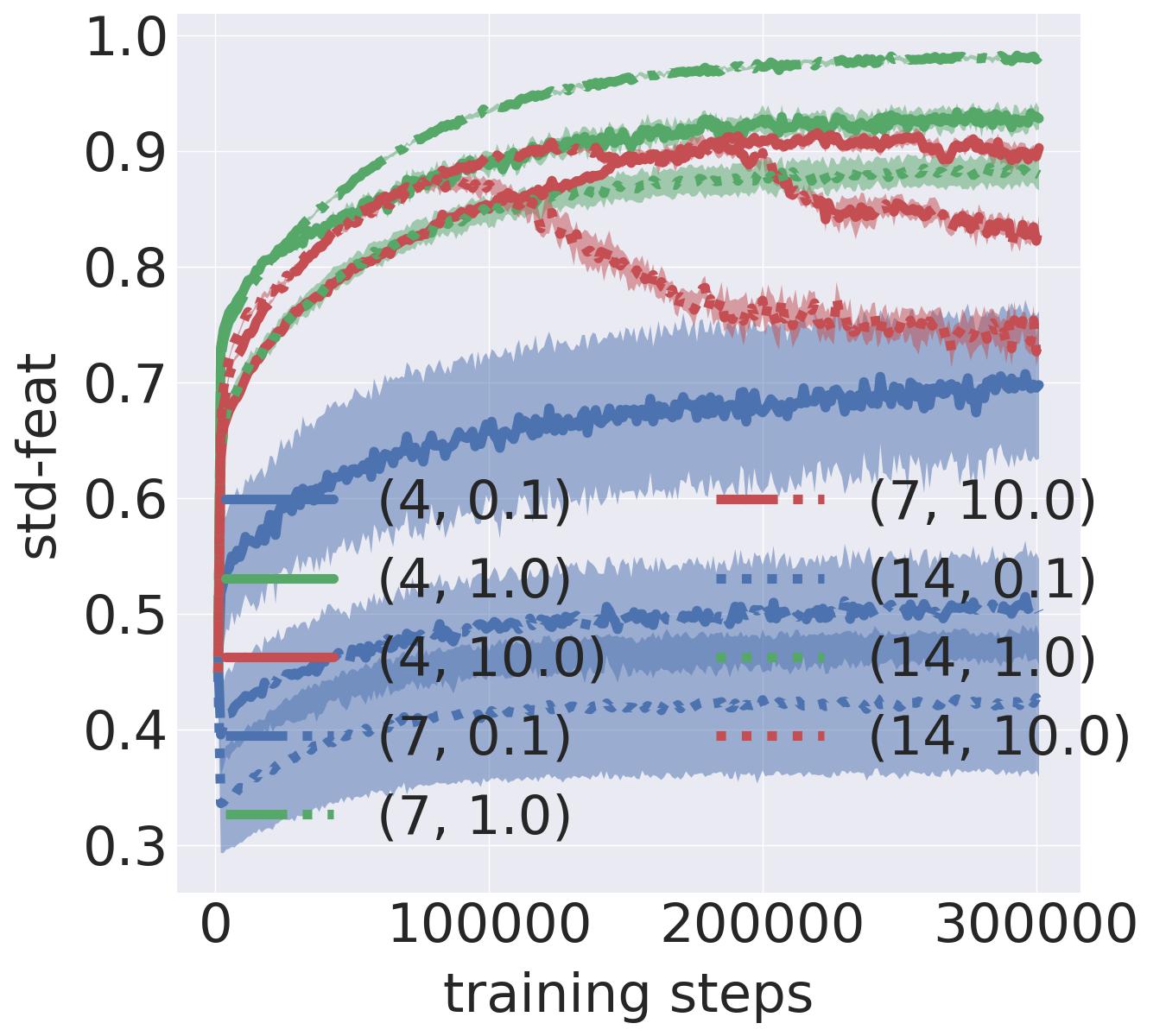}}
    \caption{\footnotesize Mean of feature dimension stats vs. training iterations on the \textsc{d4rl} Hopper dataset. \textsc{iqm} of errors for each domain were computed over $20$ trials with $95\%$ confidence intervals.}
    \label{fig:rope_feature_stats}
\end{figure*}

\subsection{Hardware For Experiments}
For all experiments, we used the following compute infrastructure:
\begin{itemize}
    \item Distributed cluster on HTCondor framework
    \item Intel(R) Xeon(R) CPU E5-2470 0 @ 2.30GHz
    \item RAM: 7GB
    \item Disk space: 4GB 
\end{itemize}

\end{document}